\def\year{2018}\relax
\documentclass[letterpaper]{article} 
\usepackage{aaai18}  
\usepackage{times}  
\usepackage{helvet}  
\usepackage{courier}  
\usepackage{url}  
\usepackage{graphicx}  
\usepackage{epsfig} 
\usepackage{mathptmx} 
\usepackage{amsmath} 
\usepackage{amssymb}  
\usepackage{amsthm}
\usepackage{thmtools,thm-restate}
\usepackage{newtxtext, newtxmath}
\usepackage{xspace}
\usepackage{color}
\usepackage{hyperref}
\usepackage{amsfonts}
\usepackage{algorithm}
\usepackage[noend]{algpseudocode}
\usepackage{url}
\usepackage{caption}
\usepackage{subcaption}
\usepackage{bm}
\usepackage{multirow}
\usepackage{booktabs}
\usepackage{appendix}

\graphicspath{{./figs/}}

\newcommand{\calE}{\ensuremath{\mathcal{E}}\xspace}
\newcommand{\calG}{\ensuremath{\mathcal{G}}\xspace}

\newcommand{\calS}{\ensuremath{\mathcal{S}}\xspace}
\newcommand{\calQ}{\ensuremath{\mathcal{Q}}\xspace}
\newcommand{\calT}{\ensuremath{\mathcal{T}}\xspace}

\newcommand{\calV}{\ensuremath{\mathcal{V}}\xspace}
\newcommand{\calP}{\ensuremath{\mathcal{P}}\xspace}

\newcommand{\R}{\mathbb{R}}

\newcommand{\Qfront}{\ensuremath{\calQ_{\rm frontier}}\xspace}
\newcommand{\Qupdate}{\ensuremath{\calQ_{\rm update}}\xspace}
\newcommand{\Qextend}{\ensuremath{\calQ_{\rm extend}}\xspace}
\newcommand{\Qrewire}{\ensuremath{\calQ_{\rm rewire}}\xspace}

\def\len#1{\ensuremath{w(#1)}}
\def\llen#1{\ensuremath{\hat{w} (#1)}}

\newcommand{\ab}{\algname{LRA*}}

\newcommand{\ib}{\ensuremath {\infty / \beta}- {\text{Lazy Dijkstra}}\xspace}

\newcommand{\vs}{\ensuremath v_{\rm source}}
\newcommand{\vg}{\ensuremath v_{\rm target}}

\newcommand{\SP}{\calS\calP}

\newcommand{\ie}{{i.e.}\xspace}

\def\naive{{na\"{\i}ve}\xspace}

\def\os#1{\textcolor{magenta}{#1}}

\newtheorem{lem}{Lemma}
\newtheorem{cor}{Corollary}

\newcommand{\ignore}[1]{}

\newboolean{ICRA}
\newboolean{ARXIV}

\setboolean{ICRA}{false}
\ifthenelse{\boolean{ICRA}}
	{\setboolean{ARXIV}{false} }
	{\setboolean{ARXIV}{true}  }

\newcommand{\textVersion}[2]
{\ifthenelse{\boolean{ICRA} }{#1}{}\ifthenelse{\boolean{ARXIV}}{#2}{}}

\newboolean{ShowAuthors}
\setboolean{ShowAuthors}{false}
\def\blind#1{
\ifthenelse{\boolean {ShowAuthors}}{#1}{}
}

\newcommand\algname[1]{\textsf{#1}\xspace}
\newcommand\lwastar{\algname{LWA*}}
\newcommand\astar{\algname{A*}}
\newcommand\lazySP{\algname{LazySP}}

\begin{document}
%
\title{Lazy Receding Horizon A* for \\Efficient Path Planning in Graphs with Expensive-to-Evaluate Edges}
\author{
Aditya Mandalika \\ University of Washington \\ \texttt{adityavk@cs.uw.edu}
\thanks{This work was (partially) funded by the National Science Foundation IIS (\#1409003), and the Office of Naval Research.}
\And
Oren Salzman \\ Carnegie Mellon University \\ \texttt{osalzman@andrew.cmu.edu}
\footnotemark[1]
\And
Siddhartha Srinivasa \\ University of Washington \\ \texttt{siddh@cs.uw.edu}
\footnotemark[1]}
\maketitle
\begin{abstract}
	Motion-planning problems, such as manipulation in cluttered environments, often require a collision-free shortest path to be computed quickly given a roadmap graph~$\calG$. 
	Typically, the computational cost of evaluating whether an edge of~$\calG$ is collision-free dominates the running time of search algorithms.	
	Algorithms such as Lazy Weighted A* (\lwastar) and \lazySP have been proposed to reduce the number of edge evaluations by employing a \emph{lazy lookahead} (one-step lookahead and infinite-step lookahead, respectively). However, this comes at the expense of additional graph operations: 
	the larger the lookahead, the more the graph operations that are typically required.
	We propose Lazy Receding-Horizon A* (\ab) to minimize the \emph{total planning time} by balancing edge evaluations and graph operations. Endowed with a lazy lookahead, \ab represents a family of lazy shortest-path graph-search algorithms that generalizes \lwastar and \lazySP.
	We analyze the theoretic properties of $\ab$ and demonstrate empirically that, in many cases, to minimize 
	the total planning time, the algorithm requires an intermediate lazy lookahead.
	Namely, using an intermediate lazy lookahead, our algorithm outperforms both \lwastar and \lazySP. 
	These experiments span simulated random worlds in $\mathbb{R}^2$ and $\mathbb{R}^4$, and manipulation problems using a 7-DOF manipulator.
\end{abstract}


\section{Introduction}
\label{sec:intro}
Robotic motion-planning has been widely studied in the last few decades. 
Since the problem is computationally hard~\cite{reif1979,S04}, 
a common approach is to apply sampling-based algorithms which typically construct a graph where vertices represent robot configurations and edges represent potential movements of the robot~\cite{CBHKKLT05,L06}. A shortest-path algorithm is then run to compute a path between two vertices on the graph. 

	There are numerous shortest-path algorithms, each suitable for a particular problem domain based on the computational efficiency of the algorithm.
	For example, \astar~\cite{HNR68} is optimal with respect to node expansions, and planning techniques such as partial expansions \cite{partialA*} and iterative deepening \cite{Korf85} are well-suited for problems with large graphs and large branching factors.

	However, in most robotic motion-planning problems, path validations and edge evaluations are the major source of computational cost~\cite{L06}. 
	Our work addresses these problems of quickly producing collision-free optimal paths, when the cost of evaluating an edge for collision is a computational bottleneck in the planning process.

	A common technique to reduce the computational cost of edge evaluation and consequently the planning time is to employ a \textit{lazy} approach.
	Two notable search-based planners that follow this paradigm are Lazy Weighted A* (\lwastar)~\cite{CPL14} and \lazySP \cite{DS16,HMPSS17}. 
	
	Both \lwastar and \lazySP assume there exists a lower bound on the weight of an edge that is efficient to compute. 
	This lower bound is used as a lookahead (formally defined in Section~\ref{sec:problem}) to guide the search without having to explicitly evaluate edges unless necessary. \lazySP uses an infinite-step lookahead which can be shown to minimize the number of edge evaluations but requires a large number of graph operations (node expansions, updating the shortest-path tree, etc.).
	On the other hand, \lwastar uses a one-step lookahead which may result in a larger number of edge evaluations compared to \lazySP but with much fewer graph operations. 
	
	Our key insight is that there should exist an optimal lookahead for a given environment, which balances the time for edge evaluations and graph operations, and minimizes the \textit{total planning time}. We make the following contributions:
\begin{enumerate}
\item We present Lazy Receding-Horizon A* (\ab), a family of lazy shortest-path algorithms parametrized 
by a \emph{lazy lookahead} (Sections~\ref{sec:problem} and~\ref{sec:alg}) which allows us to continuously
 interpolate between \lwastar and \lazySP, balancing edge evaluations and graph operations.

\item 	We analyze the theoretic properties of $\ab$  (Section~\ref{sec:proofs}).
	Part of our analysis proves that \lazySP is optimal with respect to minimizing edge evaluations thus closing a theoretic gap  left open in \lazySP~\cite{DS16}.

\item We demonstrate in Section \ref{sec:experiments}, the efficacy of our algorithm on a range of planning problems for simulated $\mathbb{R}^n$ worlds and robot
 manipulators. We show that $\ab$ outperforms both \lwastar and \lazySP by minimizing
 not just edge evaluations or graph operations but the \textit{total planning time}.
\end{enumerate}

\section{Related Work}
\label{sec:relatedWork}

A large number of motion-planning algorithms consist of 
(i)~constructing a graph, or a roadmap, embedded in the configuration space
and
(ii)~finding the shortest path in this graph.
The graph can be constructed in a preprocessing stage~\cite{kavraki96prm,KF11} or vertices and edges can be added in an incremental fashion~\cite{GSB15,SH15}.

In domains where edge evaluations are expensive and dominate the planning time, a \emph{lazy approach} is often employed~\cite{lazyPRM,hauser15lazy} wherein the graph is constructed \emph{without} testing if edges are collision-free. Instead, the search algorithm used on this graph is expected to evaluate only a subset of the edges in the roadmap and hence save computation time.
While standard search algorithms such as Dijkstra's~\cite{dijkstra} and \astar~\cite{HNR68} can be used,
specific search algorithms~\cite{CPL14,DS16} were designed for exactly such problems. They aim to further reduce the number of edge evaluations and thereby the planning time.

Alternative algorithms that reduce the number of edge evaluations have been studied.
One approach was by foregoing optimality and computing near-optimal paths~\cite{SH16,SPARS}.
Another approach was re-using information obtained from previous edge evaluations~\cite{BF16,CDS16,CSCS17}.


	In this paper, we propose an algorithm that makes use of a \emph{lazy lookahead} to guide the search and minimize the total planning time. 
	\ignore{and \emph{greediness}, which reflects the confidence in the quality of a lazily evaluated path.} 
	It is worth noting that the idea of a lookahead has previously been used in algorithms such as \textsf{RTA*}, \textsf{LRTA*}~\cite{KORF1990189} and \textsf{LSS-LRTA*}~\cite{koenig2009}.
	However, these algorithms use the lookahead in a different context by interleaving planning with execution before the shortest path to the goal has been completely computed.
	Using a lazy lookahead has also been considered in the control literature~\cite{kwon2006receding}.
Receding horizon optimization can be summarized as iteratively solving an optimal-control
problem over a fixed future interval.
Only the first step in the resulting optimal control
sequence is executed and the process is repeated after measuring the state that was reached.
Our lazy lookahead is analogous to the fixed horizon used by these algorithms. Additionally, the lazy lookahead can also be seen as a threshold that defines the extent to which (lazy) search is performed. This is similar to the Iterative Deepening version of \astar  (\textsf{IDA*}) \cite{korf1985} which performs a series of depth-first searches up to a (increasing) threshold over the solution cost.

\section{Algorithmic Background}
\label{sec:background}

\subsection{Single Source Shortest Path (SSSP) Problem}
\label{subsec:sssp}
Given a directed graph $\cal{G} = (\cal{V},\cal{E})$ with a cost function $w: \calE \rightarrow \R^+$ on its edges, the
Single Source Shortest Path (SSSP) problem is to find a path of minimum cost between two given vertices~$\vs$ and $\vg$.
Here, a path $P = (v_1, \ldots, v_k)$ on the graph is a sequence of vertices where
$\forall i, (v_i, v_{i+1}) \in \calE$. 
An edge $e = (u,v)$ belongs to a path if $\exists i \ s.t. \ u = v_i, \ v = v_{i+1}$. 
The \emph{cost} of a path is the sum of the weights of the edges along the path:
$$
\len{P} = \sum_{e \in P} w(e).
$$

\subsection{Solving the SSSP Problem}
\label{subsec:solve_sssp}
To solve the SSSP problem, algorithms such as Dijkstra~\cite{dijkstra} compute the shortest path by building a shortest-path tree $\calT$ rooted at $\vs$ and terminate once~$\vg$ is reached.
This is done by maintaining a minimal-cost priority queue~$\calQ$ of nodes called the OPEN list.
Each node~$\tau_u$ is associated with a vertex $u$ and a pointer to~$u$'s parent in $\calT$.
The nodes are ordered in~$\calQ$ according to their cost-to-come \ie, the cost to reach $u$ from $\vs$ in $\calT$.

The algorithm begins with $\tau_{\vs}$ (associated with $\vs$) in $\calQ$ with a cost-to-come of $0$. All other nodes are initialized with a cost-to-come of $\infty$. At each iteration, the node~$\tau_u$ with the minimal cost-to-come is removed from~$\calQ$ and \emph{expanded}, wherein the algorithm considers each of $u$'s neighbours~$v$, and evaluates if the path to reach $v$ through $u$ is cheaper than $v$'s current cost-to-come.
If so, then $\tau_v$'s parent is set to be $\tau_u$ (an operation we refer to as ``rewiring'') and is inserted into~$\calQ$.

The search, \ie, the growth of $\calT$, can be biased towards~$\vg$ using a heuristic function $h: \calV \rightarrow \R$ which estimates the cost-to-go, the cost to reach $\vg$ from vertex $v \in \calV$.
It can be shown that under mild conditions on $h$, if $\calQ$ is ordered according to the sum of the cost-to-come and the estimated cost-to-go, this algorithm, called \astar, \emph{expands} fewer nodes than any other search algorithm with the same heuristic~\cite{DP83,P84}.

A key observation in the described approach for Dijkstra or \astar is that for \emph{every} node in $\calQ$, the algorithm computed the cost $w(e)$ of the edge to reach this node from its current parent, a process we will refer to as \emph{evaluating} an edge.
Edge evaluation occurs for all edges leading to nodes in the OPEN list~$\calQ$ irrespective of whether there exists a better parent to the node or whether the node will subsequently be expanded for search.
In problem domains where computing the weight of an edge is an expensive operation, such as in robotic applications, this is highly inefficient. To alleviate this problem, we can apply \emph{lazy} approaches for edge evaluations that can dramatically reduce the number of edges evaluated.

\subsection{Computing SSSP via Lazy Computation}
\label{subsec:lazyAlgs}
In problem domains where computing $w(e)$ is expensive, we assume the existence of a function $\hat{w}: \calE \rightarrow \R^+$ which 
(i)~is efficient to compute and 
(ii)~provides a lower bound on the true cost of an edge \ie, $\forall e \in \calE, \hat{w}(e) \leq {w}(e)$.
We call $\hat{w}$ a \emph{lazy estimate} of~$w$.

Given such a function, \citeauthor{CPL14} proposed \lwastar which modifies \astar as follows: Each edge $(u,v)$ is evaluated, \ie, $w(u,v)$ is computed, only when the algorithm  believes that $\tau_v$ should be the next node to be expanded.
Specifically, each node in $\calQ$ is augmented with a flag stating whether the edge leading to this vertex has been evaluated or not. 
Initially, every edge is given the estimated value computed using $\hat{w}$ for its cost and this is used to order the nodes in $\calQ$.
Only when a node is selected for expansion, is the true cost of the edge leading to it evaluated. 
After the edge is evaluated, its cost may be found to be higher than the lazy estimate, or even $\infty$
(if the edge is untraversable---a notion we will refer to as ``in collision''). In such cases, we simply discard the node. 
If it is valid, we now know the true cost of the edge, as well as the true cost-to-come for this vertex from its current parent. The node is marked to have its true cost determined and is inserted into $\calQ$ again. Only when this node is chosen from~$\calQ$ the second time will it actually be expanded to generate paths to its neighbours. The algorithm terminates when $\vg$ is removed from $\calQ$ for the second time.

This approach increases the size of $\calQ$ as there can be multiple nodes associated with every vertex, one for each incoming edge. Since the true cost of an edge is unknown until evaluated, it is essential that all these nodes be stored. Although this causes an increase in computational complexity and in the memory footprint of the algorithm, the approach can lead to fewer edges evaluated and hence reduce the overall running time of the search.

\lwastar uses a one-step lookahead to reduce the number of edge evaluations.
Namely, every path in the shortest-path tree~$\calT$ may contain one edge (the last) which has only been evaluated lazily.
Taking this idea to the limit, \citeauthor{DS16} proposed the Lazy Shortest Path, or \lazySP algorithm which uses an infinite-step lookahead.
Specifically, it runs a series of shortest-path searches on the graph defined using $\hat{w}$ for all unevaluated edges.
At each iteration, it chooses the shortest path to the goal and evaluates edges along this path\footnote{In the original exposition of LazySP, the method for which edges are evaluated along the shortest path is determined using a procedure referred to as an \emph{edge selector}. In our work we consider the most natural edge selector, called forward edge selector. Here, the first unevaluated edge closest to the source is evaluated.}.
When an edge is evaluated, the algorithm considers the evaluated true cost of the edge for subsequent iterations of the search. Hence, \lazySP evaluates only those edges which potentially lie along the shortest path to the goal. The algorithm terminates when all the edges along the current shortest path have been evaluated to be valid (namely, not in collision).

A \naive implementation of \lazySP would require running a complete shortest-path search every iteration. However, the search tree computed in the previous iterations can be reused:
When an edge is found to be in collision, the search tree computed in previous iteration is locally updated using dynamic shortest-path algorithms such as \textsf{LPA*}~\cite{KLF04}. 


\subsection{Motivation}
\label{subsec:motivation}
As described in Section~\ref{subsec:lazyAlgs}, \lwastar and \lazySP attempt to reduce the number of edge evaluations by delaying evaluations until necessary. 
As we shall prove in Section~\ref{sec:proofs}, \lazySP (with a lookahead of infinity), minimizes the number of edge evaluations, at the expense of greater graph operations. When an edge is found to be in collision, the entire subtree emanating from that edge needs to be updated (a process we will refer to as \emph{rewiring}) to find the new shortest path to each node in the subtree. On the other hand, \lwastar which has a lookahead of one, evaluates a larger number of edges relative to \lazySP but does not perform any rewiring or repairing. When an edge to a node is found to be invalid, the node is simply discarded and the algorithm continues. 

Therefore, these two algorithms, \lwastar and \lazySP, with a one-step and an infinite-step lookahead respectively, form two extremals to an entire spectrum of potential lazy-search algorithms based on the lookahead chosen. We aim to leverage the advantage that the lazy lookahead can provide to interpolate between \lwastar and \lazySP, and strike a balance between edge evaluations and graph operations to minimize the total planning time.

\section{Problem Formulation}
\label{sec:problem}
	In this section we formally define our problem. To make this section self contained, we repeat definitions that were mentioned in passing in the previous sections.
	We consider the problem of finding the shortest path between source and target vertices $\vs$ and $\vg$ on a given graph $\cal{G} = (\cal{V},\cal{E})$. 
	Since we are motivated by robotic applications where edge evaluation, \ie, checking if the robot collides with its environment while moving along an edge, is expensive, we do not build the graph $\calG$ with just feasible edges.	
	Rather, as in the lazy motion-planning paradigm, the idea is to construct a graph with edges \textit{assumed} to be feasible and delay the evaluation to only when absolutely necessary.
	 
	For simplicity, we assume the lazy estimate $\hat{w}$ to tightly estimate the true cost $w$ for edges that are collision-free\footnote{We discuss relaxing the assumption that $\hat{w}(e)$ tightly estimates $w(e$) in Section~\ref{sec:futureWork}. 
	In the general case, we require only that it is a lower bound \ie, $\forall e \in \calE, \hat{w}(e) \leq {w}(e)$.}. 
	Therefore $\hat{w}$ is a \emph{lazy estimate} of~$w$ such that 
\begin{equation}
w(e) = 
     \begin{cases}
        \hat{w}(e) &\quad\text{if } e \text{ is not in collision,}\\
        \infty  &\quad\text{if } e \text{ is in collision.}         
     \end{cases}
     \label{wEst}
\end{equation} 

	We use the cost function $w$ and its lazy estimate $\hat{w}$ to define the cost of a path on the graph. The \emph{(true) cost} of a path~$P$ is the sum of the weights of the edges along~$P$:
$$
\len{P} = \sum_{e \in P} w(e).
$$
Similarly, the \emph{lazy cost} of a path is the sum of the lazy estimates of the edges along the path:
$$
\llen{P} = \sum_{e \in P} \hat{w}(e).
$$
Our algorithm will make use of paths which are only partially evaluated. 
Specifically, every path~$P$ will be a concatenation of two paths 
$P = P_{\text{head}} \cdot P_{\text{tail}}$ (here, ($\cdot$) denotes the concatenation operator).
Edges belonging to $P_{\text{head}}$ will have been evaluated and known to be collision-free while 
edges belonging to~$P_{\text{tail}}$ will only be lazily evaluated.
Notice that~$P_{\text{tail}}$  may be empty.
We also define the \textit{estimated total cost} of a path $P = P_{\text{head}} \cdot P_{\text{tail}}$ as:
$$
\bar{w}(P) = w(P_{\text{head}}) + \hat{w}(P_{\text{tail}}).
$$

	Although $\hat{w}$ helps guide the search of a \textit{lazy} algorithm, as in \lwastar or \lazySP, as noted in Section~\ref{subsec:motivation}, it can lead to additional computational overhead when the estimate is wrong, \ie, when the search algorithm encounters edges in collision. In this work we balance this computational overhead with the number of edge evaluations, by endowing our search algorithm with a \emph{lookahead}~$\alpha$.
In essence, the \emph{lookahead} controls the extent 
to which we use~$\hat{w}$ to guide our search. \ignore{
and~$\beta$ controls the amount we exploit the data computed using lookahead~$\alpha$.} 

As we will see later in Section~\ref{sec:futureWork}, the lookahead $\alpha$ can be interpreted in various ways. 
However, in this paper we interpret the lookahead as the  \textit{number of edges} over which we use $\hat{w}$ to guide our search. 

\section{Lazy Receding-Horizon A* (\ab)}
\label{sec:alg}


\subsection{Algorithmic details}

Our algorithm maintains a lazy shortest-path tree~$\calT$ over the graph~$\calG$.
Every node in~$\calT$ is associated with a vertex of~$\calG$ and the tree is rooted at the node $\tau_{\rm source}$ associated with the vertex~$\vs$. 
We define the node entry $\tau \in \calT$ as $\tau = (u, p, c, \ell, b)$,
were
$u[\tau] = u$ is the vertex associated with~$\tau$,
$p[\tau] = p$ is $\tau$'s parent in $\calT$ which can be backtracked to compute a path $P[\tau]$ from $\vs$ to $u$.
The node~$\tau$ also stores
$c[\tau] = c$ 
and
$\ell[\tau] = \ell$
which are the costs of the evaluated and lazily-evaluated portions of $P[\tau]$, respectively. 
Namely, 
$c[\tau] = w(P[\tau]_{\text {head}})$ and
$\ell[\tau] = \hat{w}(P[\tau]_{\text {tail}})$.
Finally, 
$b[\tau] = b$ is the budget of~$P[\tau]$ \ie, the number of edges that have been lazily evaluated in $P[\tau]$ or equivalently, the number of edges in $P[\tau]_{\text {tail}}$. 
Given a lookahead $\alpha$, our algorithm will maintain shortest paths to a set of nodes represented by the search tree~$\calT$, where $\forall \tau \in \calT$, $b[\tau] \leq \alpha$. The budget of any node in~$\calT$ never exceeds $\alpha$. 

Given a node $\tau \in \calT$, 
we call it a \emph{frontier} node if $b[\tau] =~\alpha$
($P[\tau]_{\rm {tail}}$ has exactly $\alpha$ edges).
Additionally, $\tau$ is said to belong to the \emph{$\alpha$-band} if $b[\tau] > 0 $.
We call $\tau$ a \emph{border} node if it does not belong to the $\alpha$-band but one of its children does. Finally,~$\tau$ is called a \textit{leaf} node if it has children in $\calG$ but not in $\calT$. Note that all frontier nodes are leaf nodes. See Fig.~\ref{fig:taxonomy} for reference.

\textVersion{The algorithm maintains a priority queue $\Qfront$ that stores the frontier nodes ordered according to the estimated cost-to-come $\bar{w}(P[\tau]) = c[\tau] + \ell[\tau]$. This queue is used to choose which path to evaluate at each iteration.}{
The algorithm maintains four priority queues to efficiently process the different kinds of nodes, each of which is ordered according to the estimated cost-to-come $\bar{w}(P[\tau]) = c[\tau] + \ell[\tau]$.
Specifically, we will make use of the following queues:
\begin{itemize}
	\item $\Qfront$ stores the frontier nodes.
This queue is used to choose which path to evaluate at each iteration.

	\item $\Qextend$ stores leaf nodes that have a budget smaller than~$\alpha$. \ignore{It is used to \emph{extend} the $\alpha$-band such that all leaf nodes are frontier nodes.}

	\item $\Qrewire$ stores nodes that require rewiring. It is used to update the structure of $\calT$ when an edge is evaluated to be in collision.

	\item $\Qupdate$ stores the nodes that have children in $\calT$ whose entries need to be updated. It is used to update the structure of $\calT$ after an edge is found to be collision-free.
\end{itemize}}
For ease of exposition, we present a high-level description of the algorithm (Alg.~\ref{alg:highLevel}) \textVersion{}{which uses only one of these queues}.
For detailed pseudo-code, see%
\textVersion{
supplementary material.}
{
Appendix~\ref{app:code}.
}

\subsection{Algorithm Description}
	Lazy Receding-Horizon A* (\ab) begins by initializing the node $\tau_{\rm source}$ associated with $\vs$ (line~2).
	Our algorithm maintains the invariant that at the beginning  of any iteration all leaf nodes are frontier nodes. 
	When the algorithm starts,~$\tau_{\vs}$ is a leaf node with $b[\tau_{\vs}] = 0 < \alpha$. 
	Therefore we extend the $\alpha$-band (lines~3-4) and consequently the search tree $\calT$, adding all the frontier nodes to $\Qfront$ (line~5).

%
%
\begin{algorithm}[tb]
\caption{$\texttt{LRA}^*(\calG,~\vs,~\vg,~\alpha$)}
\label{alg:highLevel}	
\begin{algorithmic}[1]
\State $\Qfront,~\calT := \emptyset$ \Comment{Initialization}
\State insert $\tau_{\vs} = (\vs,~\text{NIL},~0,~0,~0)$ into $\calT$
\vspace{2mm}
\For{each leaf node $\tau \in \calT$} \Comment{Extend $\alpha$-band} 
	\State add all nodes at distance $(\alpha - b[\tau])$ edges into $\calT$
\EndFor
\State insert all frontier nodes in $\calT$ into $\Qfront$ 
\vspace{2mm}
\While{$\Qfront$ is not empty} \Comment{Search}
	\State remove $\tau$ with minimal key $\bar{w}(\tau)$ from $\Qfront$
	\State evaluate first edge $(u,v)$ along $P[\tau]_{\text{tail}}$ \Comment{Expensive} 
	\If{$(u,v)$ is collision-free}
		\State update $\tau_v$
		\If{$v = \vg$} 
			\State \textbf{return} $P[\tau_{\vg}]$
		\EndIf  	
		\State update descendants $\tau$ of $\tau_v$ s.t $\tau \in \calT$
	\Else \Comment{Edge is in collision}
		\State remove edge $(u,v)$ from graph
	\For{each descendant $\tau$ of $\tau_v$ s.t $\tau \in \calT$}
		\State rewire $\tau$ to the best parent $\tau' \in~$\calT, $\tau' \neq \tau_{\vg}$
	\EndFor
	\EndIf   

		
	\State \textbf{repeat} steps 3-5 to extend the $\alpha$-band
\EndWhile
\State \textbf{return} failure
\end{algorithmic}
\end{algorithm}

\begin{figure*}[t]%
	\centering
	\begin{subfigure}[h]{0.24\textwidth}
	  	\centering
  		  \includegraphics[width=0.9\textwidth]{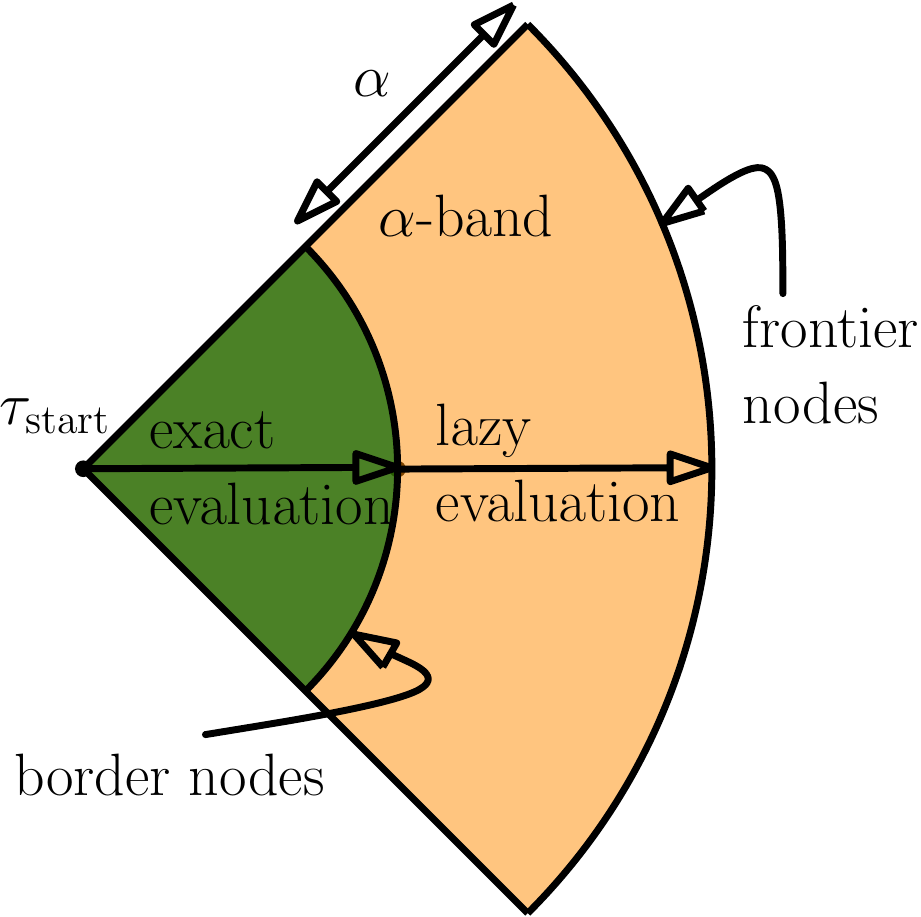}
  		  \caption{}
   	      \label{fig:taxonomy}
   	\end{subfigure}
	\begin{subfigure}[h]{0.24\textwidth}
		\centering
		  \includegraphics[width=0.9\textwidth]{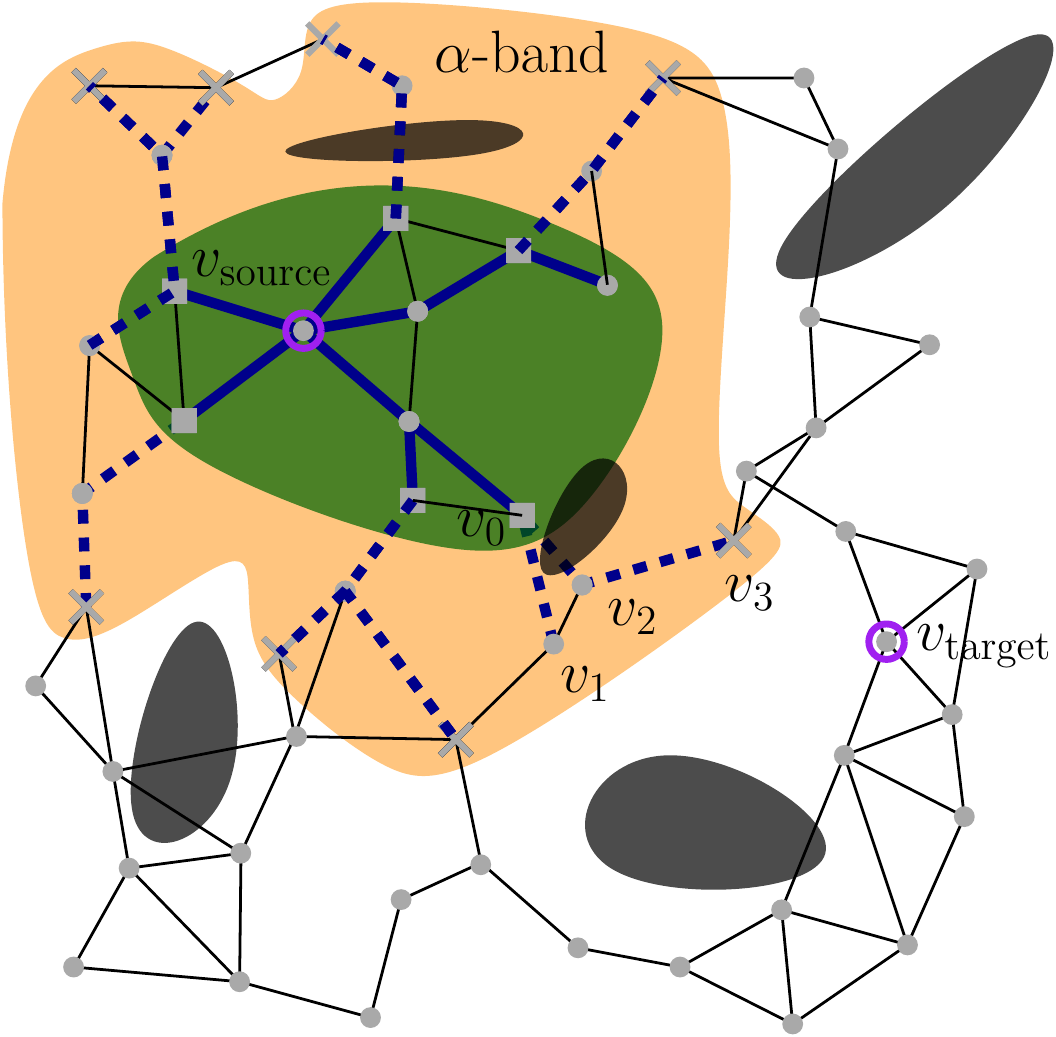}
		  \caption{}
		  \label{fig:alg1}		
	\end{subfigure}
	\begin{subfigure}[h]{0.24\textwidth}
		\centering
		  \includegraphics[width=0.9\textwidth]{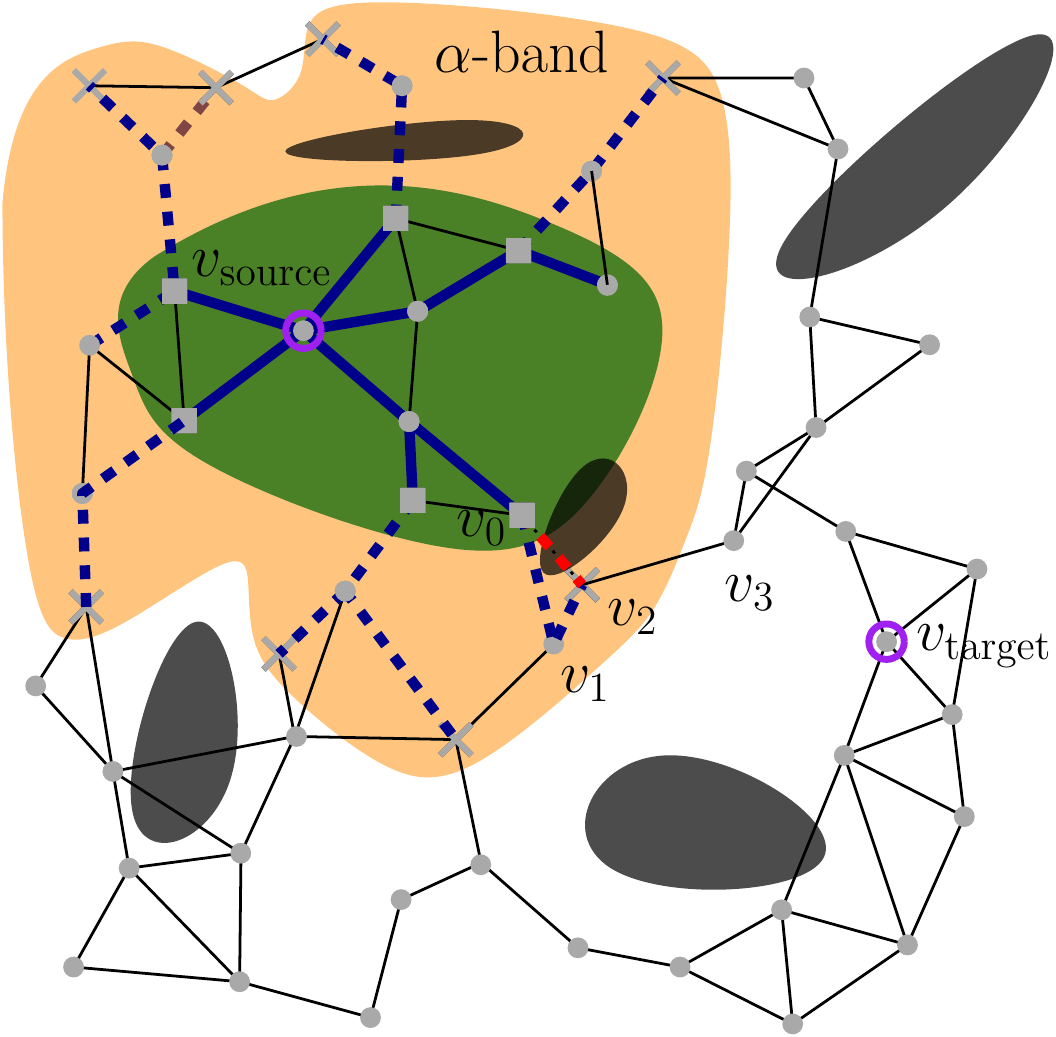}
		  \caption{}
		  \label{fig:alg2}		
	\end{subfigure}
	\begin{subfigure}[h]{0.24\textwidth}
		\centering
		  \includegraphics[width=0.9\textwidth]{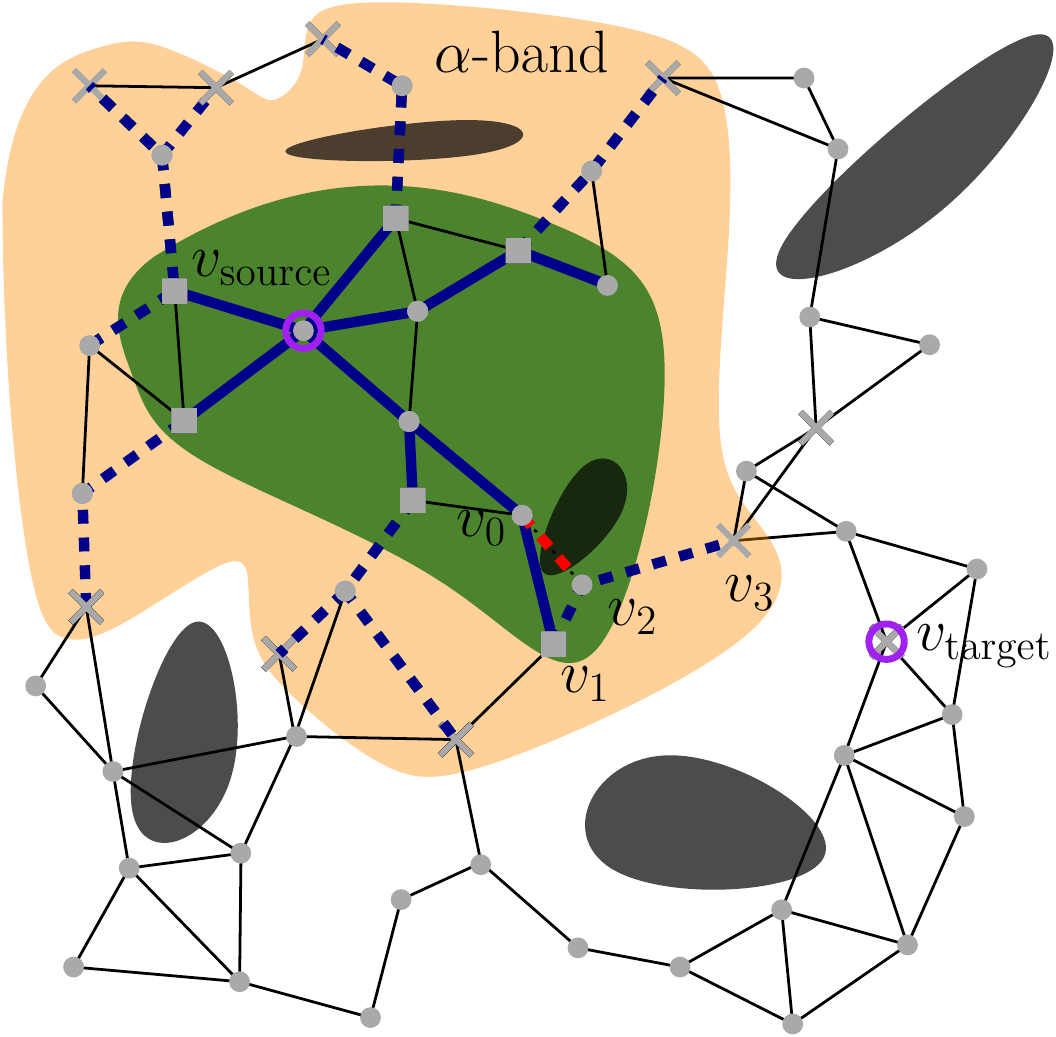}
		  \caption{}
		  \label{fig:alg3}		
	\end{subfigure}
  \caption{%
  	(\subref{fig:taxonomy}) Search space of \ab.
    Figures (\subref{fig:alg1}, \subref{fig:alg2}, \subref{fig:alg3}) visualize \ab running on $\calG$ embedded in a workspace cluttered with obstacles (dark grey) and~$\alpha=2$. The regions where edges are evaluated and lazily evaluated are depicted by green and orange regions, respectively. Shortest-path tree~$\calT$ in the two regions is depicted by solid and dashed blue edges, respectively. Finally, vertices associated with border and frontier nodes are depicted by squares and crosses, respectively. Figure is best viewed in color.
    (\subref{fig:alg1})~Node associated with $v_3$ has the minimal key and the path ending with nodes $v_0, v_2, v_3$ is evaluated.
    Edge $(v_0, v_2)$ is found to be in collision.
    (\subref{fig:alg2})~Node $\tau_2$ associated with $v_2$ is rewired and the $\alpha$-band is recomputed. Now $\tau_2$ has the minimal key and the path ending with nodes $v_0, v_1, v_2$ is evaluated and found to be collision free.
    (\subref{fig:alg3})~The $\alpha$-band is extended from $v_2$.
		}%
  \label{fig:filmstrip}%
%

\end{figure*}

	The algorithm iteratively finds the frontier node $\tau$ with minimal estimated cost (line~7) and evaluates the first edge along the lazy portion $P[\tau]_{\text{tail}}$ of the path $P[\tau]$ from $\tau_{\vs}$ to $\tau$ in $\calT$ (line~8).
	If a collision-free shortest path to $\vg$ is found during this evaluation (line~11-12), the algorithm terminates.
	Every evaluation of a collision-free edge $(u,v)$ causes the node $\tau_v$, that was previously in the $\alpha$-band, to be a border node. Consequently the node entry is updated and this update is cascaded to all the nodes in the $\alpha$-band belonging to the subtree rooted at $\tau_v$ (lines~10, 13). Specifically, the new cost, lazy cost and budget of $\tau_v$ is used to update the nodes in its subtree. 
	However, if the edge $(u,v)$ is found to be in collision, the edge is removed from the graph, and the entire subtree of~$\tau_v$ is rewired appropriately (lines~14-17). This can potentially lead to some of the nodes being removed from the $\alpha$-band.
	Both updating and rewiring subtrees can generate leaf nodes with budget less than $\alpha$. Therefore at the end of the iteration, the $\alpha$-band is again extended to ensure all leaf nodes have budget equal to the lookahead~$\alpha$ (line~18). See Fig. \ref{fig:filmstrip} for an illustration.
	
	As we will show in Section \ref{sec:proofs}, the algorithm described 	is guaranteed to terminate with the shortest path, if one exists, and is hence complete for all values of~$\alpha$. 


\subsection{Implementation Details---Lazy computation of the $\alpha$-band}
Every time an edge $(u,v)$ is evaluated, a series of updates is triggered (Alg.~1 lines 13 and 16-17)
Specifically, let $\tau$ be the node associated with $v$ and $\calT(\tau)$ be the subtree of $\calT$ rooted at $\tau$.
If the edge~$(u,v)$ is collision-free, then the budget of all the nodes $\calT(\tau)$ needs to be updated.
Alternatively, if the edge $(u,v)$ is in collision, then a new path to every node in $\calT(\tau)$ needs to be computed.
These updates may be time-consuming and we would like to minimize them.
To this end, we propose the following optimization which reduces the size of the $\alpha$-band and subsequently, potentially reduces the number of nodes in~$\calT(\tau)$.

We suggest that if we already know that a node $\tau'$ in the $\alpha$-band  will \emph{not} be part of a path that is chosen for evaluation in an iteration, then we defer expanding the $\alpha$-band through this node.
The key insight behind the optimization is that there is no need to expand a node $\tau'$ in the $\alpha$-band if its key, ~$\bar{w}(P[\tau'])$,  is larger than the key of the first node in~$\Qfront$.
Using this optimization may potentially reduce the size of~$\calT(\tau$) and save computations.

%
\textVersion{}
{This is implemented by changing the termination criteria in Alg.~\ref{alg:extend} (line 1) to test if the key of the first node in~$\Qextend$ is larger than the key of the head of the first node in~$\Qfront$.}

\subsection{Implementation Details---Heuristically guiding the search}
We described our algorithm as a lazy extension of Dijkstra's algorithm which orders its search according to cost-to-come.
In practice we will want to heuristically guide the search similar to \astar, which orders its search queues according to the sum of cost-to-come to a vertex from $\vs$ and an estimate of the cost-to-go to $\vg$ from the vertex, \ie, a heuristic.

We apply a similar approach by assuming that the algorithm is given a heuristic function that under estimates the cost to reach $\vg$. \textVersion{}{We add this value to the key of every node in $\Qfront$ and $\Qextend$.} In Section~\ref{sec:proofs} we state and prove that as the heuristic is strictly more informative, the number of edge evaluations and rewires further reduce, for a given lazy lookahead.

\subsection{Discussion---\ab as an approximation of optimal heuristic}
\label{sec:intuition}
In this section, we provide an intuition on the role that the lazy lookahead plays when guided by a heuristic.
Given a graph~$\calG$, we can define the optimal heuristic $h_\calG^*(v)$ as the length of the shortest path from $v$ to $\vg$ in~$\calG$.
Indeed, if all edges of $\calG$ are collision-free, an algorithm such as \textsf{A*} guided by $h_\calG^*$ will only evaluate edges along the shortest path to  $\vg$. To take advantage of this, \lazySP proceeds by computing $h_\calG^*$. 
If an edge is found to be in collision, it is removed from~$\calG$ and $h_\calG^*$  is recomputed.
This is why no other algorithm can perform fewer edge evaluations (see Section~\ref{sec:proofs}).

Using a finite \emph{lookahead} and a \emph{static} admissible heuristic, \ab can be seen as a method to \emph{approximate} the optimal heuristic.   
Every frontier node $\tau$ is associated with the key $c[\tau] + \ell[\tau] + h(u[\tau])$. The minimal of all such keys forms the approximation for the optimal heuristic $h_\calG^*(\vs)$ \ie, if $\tau_v$ associated with vertex $v$ has the minimal key, we have,
$$
h_\calG^*(\vs) \geq c[\tau_v] + \ell[\tau_v] + h(v) \geq h(\vs)
$$
and the algorithm chooses to evaluate an edge along the path from $\vs$ to $v$ in $\calT$.
This approximation improves as the $\alpha$-band approaches the target.
When the algorithm starts, this approximation may be crude (when a small lazy lookahead is used).
However, as the algorithm proceeds and $\alpha$-band is expanded, this approximation dynamically converges to the optimal heuristic. 
The approximation can also be improved by increasing the lookahead since a larger lookahead enables the algorithm to be more informed.
We formalize these ideas in Section~\ref{sec:proofs} (see Lemmas \ref{lem:heuristic}, \ref{lem:larger_lookahead_no_greediness}), Section~\ref{sec:dynamicResult} and show this phenomenon empirically in Section~\ref{sec:experiments}.

\subsection{Discussion---Is greediness beneficial?}
\label{sec:greediness}
A possible extension to \ab is to employ greediness in edge evaluation:
Given a path, we currently evaluate the first edge along this path (Alg.~\ref{alg:highLevel}, lines~7 and~8).
However, we can choose to evaluate more than one edge, hence performing an exploitative action.
This introduces a second parameter~$\beta \leq \alpha$ that indicates how many edges to evaluate along the path.
However, we can show that our current formulation using a minimal greediness value of $\beta = 1$ always outperforms any other greediness value.
This is only the case when we seek \emph{optimal} paths.
If we relax the algorithm to produce suboptimal paths, greediness may be of use in early termination.
While this relaxation is out of the scope of the paper, we provide proofs pertaining to the superiority of no greediness in 
\textVersion
{the supplementary material}
{the Appendix~\ref{app:proofs2}}
for the case that optimal paths are required.

\section{Correctness, Optimality and Complexity}
\label{sec:proofs}

In this section we provide theoretical properties regarding our family of algorithms \ab.
For brevity, we defer all proofs to
\textVersion
{the supplementary material.}
{Appendix~\ref{app:proofs}.}
We start in Section~\ref{subsec:correctness} with a correctness theorem stating that upon termination of the algorithm, the shortest path connecting~$\vs$  and~$\vg$ is found.
We continue in Section~\ref{subsec:opt} to detail how the lazy lookahead affects the performance of the algorithm with respect to edge evaluations.
Specifically, we show that for $\alpha = \infty$, the algorithm is edge optimal. That is, it tests the minimal number of edges possible (this notion is formally defined).
Furthermore, we examine how the lazy lookahead affects the number of edges evaluated by our algorithm.
Finally, in Section~\ref{subsec:complexity} we bound the running time of the algorithm as well as its space complexity as a function of the lazy lookahead $\alpha$.
Here, we show that the running time (governed, in this case, by graph operations) can  grow exponentially with the lazy lookahead  $\alpha$. This further backs our intuition that in order to minimize the running time in practice, an intermediate lookahead is required to balance edge evaluation and graph operations.

The following additional notation will be used throughout this section:
Let $P^*_v$ denote the shortest collision-free path from $\vs$ to a vertex $v$ and let $w^*(v) = w(P^*_v)$ be the minimal \emph{true} cost-to-come to reach~$v$ from $\vs$.
Finally, for the special case of $\vg$, we will use~$w^* = w^*(\vg)$.
That is,~$w^*$ denotes the minimal cost-to-come to reach $\vg$ from~$\vs$.
\subsection{Correctness}
\label{subsec:correctness}

\begin{restatable}{lem}{lemmaOne}
\label{lem:correctness}
Let $(v_0,v)$  be an edge evaluated by \ab and found to be collision free.
Then the shortest path to the node~$\tau_v$ associated with vertex~$v$ has been found and $
c[\tau_v] = w^*(v)
$.
\end{restatable}

Replacing $v$ with $\vg$, we have,
\begin{cor}
\ab is complete, \ie, if an edge $(v_0, \vg)$ is found to be collision-free, the shortest path to $\tau_{\vg}$ associated with $\vg$ has been found.
\end{cor}

\subsection{Edge Optimality}
\label{subsec:opt}

We analyze how the lazy lookahead allows to balance between the number of edge evaluations and rewiring operations.
We start by looking at the extreme case where there is an infinite lookahead ($\alpha = \infty$).
We define a general family of algorithms $\SP$ that solve the shortest-path problem and show in Lemma~\ref{lem:infinte_lookahead} that when  $\alpha = \infty$, no other algorithm in $\SP$ can perform fewer edge evaluations. 
We then show in Lemma~\ref{lem:larger_lookahead_no_greediness} that larger the lookahead, fewer the edge evaluations \ab will perform.

Recall that a shortest-path problem consists of 
a graph $\calG = (\calV,\calE)$, 
a lazy estimate of the weights $\hat{w}$,
a weight function~$w$ and 
start ($\vs$) and goal ($\vg$) vertices.
Given a shortest-path problem, let~$\SP$ be the family of shortest-path algorithms that build a shortest-path tree $\calT$ rooted at $\vs$.
Assume that for every shortest-path problem, there are no two paths in $\calG$ that have the same weight\footnote{To avoid handling tie-breaking in proofs. \ab does not require this assumption.}.

An algorithm \texttt{ALG} $\in \SP$ can only call the weight function~$w$ for an edge $e = (u,v)$ if $u \in \calT$.
When terminating, it must report the shortest path from $P^*_{\vg}$ and validate that no shorter path exists.
Thus for any other path $P$ from $\vs$ to $\vg$ with
$\hat{w}(P) < w(P^*_{\vg})$, \texttt{ALG} must explicitly test an edge $e \in P$ with $w(e)= \infty$.
Since \texttt{ALG} constructs a shortest-path tree, this will be the first edge on $P$ that is in collision.
 
Finally, an algorithm \texttt{ALG} $\in \SP$ is said to be \emph{edge-optimal} if for any other algorithm \texttt{ALG'} $\in \SP$, and any shortest-path problem, \texttt{ALG} will test no more edges than \texttt{ALG'}.

\begin{restatable}{lem}{lemmaTwo}
\label{lem:infinte_lookahead}
\ab with $\alpha = \infty$ 
is edge-optimal.
\end{restatable}


\begin{cor}
\lazySP is edge-optimal.
\end{cor}

\textVersion{}{
Considering \ab as an approximation of the optimal heuristic can provide a different perspective on how \lazySP is edge optimal. Consider consistent heuristics $h_1$ and $h_2$, such that~$h_1$ strictly dominates $h_2$ \ie, 
$$
h^*_{\calG}(v) \geq h_1(v) > h_2(v)~\forall v \in \calV,~v\neq \vg,
$$
where $h^*_{\calG}(v)$ is the optimal heuristic for a given graph $\calG$.

\begin{restatable}{lem}{lemmaHeuristic}
\label{lem:heuristic}
For every graph $\calG$ and lookahead $\alpha$, we have that $E_1 \subseteq E_2$, where $E_i$ denotes the set of edges evaluated by \ab with heuristic $h_i$, $i \in \{1,2\}$.
\end{restatable}


\begin{cor}
\lazySP is edge-optimal.
\end{cor}
}

%

%
\begin{restatable}{lem}{lemmaThree}
\label{lem:larger_lookahead_no_greediness}
For every graph~$\calG$ and every  $\alpha_1 > \alpha_2$,
we have that $E_1 \subseteq E_2$.
Here, $E_i$ denotes the set of edges evaluated by \ab with $\alpha = i$.
\end{restatable}
%

%
\subsection{Complexity}
\label{subsec:complexity}
In this section we analyse \ab with respect to the space (Lemma~\ref{lem:space}) and running time (Lemma~\ref{lem:time}) complexity.

\begin{figure*}[tb!]	
	\hspace*{-0.5em}
	\begin{subfigure}[h]{0.2\textwidth}
		\centering	
		\includegraphics[height=3cm]{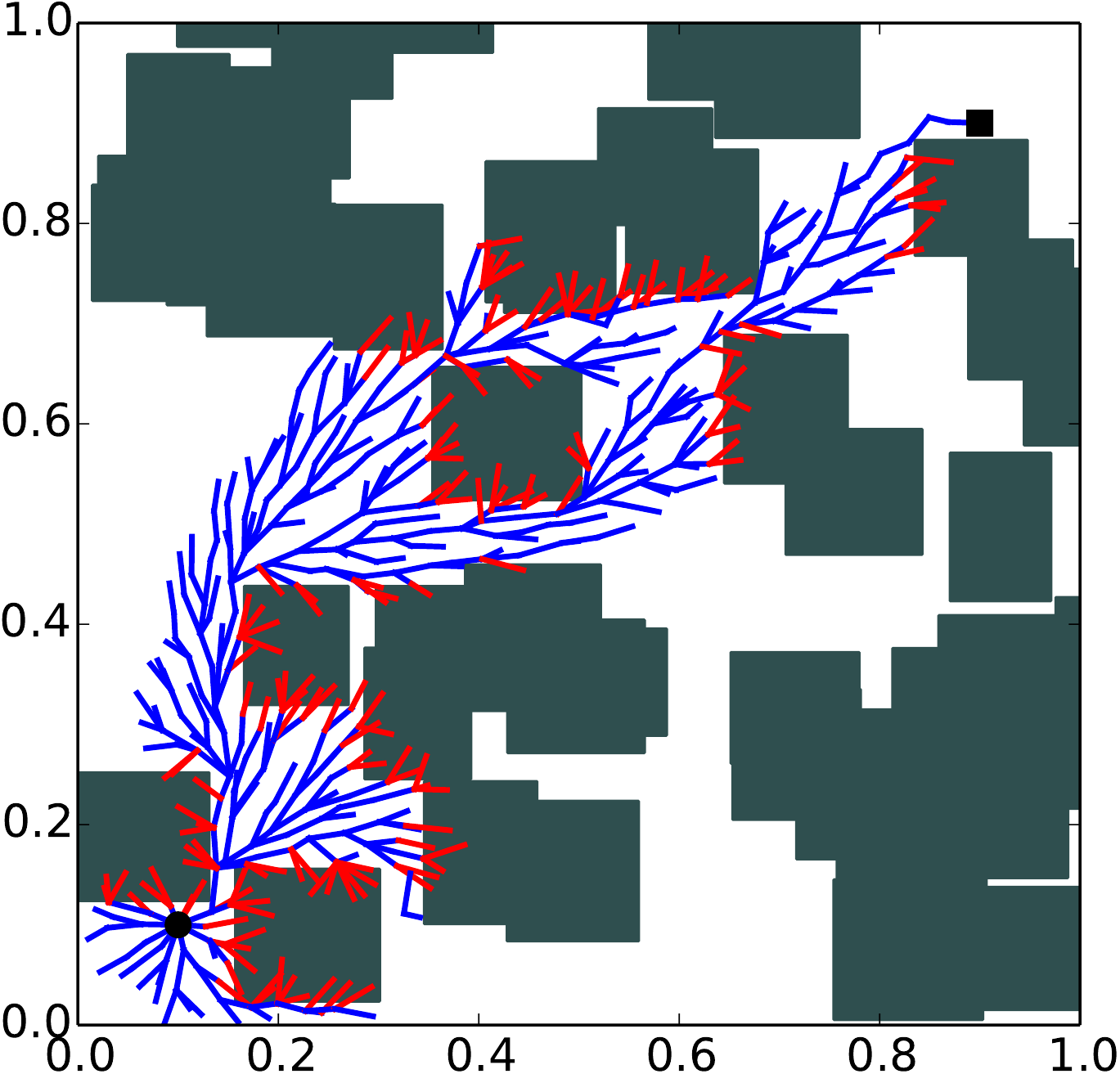} 
		\caption{\lwastar}\label{fig:lwaR}		
	\end{subfigure} \hspace*{-1.5em}
	\begin{subfigure}[h]{0.2\textwidth}
		\centering
		\includegraphics[height=3cm]{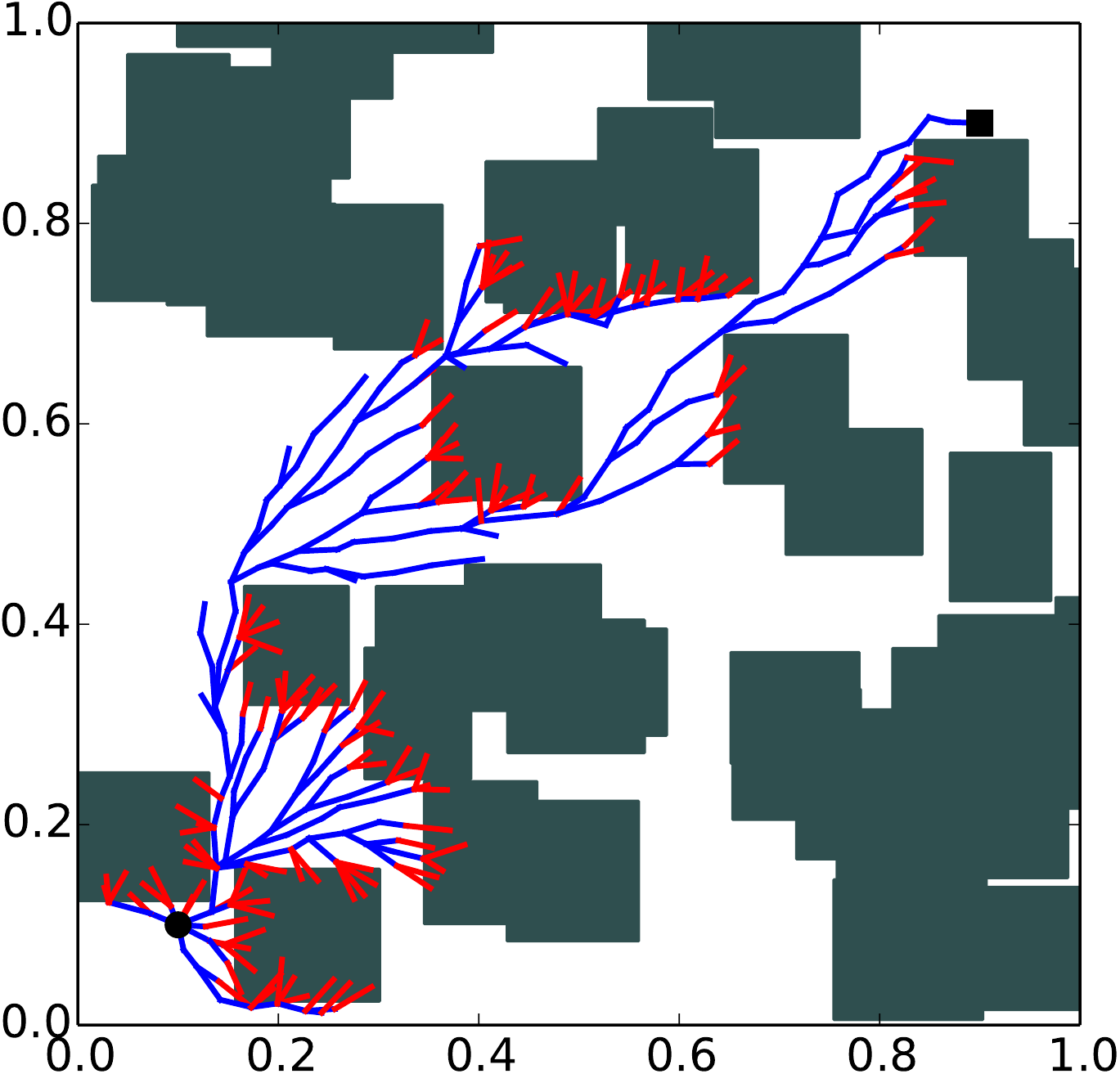} 
		\caption{\ab ($\alpha^*$)}\label{fig:optimalR}
	\end{subfigure} \hspace*{-1.5em}
	\begin{subfigure}[h]{0.2\textwidth}
		\centering
		\includegraphics[height=3cm]{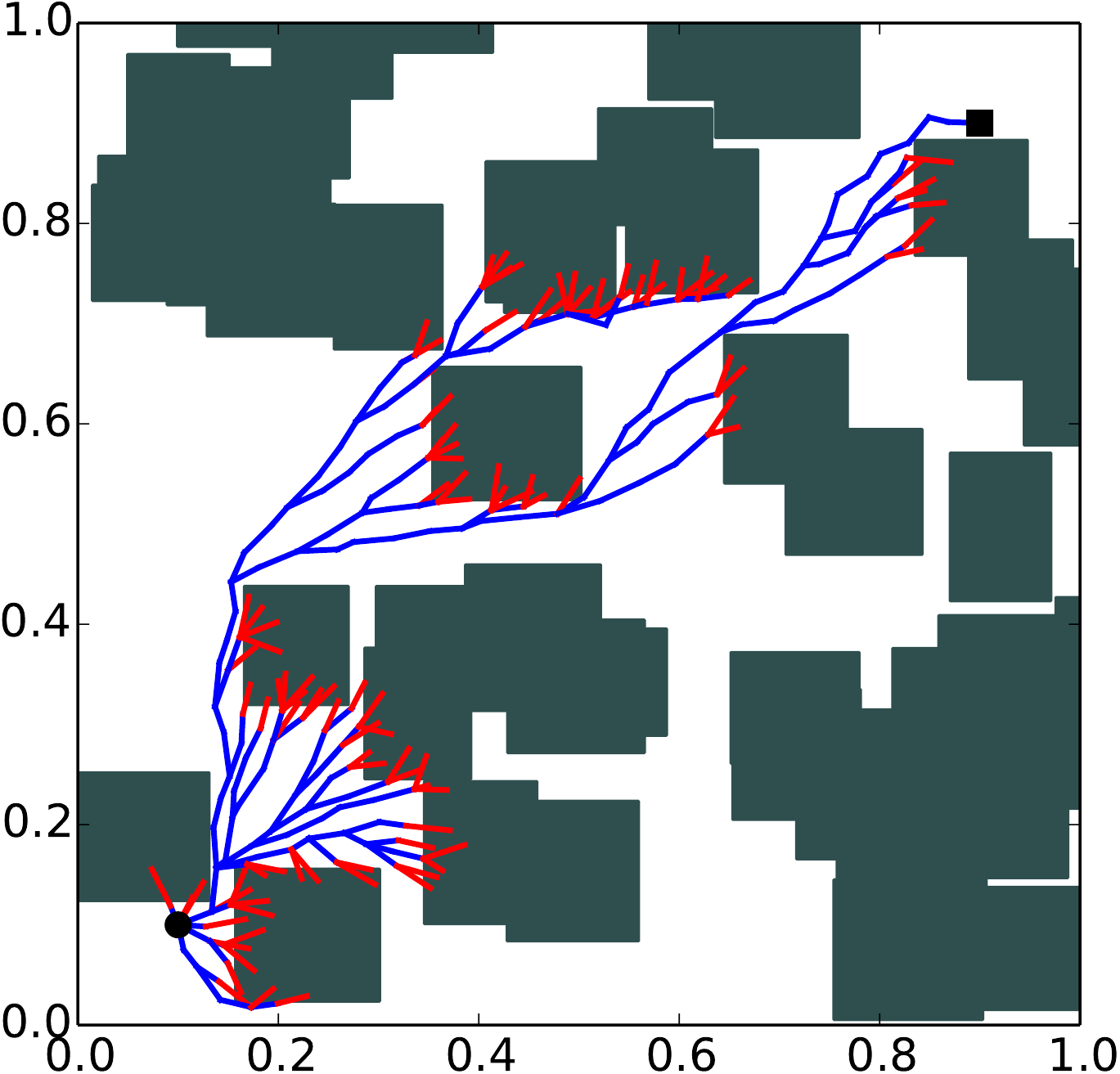} 
		\caption{\lazySP}\label{fig:lazySPR}
	\end{subfigure} \hspace*{-1.1em}
	\begin{subfigure}[h]{0.233\textwidth}
		\centering
		\includegraphics[height=3cm]{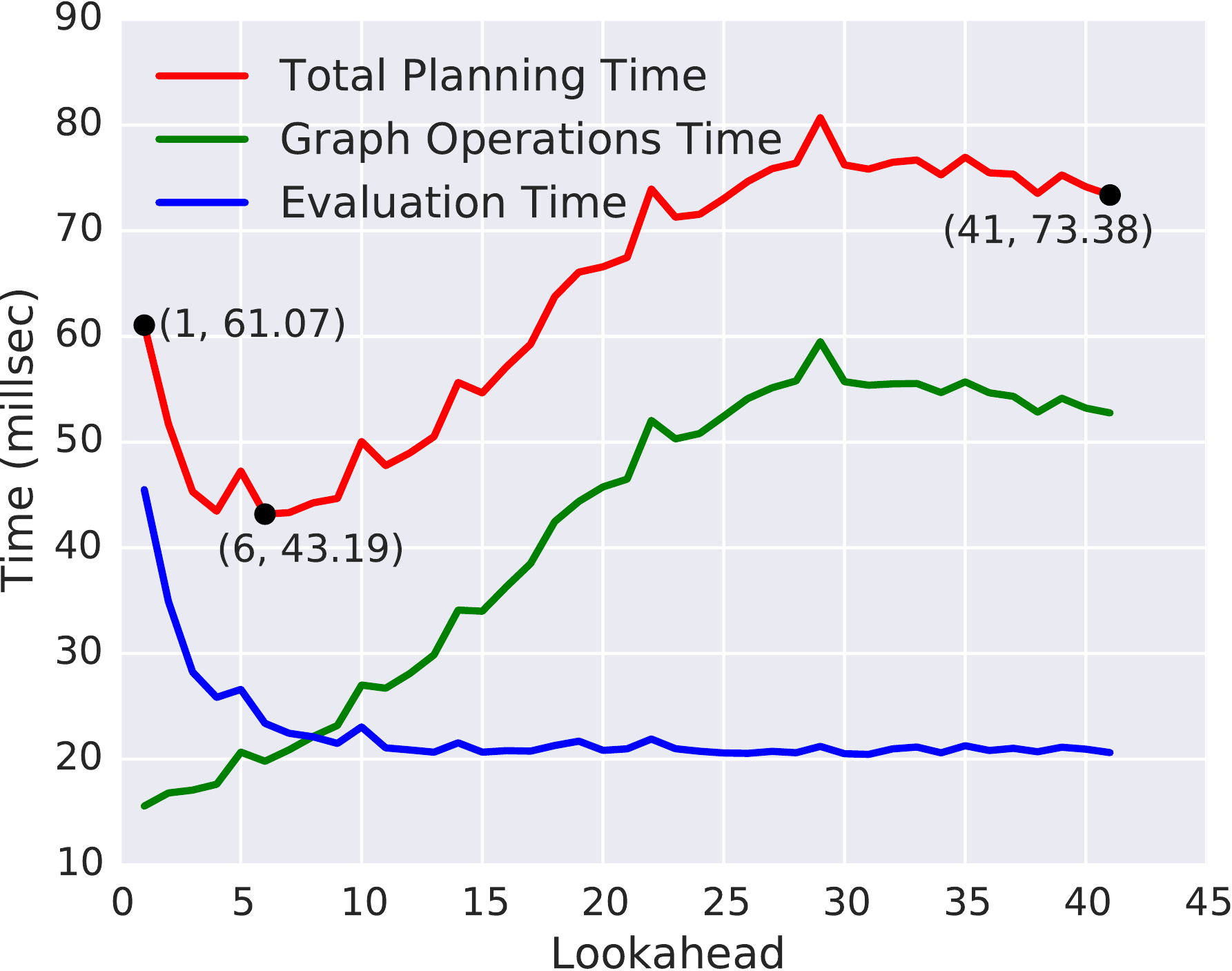} 
		\caption{}\label{fig:2DCompare}
	\end{subfigure} \hspace*{-1em}
	\begin{subfigure}[h]{0.233\textwidth}
		\centering
		\includegraphics[height=3cm]{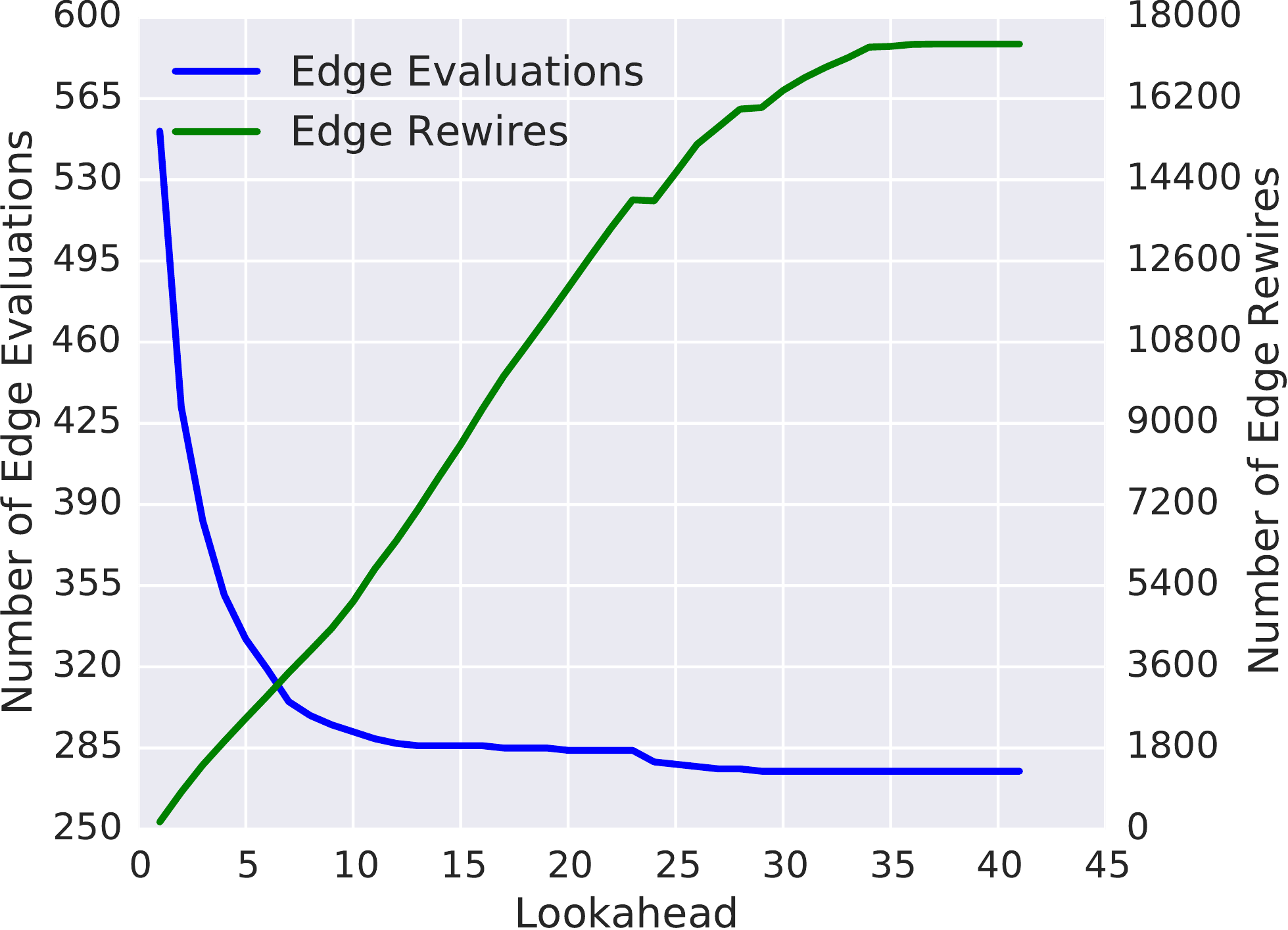} 
		\caption{}\label{fig:2DCompare2}
	\end{subfigure}
\caption{Visualization of edge evaluations by (\subref{fig:lwaR}) \lwastar, (\subref{fig:optimalR}) \ab with an optimal lookahead~$\alpha^*$, and (\subref{fig:lazySPR}) \lazySP. Source and target are $(0.1,0.1)$ and $(0.9,0.9)$, respectively. Edges evaluated to be in collision and free are marked red and blue, respectively. Computation times (\subref{fig:2DCompare}) and number of operations~(\subref{fig:2DCompare2}) of \ab as a function of the lookahead~$\alpha$. }
\label{fig:roadmapStrip}
\vspace{-3mm}
\end{figure*}
%
%

\begin{restatable}{lem}{lemmaSeven}
\label{lem:space}
The total space complexity of our algorithm is bounded by $O(n + m)$, where $n$ and $m$ are the number of vertices and edges in $\calG$, respectively.
\end{restatable}

%

\begin{restatable}{lem}{lemmaEight}
\label{lem:time}
The total running time of the algorithm is bounded by $O(n d^\alpha \cdot \log(n) + m)$, where $n$ and $m$ are the number of vertices and edges, $d$ is the maximal degree of a vertex and $\alpha$ is the lookahead.
\end{restatable}

\newpage
\section{Results}
\label{sec:experiments}
\begin{figure*}	
	\centering
	\begin{subfigure}[h]{0.22\textwidth}
		\centering
		\includegraphics[height=3.3cm]{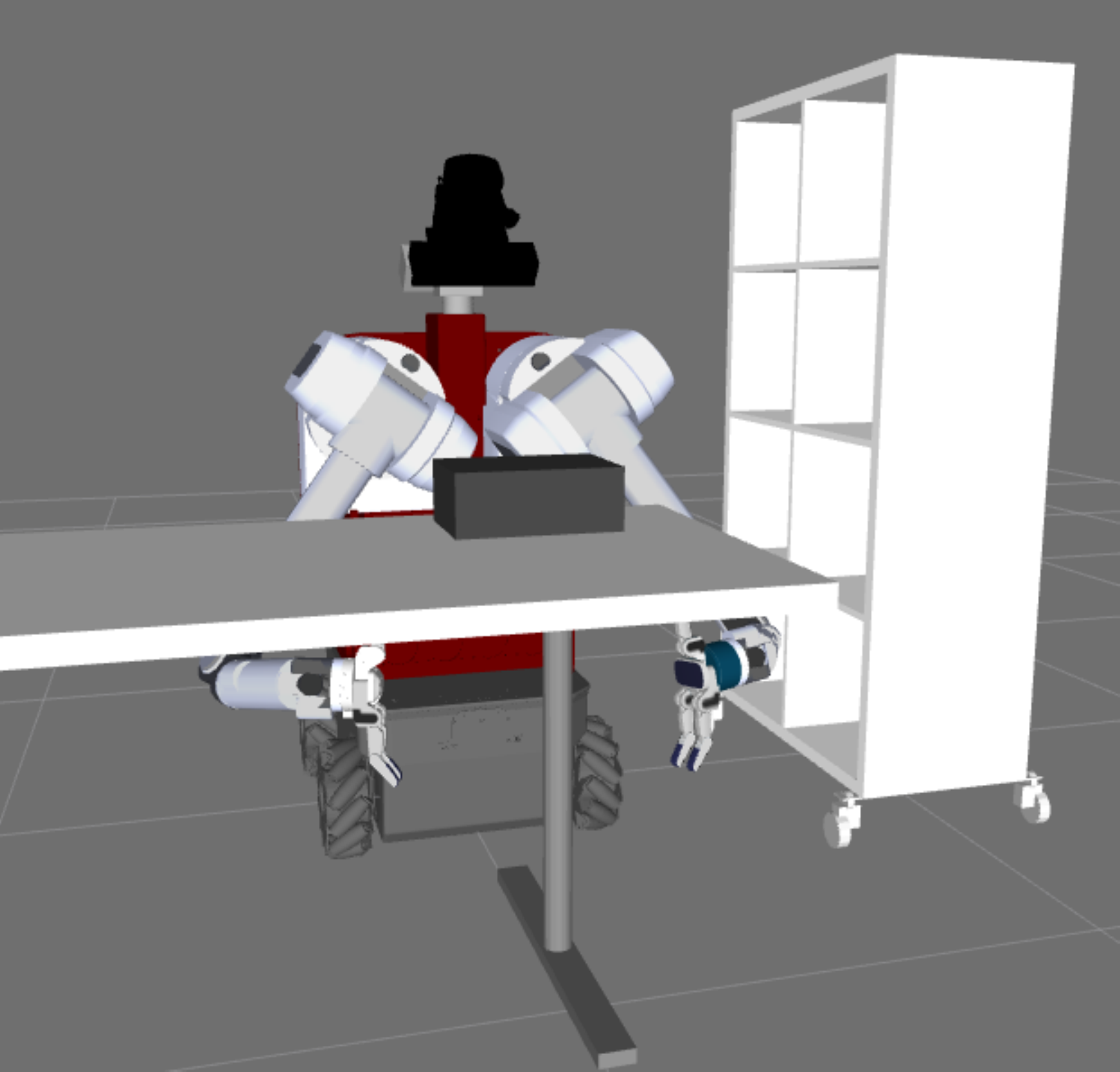}  
		\caption{}\label{fig:herba}		
	\end{subfigure}
	\begin{subfigure}[h]{0.22\textwidth}
		\centering
		\includegraphics[height=3.3cm]{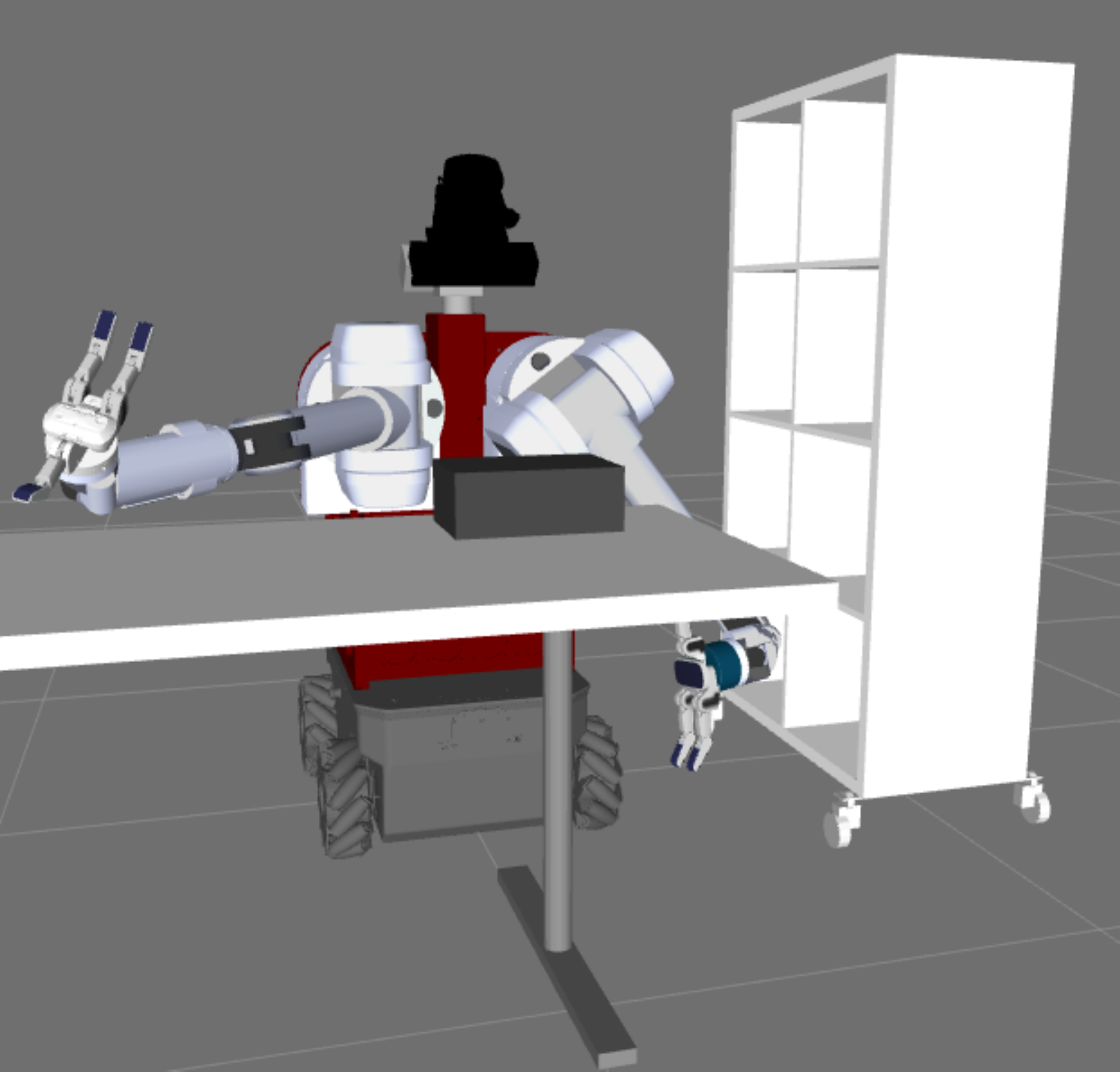}  
		\caption{}\label{fig:herbb}
	\end{subfigure}
	\begin{subfigure}[h]{0.22\textwidth}
		\centering
		\includegraphics[height=3.3cm]{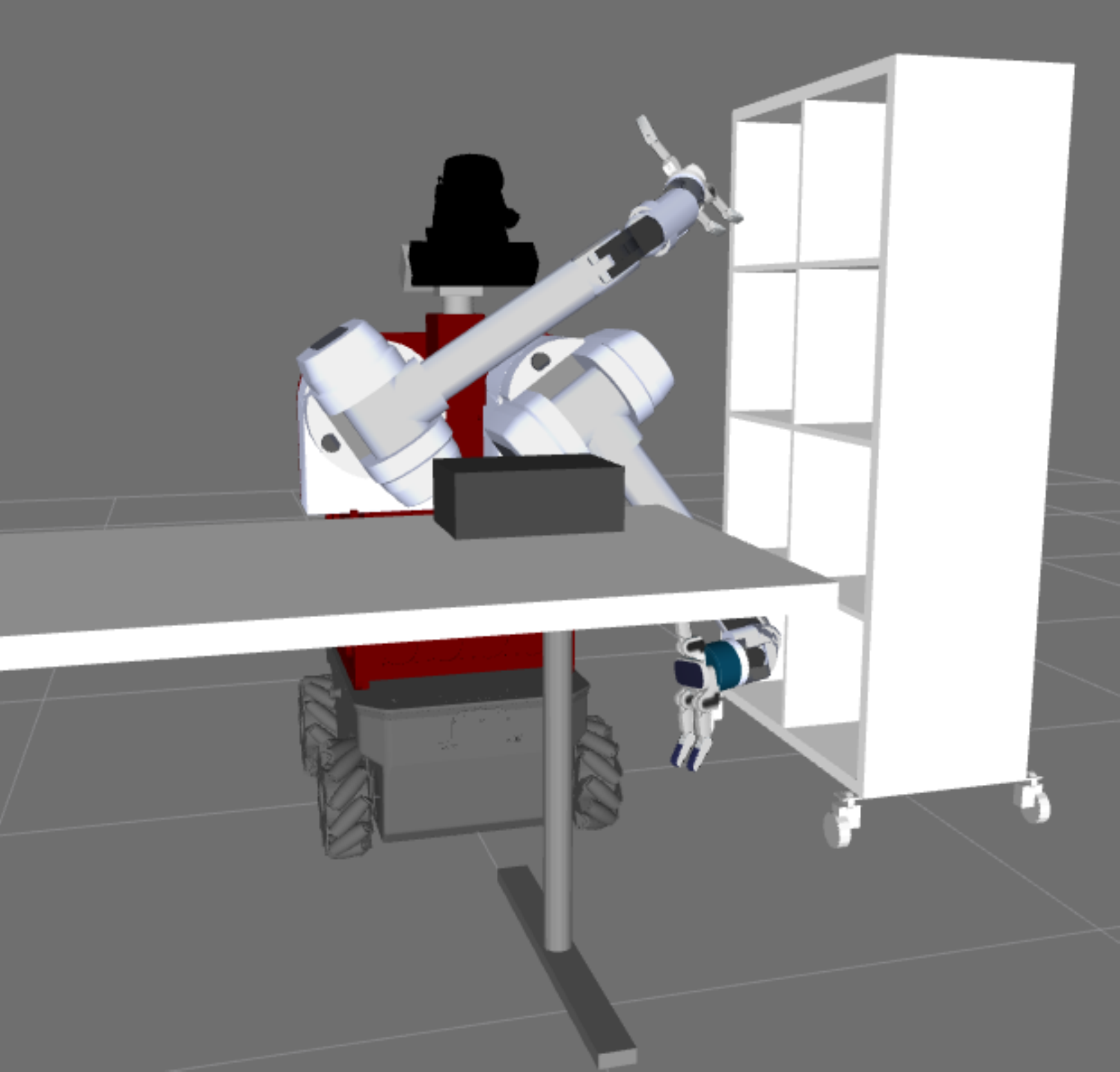}  
		\caption{}\label{fig:herbc}
	\end{subfigure}
	\begin{subfigure}[h]{0.27\textwidth}
		\centering
		  \includegraphics[height=3.4cm]{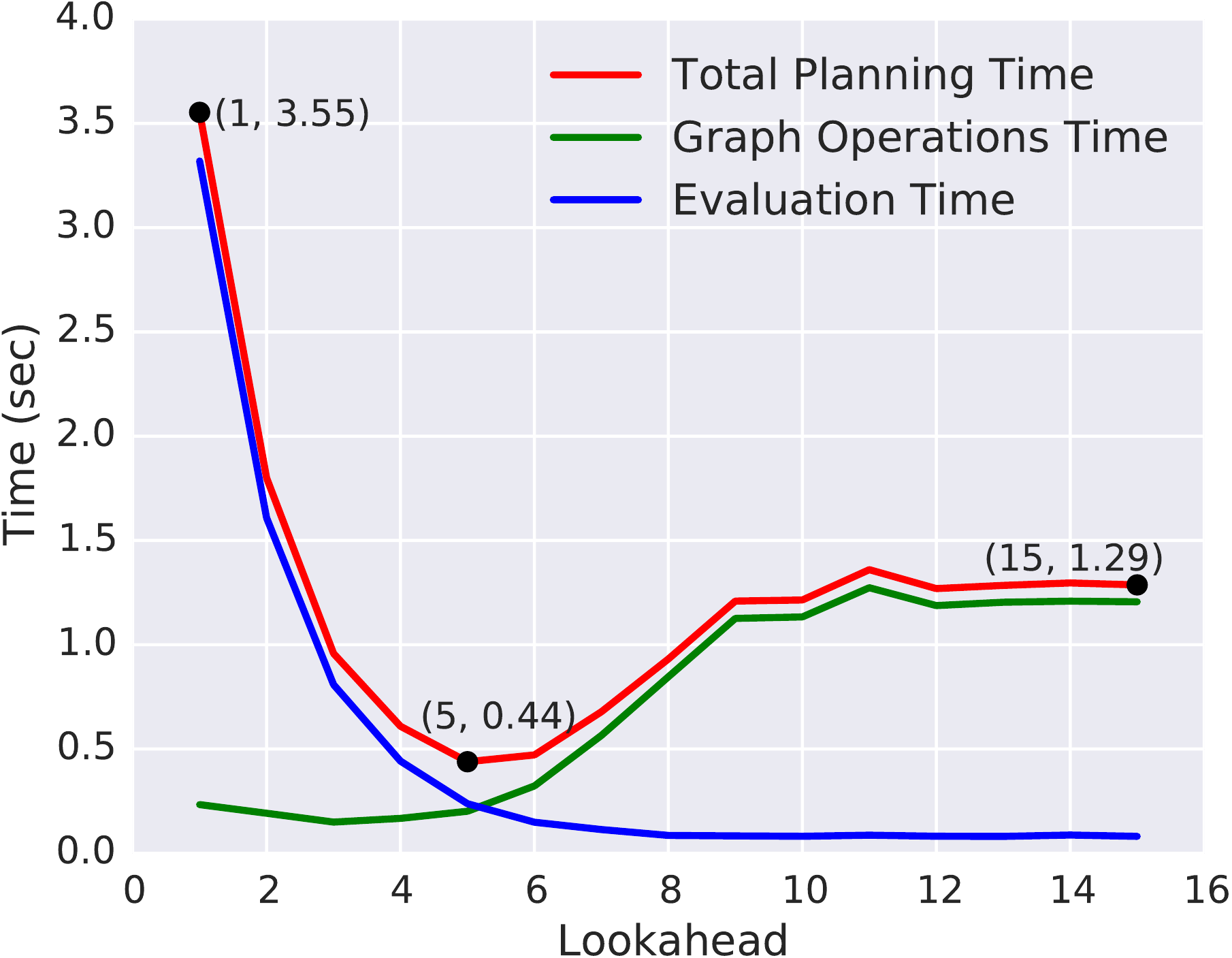}
		\caption{}\label{fig:herb5}
	\end{subfigure}
	\vspace{-3.5mm}
\caption{Manipulation experiments. (\subref{fig:herba}-\subref{fig:herbc}) HERB is required to reach into the bookshelf while avoiding collision with the table.
(\subref{fig:herb5}) Edge evaluation, rewiring and total planning time as a function of the lookahead. }
\label{fig:herbstrip}
	\vspace{-3.5mm}
\end{figure*}	
\begin{figure*}	
	\centering
	\begin{subfigure}[h]{0.33\textwidth}
		\centering
		\includegraphics[height=3.75cm]{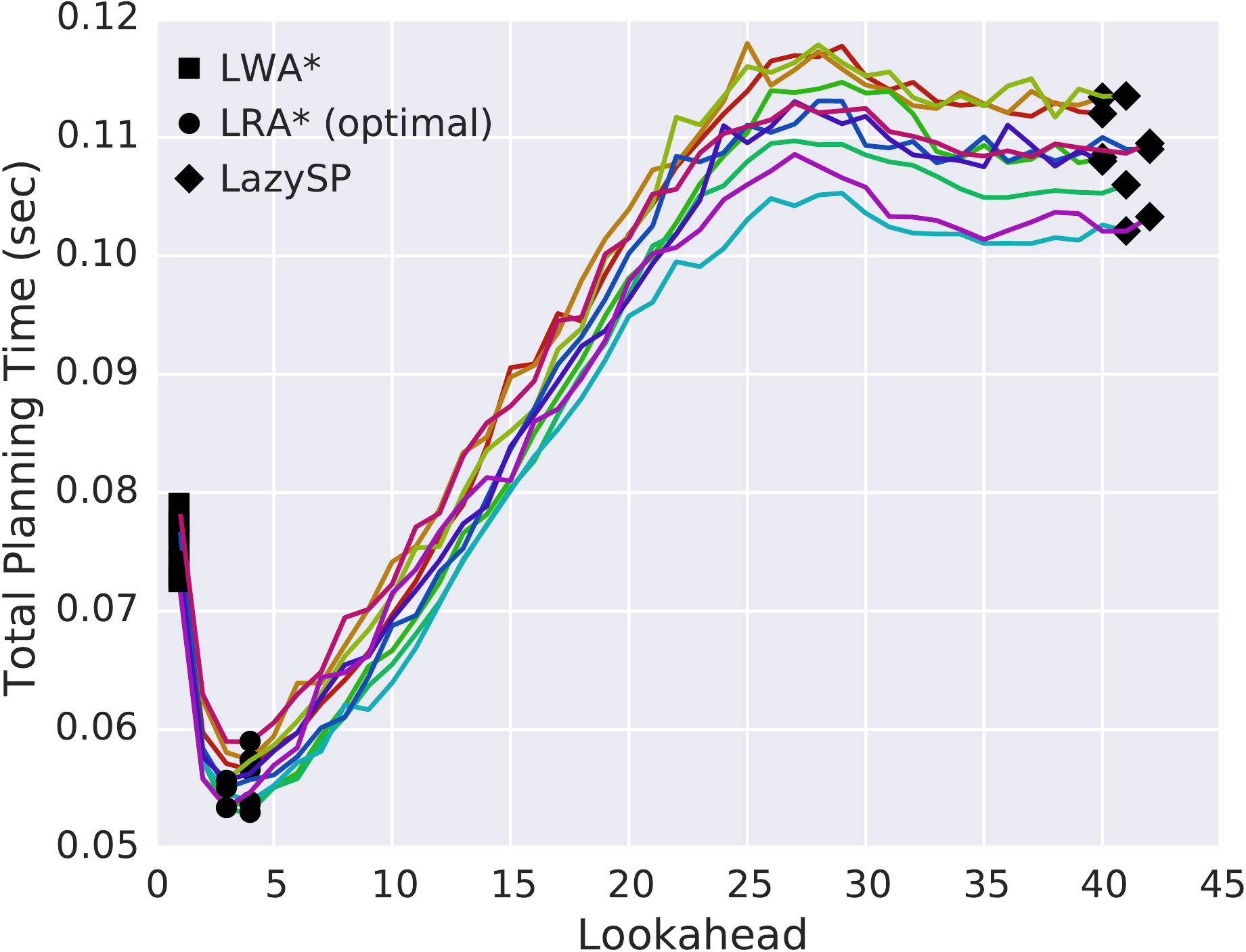} 
		\caption{$\mathbb{R}^2$ environments}\label{fig:instance2D}		
	\end{subfigure}
	\begin{subfigure}[h]{0.33\textwidth}
		\centering
		\includegraphics[height = 3.75cm]{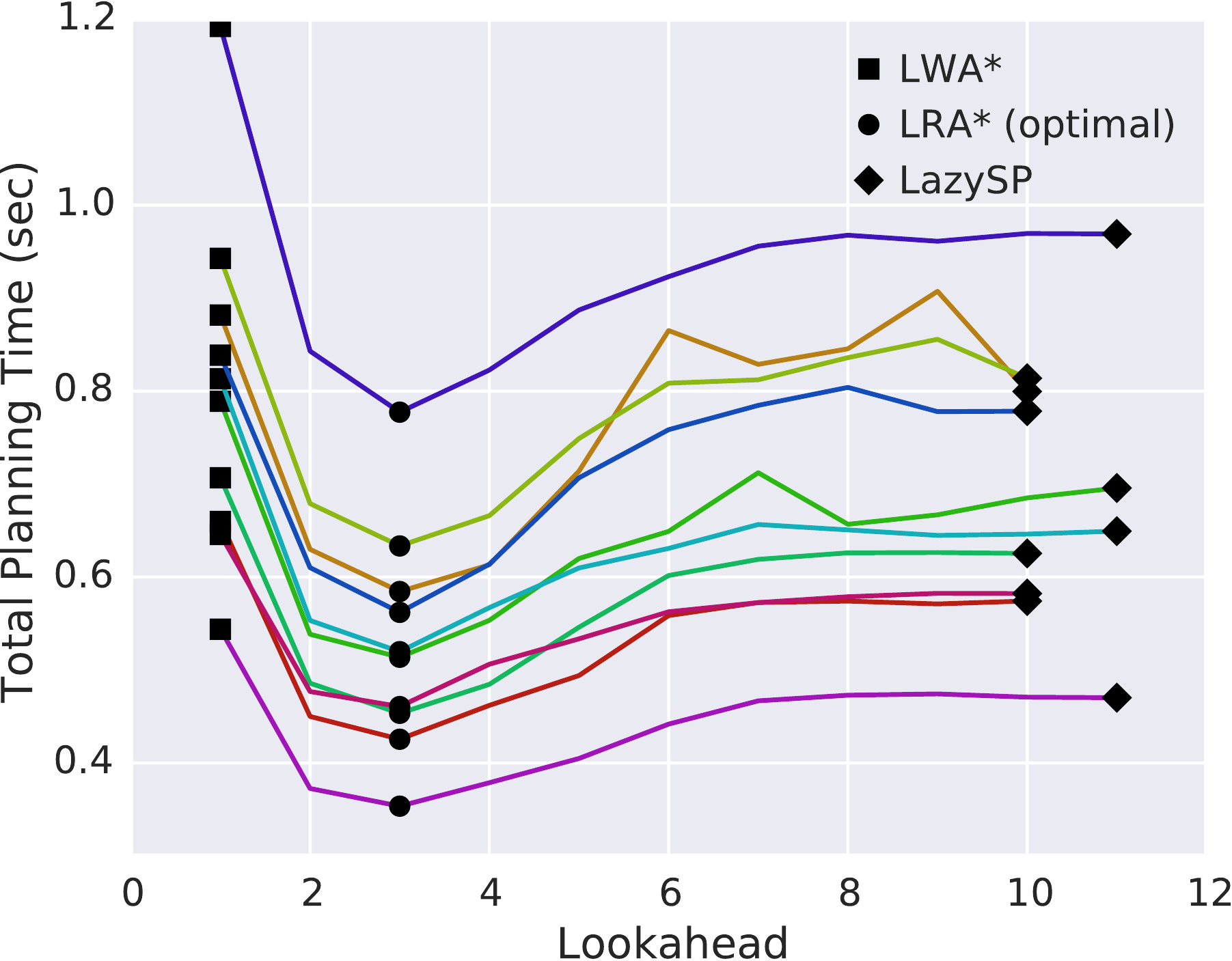} 
		\caption{$\mathbb{R}^4$ environments}\label{fig:instance4D}
	\end{subfigure}
	\begin{subfigure}[h]{0.33\textwidth}
		\centering
		\includegraphics[height=3.75cm]{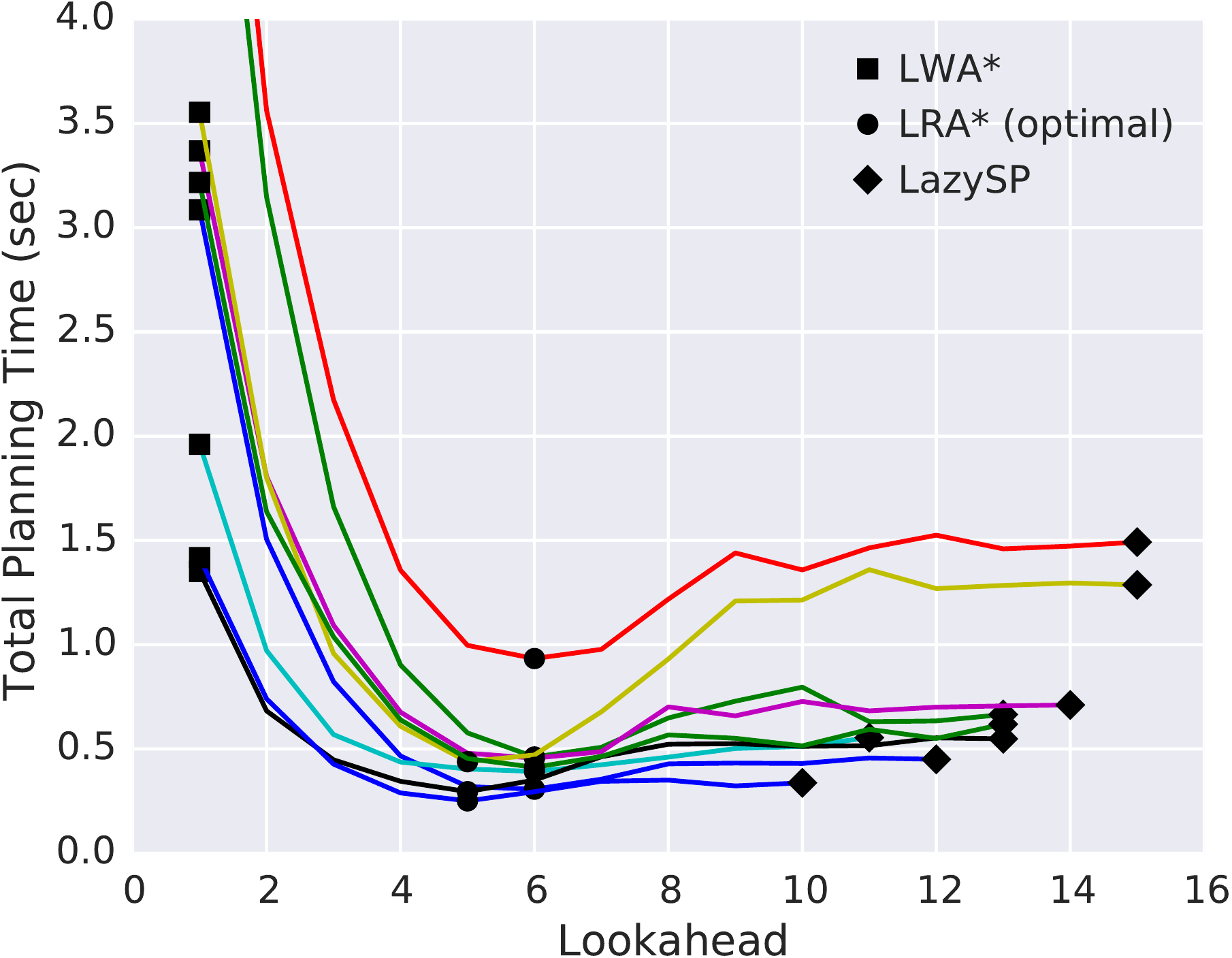} 
		\caption{Manipulation environments}\label{fig:allHerb}
	\end{subfigure}
	\caption{Planning time vs. lookahead for similar problems on different environments.}\label{fig:all}
	\vspace{-3.5mm}
\end{figure*}

In this section we empirically evaluate \ab.
We start by demonstrating the different properties of \ab as a \emph{family} of algorithms parameterized by $\alpha$.
Specifically, we show that to minimize the total planning time, an optimal lookahead~$\alpha^*$ exists (where $1 < \alpha^* < \infty$) that allows to balance between edge evaluation and graph operations.

We then continue to evaluate properties of the optimal lookahead $\alpha^*$. While choosing the exact lookahead value is out of the scope of the paper (see Sec.~\ref{sec:futureWork}), we provide general guidelines regarding this choice.

\subsection{Experimental setup}
We evaluated \ab on a range of planning problems in simulated random $\mathbb{R}^2$ and $\mathbb{R}^4$ environments as well as real-world manipulation problems on HERB \cite{Srinivasa2009}, a mobile manipulator with 7-DOF arms. 
	We implemented the algorithm using the Open Motion Planning Library (OMPL) \cite{ompl}\footnote{Simulations were run on a desktop machine with 16GB RAM and an Intel i5-6600K processor running a 64-bit Ubuntu 14.04.}. Our source code is publicly available and can be accessed at \url{https://github.com/personalrobotics/LRA-star}.
	
\subsubsection{Random environments}
We generated 10 different random environments  for $\mathbb{R}^2$ and $\mathbb{R}^4$. 	
For a given environment, we consider 10 distinct random roadmaps for a total of 100 trials for each dimension. 
Each roadmap was constructed as follows:
The set of vertices were generated in a unit hypercube using Halton sequences \cite{Halton64}, which are characterized by low dispersion. 
The vertex positions were then offset by uniform random values to generate distinct roadmaps.
An edge existed in the graph between every pair of vertices whose Euclidean distance is less than a predefined threshold~$r$.
The value~$r$ was chosen to ensure that, asymptotically, the graph can capture the shortest path connecting the start to the goal~\cite{JSCP15}.
The number of vertices was chosen such that the roadmap contained a solution.
Specifically, it was 2000 for  $\mathbb{R}^2$ and 3000 for  $\mathbb{R}^4$.

The source and target were set to $(0.1, 0.1, \ldots, 0.1)^d$ and $(0.9, 0.9, \ldots, 0.9)^d$, respectively, with $d \in \{ 2, 4\}$.
For the 2D environments, the obstacles were a set of axis-aligned hypercubes that occupy 70\% of an environment to simulate a cluttered space. 
One such randomly-generated environment is shown in Fig. \ref{fig:roadmapStrip} along with the edges evaluated by \lwastar, \lazySP and \ab with an optimal lookahead. 
For the 4D environments, we chose a maze generated similar to the recursive mazes defined by \citeauthor{JSCP15}. 
The choice of such a maze in $\mathbb{R}^4$ is motivated by the fact that it is inherently a \emph{hard} problem to solve, since many lazy shortest paths need to be invalidated before a true shortest path is determined by the planner. 
A detailed discussion about the complexity of the recursive maze problem is found in \cite{JSCP15}.

\subsubsection{Manipulation}
Our manipulation problems simulate the task of reaching into a bookshelf while avoiding obstacles such as a table.
We consider 10 different roadmaps, each with 30,000 vertices constructed by applying a random offset to the 7D Halton sequence. 
Two vertices are connected if their Euclidean distance is less than $r =1.3$ radians.
These choices are similar to the simulated $\mathbb{R}^n$ worlds, where we choose $r$ using the bounds provided by \citeauthor{JSCP15} and enough vertices such that we are ensured a solution exists on the roadmap.
Fig. \ref{fig:herbstrip} illustrates the environment and the planning problem considered.

\subsection{Properties of \ab}
Figures \ref{fig:roadmapStrip} and \ref{fig:herbstrip} visualize the search space for our simulated~$\mathbb{R}^2$ environments as well as our manipulation environment.
For both settings, we ran \ab with a range of lookahead values.

Notice that the number of  edge evaluations as a function of the lookahead is a monotonically  decreasing function (Fig.~\ref{fig:2DCompare2} and Lemma~\ref{lem:larger_lookahead_no_greediness}).
However, the time spent on edge evaluations (Fig.~\ref{fig:2DCompare}) is not monotonic. This is because the time for evaluating an edge depends on the edge length and if it is in collision. Having said that, the overall trend of this plot decreases as the lookahead increases.
In addition, the time spent on rewiring (Fig.~\ref{fig:2DCompare} and~\ref{fig:herb5} roughly increases with the lookahead.
Following these two trends we find that, in both experiments, an intermediate lookahead does indeed balance edge evaluations and graph operations.
This, in turn reduces the overall planning time.
\textVersion{For additional experiments, see supplementary material.}{}

%
%
%

\subsection{Lazy Lookahead and Dynamic Heuristic}
\label{sec:dynamicResult}
We consider every border node in the $\alpha$-band to be associated with a dynamic heuristic that extracts information about the graph structure up to $\alpha$ edges away, and the static heuristic associated with a frontier node exactly $\alpha$ edges away. 

As the lookahead increases from one to infinity, \ab lazily obtains an increasing amount of information about the underlying graph structure. This information, encoded in the dynamic heuristic, allows \ab with a larger lookahead to search a smaller region of the graph and evaluate at most as many edges as \ab with a smaller lookahead (Lemma \ref{lem:larger_lookahead_no_greediness}).

\begin{figure}[tb]
  \centering
  	\includegraphics[height=4cm]{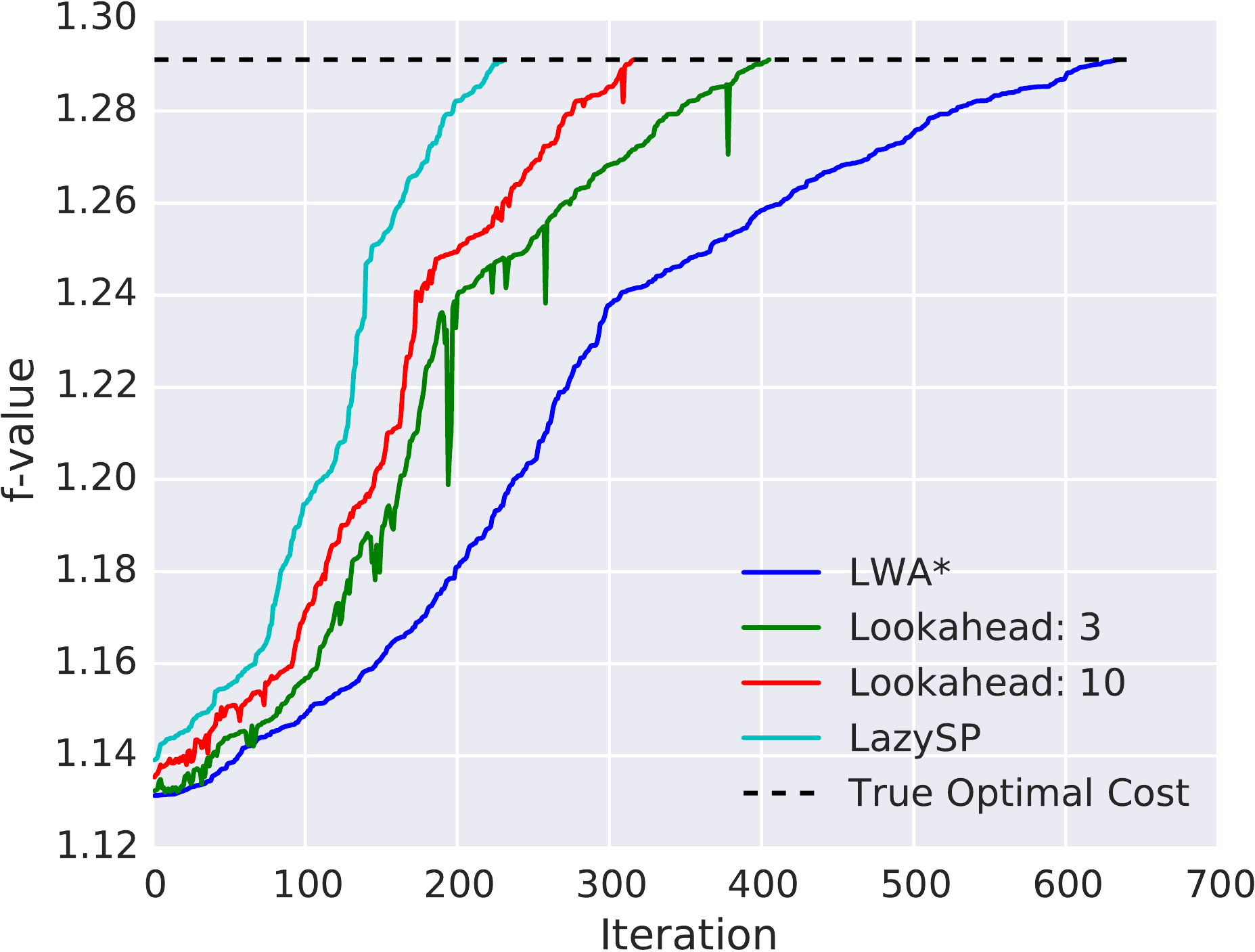}
  \caption{
	The f-value of the top node popped from $\Qfront$ every iteration of LRA* for various lookaheads.
  	}
   	\label{fig:fValue}
\end{figure}

In Fig. \ref{fig:fValue} we plot, for a 2D problem, the f-value\footnote{Recall f-value is the key used to order nodes in $\Qfront$. For a given node, it is the sum of estimated cost-to-come and cost-to-go.} of the top node in $\Qfront$ for each iteration of the algorithm. Note that each iteration of the algorithm corresponds to an edge evaluation.
\lazySP converges to the optimal f-value in the fewest number of iterations while
\lwastar evaluates the most number of edges since it is least informed amongst the family of \ab algorithms.

We observe in Fig. \ref{fig:fValue} that for \ab with an intermediate  lookahead $1 < \alpha < \infty$, the f-values of the nodes popped from the $\Qfront$ do not monotonically increase. This is attributed to the fact that the dynamic heuristic does not necessarily capture the true underlying graph structure when there are edges in collision. When an edge is found to be in collision, the $\alpha$-band can potentially shrink closer to the source based on the budget available. This generates new leaf nodes closer to the source which can possibly be characterized by a lower f-value compared to the f-value of the node popped in the earlier iteration. 
This does not occur in \lwastar since it has a dynamic heuristic over one step.
In case of a collision, the heuristic is trivially updated after removing the edge from the graph. 
\lazySP does not exhibit this behavior either since it always pops the goal node and the f-value of a \emph{particular} node in the graph can only monotonically increase.

\subsection{Properties of optimal lookahead $\alpha^*$}
While determining how to choose the lookahead value for a specific problem instance is beyond the scope of this paper (see Sec. \ref{sec:futureWork}), we provide some insight on some properties of optimal lookahead $\alpha^*$.
In Fig.~\ref{fig:all} we plotted the planning time as a function of the lookahead for different random instances.
We observe two phenomena:
(i)~the value of the optimal lookahead $\alpha^*$ has a very small variance when considering similar environments.
Thus, if we will face multiple problems on a specific type of environment, it may be beneficial to run a preprocessing phase to estimate $\alpha^*$.
(ii)~As the dimension increases, the relative speedup, when compared to \lazySP diminishes.
We conjecture that this is because the cost of edge evaluation increases with the complexity of the robot (namely, with the dimension).



\section{Future Work}
\label{sec:futureWork}

\subsubsection*{Setting the lazy lookahead}
Our formulation assumed that the lazy lookahead~$\alpha$ is fixed and provided by the user.
In practice, we would like to automatically find the value of~$\alpha$ and, possibly, change its value through the running time of the algorithm.
This is especially useful when the search algorithm is interleaved with graph construction---namely, when vertices and edges are incrementally added to~$\calG$ (see, e.g.,~\cite{GSB15}).

\subsubsection*{Non-tight estimates of edge weights}
In this paper we assumed that $\hat{w}$ tightly estimates the true cost $w$ (see Eq.~\ref{wEst}), however it can be easily extended to take into account non-tight estimates.
Once an edge $(u,v)$ is evaluated, if its true cost is larger than the estimated cost, the entire subtree rooted at~$v$ may need to be rewired to potentially better parents.
Our immediate goal is to run our algorithm on such settings.
 
\subsubsection*{Alternative budget definitions and optimization criteria}
In this paper, we defined the budget and the optimization criteria in terms of number of unevaluated edges and path length, respectively.
However, the same approach can be used for alternative definitions.
For example, we can define the budget in terms of the \emph{length} of the unevaluated path.
This definition is somewhat more realistic since the computational cost of evaluating an edge is typically proportional to its length.
A different optimization criteria that we wish to consider is minimizing the \emph{expected}  number of edges checked given some belief over the probability that edges are collision-free.
This can be further extended to balance between path length (which is a proxy for the execution time) and number of edge evaluations (which is a proxy for the planning time). Here,  we need to consider some combination of path length and probability of being collision-free as the optimization criteria.

\textVersion{}{
\subsubsection*{Implementing \ab using advanced priority queues}
Recall that \ab maintains only the best path to reach each frontier node in the search tree $\calT$ at every point in time.
This is done to avoid an exponential increase in the memory footprint of \ab with respect to the lazyiness value~$\alpha$.
Consequently, all the priority queues that we use require the ability to update the key of elements in the queue.
In contrast, \lwastar, which may maintain several paths to the same node, only requires a priority queue that supports inserting elements and popping the element with the minimal key.
Interestingly, such priority queues allow for a significant speedup in \textsf{A*}-like algorithms~\cite{chen2007priority}.
\ab can easily be modified to maintain all unevaluated paths to frontier nodes.
If the lazyiness value~$\alpha$ is relatively small, then the use of fast priority queues that do not support key updates may significantly speed up the algorithm's running time. 
}

\section{Acknowledgements}
{
The authors would like to thank Shushman Choudhury, previously at Personal Robotics Lab, now at Stanford University, for his valuable insights and discussions in the development of this work.
}
\textVersion{}
{
\begin{appendices}
\section{Algorithm Description}
\label{app:code}
In this section we provide detailed pseudo-code of \ab.
We start in Alg.~\ref{alg:main} which details the main loop used by \ab.

\subsection{Main Algorithm}
\algrenewcommand\algorithmicindent{.8em}
\begin{algorithm}[tb]
\caption{$\texttt{LRA}^*$ {($\calG,~\vs,~\vg,~\alpha$)}}
\label{alg:main}	
\begin{algorithmic}[1]
\State $\tau_{\vs} = (\vs,~\text{NIL},~0,~0,~0)$; \hspace{2mm} $\calT.\texttt{insert}(\tau_{\vs})$
\ForAll{$v \in \calV,~v \neq \vs $}
	\State $\tau_v = (v, \text{NIL}, \infty, \infty, \infty)$ \Comment{Initialization}
\EndFor
\State $\Qupdate, \Qfront, \Qrewire \gets \emptyset$ 
\State $\Qextend.\texttt{push}(\tau_{\vs})$
\State \texttt{extend\_$\alpha$\_band}$()$ 
\Comment{populate $\Qfront$}

\vspace{2mm}

\While{$\Qfront \neq \emptyset$} 
	\State $\tau \gets \Qfront$.\texttt{pop}()
	\State $P \gets P[\tau]_{\rm tail} $
	\Comment{extract path from border node to $\tau$}
	\State \texttt{evaluate\_path}$(P)$ \Comment{populate $\Qupdate, \Qrewire$}
	\If{$\tau_{\vg} \in \Qupdate$}
		\State \textbf{return} $P[\tau_{\vg}]$
		\Comment{return path from $\vs$ to $\vg$}
	\EndIf 
	\State \texttt{update\_$\alpha$\_band}$()$ \Comment{populate $\Qextend$}
	\State \texttt{rewire\_$\alpha$\_band}$()$ \Comment{populate $\Qextend$}
	\State \texttt{extend\_$\alpha$\_band}$()$ \Comment{populate $\Qfront$}

\EndWhile

\State \textbf{return} failure
\end{algorithmic}
\end{algorithm}
We start (lines~1-3) by adding a node corresponding to the source into the search tree and initializing all other nodes.
We continue  (lines~4-5) by initializing all the priority queues used by the algorithm.
The algorithm then extends the $\alpha$-band (line~6 and Alg.~\ref{alg:extend}) which computes the frontier nodes stored in \Qfront.

From this point, the algorithm iterates between choosing a path  (defined by the frontier node with the minimal cost-to-come) and evaluating the first edge along this path (lines~8-10).
If the target was reached, the algorithm terminates (lines~11-12).
This evaluation also adds nodes to either~\Qupdate and \Qrewire, depending if the edge evaluated was collision free or not.
A node will be added to \Qupdate because the budget of all nodes in its subtree needs to be updated.
Similarly, a node will be added to \Qrewire because its current parent is in collision and all nodes in its subtree need to be rewire.
Thus, if the target was not reached the algorithm 
updates vertices in the $\alpha$-band (line 13 and Alg.~\ref{alg:update}), 
rewires vertices if needed (line 14 and Alg.~\ref{alg:rewire}),
and re-extends the $\alpha$-band
(line 15 and Alg.~\ref{alg:extend}).

\subsection{Path Evaluation}

\begin{algorithm}[tb]
\caption{\texttt{evaluate\_path} \small{($P = \left(\tau_0,\tau_1,\ldots,\tau_\alpha \right)$)}}
\label{alg:evaluatePath}
\begin{algorithmic}[1]

	\State $e \gets (u[\tau_0],u[\tau_1])$
	\If{$w(e) = \hat{w}(e)$} \Comment{Expensive check}
		\State $\tau_1 \gets (u_1, u_0, c[\tau_0] + w(e), 0, 0)$ \Comment{Update node}
		\State $\Qupdate.\texttt{push}(\tau_{1})$
		\If{$u_1 = \vg$}
			\State \textbf{return}
		\EndIf
	\Else \Comment{Invalid edge: rewiring required}
		\State $\calE.\texttt{remove}(e)$ \Comment{Remove edge from graph}
		\State $\calT_{\rm rewire} \gets \calT_{\rm sub}(\tau_{1})$ 
		\State \textbf{return}
	\EndIf	  

\State \textbf{return}
\vspace{2mm}
\end{algorithmic}
\end{algorithm}
When a path is chosen, \ab evaluates the first edge along the tail of the path (Alg.~\ref{alg:evaluatePath}, line~2).
If the edge is collision free (lines~3-6), its entries are updated and its target is pushed into \Qupdate. 
This will later be used in Alg.~\ref{alg:update}) to update all nodes in its subtree in a systematic manner.
If the edge is in collision (lines~8-9), the corresponding edge is removed from the graph and the subtree rooted at the target vertex of the edge is set for rewiring.
Similar to the previous case this will later be used in Alg.~\ref{alg:rewire}) to rewire  all nodes in the subtree in a systematic manner.

\subsection{Updating $\alpha$-band}
\setlength{\textfloatsep}{15pt}
\begin{algorithm}[tb]
\caption{\texttt{update\_$\alpha$\_band} $()$}
\label{alg:update}
\begin{algorithmic}[1]
\While{$\Qupdate \neq \emptyset$}

	\State $\tau \gets \Qupdate.\texttt{pop}()$
	
	\State $T_\mathrm{succ} \gets \{ \tau' \in \calT \vert p[\tau'] = \tau \}$	
	\If{$T_\mathrm{succ} = \emptyset$} \Comment{Leaf node with budget}
		\State $\Qextend.\texttt{push}(\tau)$
		\State \textbf{continue}
	\EndIf

	\ForAll{$\tau' \in T_\mathrm{succ}$}
		\If{$b[\tau'] = \alpha$} \Comment{Cleanup queue}
			\State $\Qfront.\texttt{remove}(\tau')$
		\EndIf
		\State $\tau' \gets (u[\tau'], u[\tau], c[\tau], \ell[\tau] + \hat{w}(u[\tau], u[\tau']), b[\tau]+1 )$			
		\State $\Qupdate.\texttt{push}(\tau')$
	\EndFor
\EndWhile
\State \textbf{return}
\end{algorithmic}
\end{algorithm}
In Alg.~\ref{alg:evaluatePath}, when an edge $(u,v)$ has been evaluated to be collision-free, node $\tau_v$ associated with vertex $v$ is updated and \Qupdate is populated with the updated node. 
This update to $\tau_v$ needs to be cascaded to all the vertices in its subtree in a breadth-first search manner such that the parent node is updated before the child node. 
In every iteration of Alg.~\ref{alg:update}, the node with minimal key is popped from $\Qupdate$ (line 2) and its successors in the subtree are obtained (line~3). 
If the set of successors is empty, this implies that the node is a leaf node and is hence pushed into \Qextend (lines~4-6). 
Otherwise, each of the successor nodes is updated using the parent node entries and pushed into \Qupdate (lines~7-11). 
Note that leaf nodes belonging to this subtree are removed from \Qfront as their budget is updated to less than $\alpha$ (lines~8-10).

\subsection{Rewiring $\alpha$-band}

\setlength{\textfloatsep}{15pt}
\begin{algorithm}[tb]
\caption{\texttt{rewire\_$\alpha$\_band} ($\Qrewire$)}
\label{alg:rewire}
\begin{algorithmic}[1]
\small
\ForAll{$\tau \in \calT_{\rm rewire}$}	\Comment{Assign keys to nodes in subtree}
	\State $\calT.\texttt{remove}(\tau)$ 
	\State $\tau = (u[\tau], \text{NIL}, \infty, \infty, \infty)$ 
	\If{$\tau \in \Qfront$}
		\State $\Qfront.\texttt{remove}(\tau)$
	\EndIf
	\vspace{2mm}

	\State $\calS_{\rm parents} \gets \{\tau' \in \calT$  s.t. $\ (u[\tau'],~u[\tau]) \in \calE,~b[\tau'] < \alpha\}	$
	\vspace{1mm}
	\ForAll{$\tau' \in \calS_{\rm parents} - \{\calT_{\rm rewire} \cup \Qextend \cup \tau_{\vg}$\}}		
		\If{$c[\tau] + \ell[\tau] > c[\tau'] + \ell[\tau'] + \hat{w}(u[\tau'],u[\tau])$}
			\vspace{1mm}
			\State $\tau \gets (u[\tau], u[\tau'], c[\tau'], \ell[\tau'] + \hat{w}(u[\tau'], u[\tau]), b[\tau'] + 1)$ 
		\EndIf
	\EndFor
	\vspace{1mm}
	\State $\Qrewire$.\texttt{push}($\tau$)
\EndFor

\vspace{2mm}

\While{$\Qrewire \neq \emptyset$}	\Comment{Rewire}
	\State $\tau \gets \Qrewire.\texttt{pop}()$
	\If{$p[\tau] = \text{NIL}$}
		\State \textbf{continue}
	\EndIf
	\State $\calT.\texttt{insert}(\tau)$
	\If{$b[\tau] = \alpha~\text{\textbf{or}}~u[\tau] = \vg$}
		\State $\Qfront$.\texttt{push}($\tau$)
		\State \textbf{continue}
	\EndIf
	\If{$b[\tau] < \alpha$} \Comment{Note $u[\tau] \neq \vg$}
		\State $\Qextend$.\texttt{push}($\tau$)
	\EndIf
	
	\ForAll{$v \in \calV \ s.t. \ (u[\tau],v) \in \calE,  \ \tau_v \in \Qrewire$}
		\If{$c[\tau] + \ell[\tau] + \hat{w}(u[\tau],v) < c[\tau_v] + \ell[\tau_v]$}
			\State $\tau_v \gets (v, u[\tau], c[\tau], \ell[\tau] + \hat{w}(u[\tau],v), b[\tau] + 1)$ 
			\State $\Qrewire.\texttt{update\_node}(\tau_v)$
		\EndIf
	\EndFor	
\EndWhile
\State $\calT_{\rm rewire}.\texttt{clear}()$
\State \textbf{return}
\end{algorithmic}
\end{algorithm}
In Alg.~\ref{alg:evaluatePath}, when an edge $(u,v)$ is found to be in collision, the subtree rooted at $\tau_v$ is to be rewired as in Alg.~\ref{alg:rewire}. 
Initially every node in the subtree is updated to have an infinite key, and removed from the search tree (lines~1-3). 
Since these nodes have their entries re-initialized, they are removed from any priority queue they might exist in, namely, \Qfront (line~5). 
For each of these nodes, the best \emph{valid} parent in the graph is determined. The node is updated using the new parent's node entries and pushed into \Qrewire. 
Note that valid parents do not appear in the subtree being rewired since their node entries are still unknown (lines~8-9). 
Nodes belonging to \Qextend are subsequently extended in line 15 of Alg.\ref{alg:main} and hence are considered as invalid parents (lines~10-11).

Once \Qrewire is populated with all the nodes in the subtree, the algorithm iteratively pops the node with the minimal key from \Qrewire (lines~17-18). 
If a valid parent has been determined for the node, it is inserted into the search tree, and priority queues \Qextend and \Qfront depending on its budget (lines~21-25). 
Otherwise the node is left as initialized in line 3. 
Essentially, this implies that nodes can be inserted and also removed from the search tree during rewiring. 

If a node has been successfully rewired and has budget less than $\alpha$, it is now a potential valid best parent to nodes associated with its successor vertices in the graph.
Lines~27-31 verify if the node is indeed a better parent for each of its successors and updates them accordingly.

\subsection{Extending $\alpha$-band}
\begin{algorithm}[tb]
\caption{\texttt{extend\_$\alpha$\_band} $()$}
\label{alg:extend}
\begin{algorithmic}[1]

\While{$\Qextend \neq \emptyset$}
	\State $\tau \gets \Qextend.\texttt{pop}()$
	\If{$u[\tau] = \vg$} \Comment{Goal node needn't be extended}
		\State $\Qfront.\texttt{push}(\tau)$
	\Else
		\ForAll{$v \in \calV \ s.t. \ (u[\tau],v) \in \calE$ }
			\State $\tau_v' \gets (v, u[\tau], c[\tau], \ell[\tau]+\hat{w}(u[\tau],v), b[\tau] + 1)$
			\If{$c[\tau_v'] + \ell[\tau_v'] > c[\tau_v] + l[\tau_v]$}
				\State \textbf{continue}
			\EndIf
			\If{$\exists \tau_v \in \calT$} \Comment{Cleanup queues before update}
				\ForAll{$\tau' \in T_\mathrm{subtree}(\tau_v)$}
					\State $\calT.\texttt{remove}(\tau')$
					\If{$\tau' \in \Qfront$}
						\State $\Qfront.\texttt{remove}(\tau')$
					\EndIf
					\If{$\tau' \in \Qextend$}
						\State $\Qextend.\texttt{remove}(\tau')$
					\EndIf
				\EndFor
			\EndIf

			\State $\tau_v \gets \tau_v'$ \Comment{Update}
			\State $\calT.\texttt{insert}(\tau_v)$
			\If{$b[\tau_v] = \alpha$}
				\State $\Qfront.\texttt{push}(\tau_v)$
			\Else
				\State $\Qextend.\texttt{push}(\tau_v)$
			\EndIf
		\EndFor 
	\EndIf
\EndWhile
\State \textbf{return}
\end{algorithmic}
\end{algorithm}
The queue \Qextend contains leaf nodes in the search tree~$\calT$ that have a budget less than $\alpha$. 
In Alg.~\ref{alg:extend}, the top node~$\tau$ in~\Qextend with minimal key is popped (lines~1-2) and extended unless it is the node associated with $\vg$ in which case it is pushed into \Qfront (lines~3-4). 
For each of $\tau$'s successors $\tau_v$ in $\calG$, if the cost to reach $\tau_v$ through~$\tau$ is cheaper than the current cost, the node entry for $\tau_v$ is updated using $\tau$ (lines~7-17) and inserted into the search tree. 
Note that if $\tau_v$ already belongs to the search tree, it needs to be removed from any of the priority queues it previously belongs to. 
This is done in lines~10-16. 
Finally, if the successor node $\tau_v$ has been updated to have $\tau$ as the parent in $\calT$, we push $\tau_v$ into \Qextend or \Qfront depending on its budget.

\section{Algorithmic Properties: Proofs to Section~\ref{sec:proofs}}
\label{app:proofs}
In this section we provide accompanying proofs to the lemmas presented in Sec.~\ref{sec:proofs}.
For clarity we repeat the statements of the proofs throughout this section.

\lemmaOne*
\begin{proof}
Since the algorithm has evaluated the edge~$(v_0,v)$ to be collision free, there exists a path $P = (v_{\rm{start}},\ldots,v_1,v_0,v)$ for which all edges were found to be collision free by  $\ab$. Assume there exists another collision-free path $P' = (v_{\rm{start}},\ldots,v_1',v_0',v)$ from $v_{\rm{start}}$ to $v$ such that $w(P') < w(P)$. 

Consider the iteration before $\ab$ evaluates $(v_0,v)$. 
Since the edge $(v_0,v)$ is evaluated,
there exists a border node~$\tau_0$ associated with $v_0$ where and 
a frontier node $\tau$ that is $\alpha$ edges away from $\tau_0$ for which $\bar{w}(P[\tau])$ is minimal.

Let $\tau_j'$ be the last border node on $P'$ (associated with vertex~$v_j'$) and 
let~$\tau'$ be the frontier node that is $\alpha$ edges away from $\tau_j'$ for which~$\bar{w}(P[\tau'])$ is minimal.
Note that since $(v_0,v)$ was evaluated, we have that $\bar{w}(P[\tau])$ is minimal which implies that $\bar{w}(P[\tau]) < \bar{w}(P[\tau'])$.

Consider the following cases:
\subsubsection*{C1. We have that $\bm{j > \alpha}$.}
			Namely, vertex $v_j'$ lies more than~$\alpha$ edges before~$v$. 
			Thus, the vertex $v'$ associated with frontier node $\tau'$ lies on $P'$ before $v$.
			By the assumption that $w(P')<w(P)$, we have that $\bar{w}(P[\tau']) < \bar{w}(P[\tau])$ which gives a contradiction to the fact that $\bar{w}(P[\tau]) < \bar{w}(P[\tau'])$.
				See Fig.~\ref{fig:C1}.
\subsubsection*{C2. We have that $\bm{\alpha \geq j}$.}
				Namely,  vertex $v_j'$ lies at most~$\alpha$ edges before~$v$. 
			Thus, the vertex $u'$ associated with frontier node $\tau'$ is $v$ or a descendant of $v$.
			Consider the node~$\tau_v$ associated with vertex $v$.
			Clearly, $\tau_v$ is at most $\alpha$ edges from both border nodes $\tau_i$ and $\tau_j'$.
			By the assumption that $w(P')<w(P)$, we have that $\tau_v$ will have $\tau_i$ as it's ancestor in $\calT$ and \emph{not} $\tau_j'$ which contradicts the fact that $(v_0,v)$ is chosen for evaluation.
				See Fig.~\ref{fig:C2}.
\end{proof}
\begin{figure*}	
	\centering
	\begin{subfigure}[t]{0.35\textwidth}
		\centering
		\includegraphics[height = 4cm]{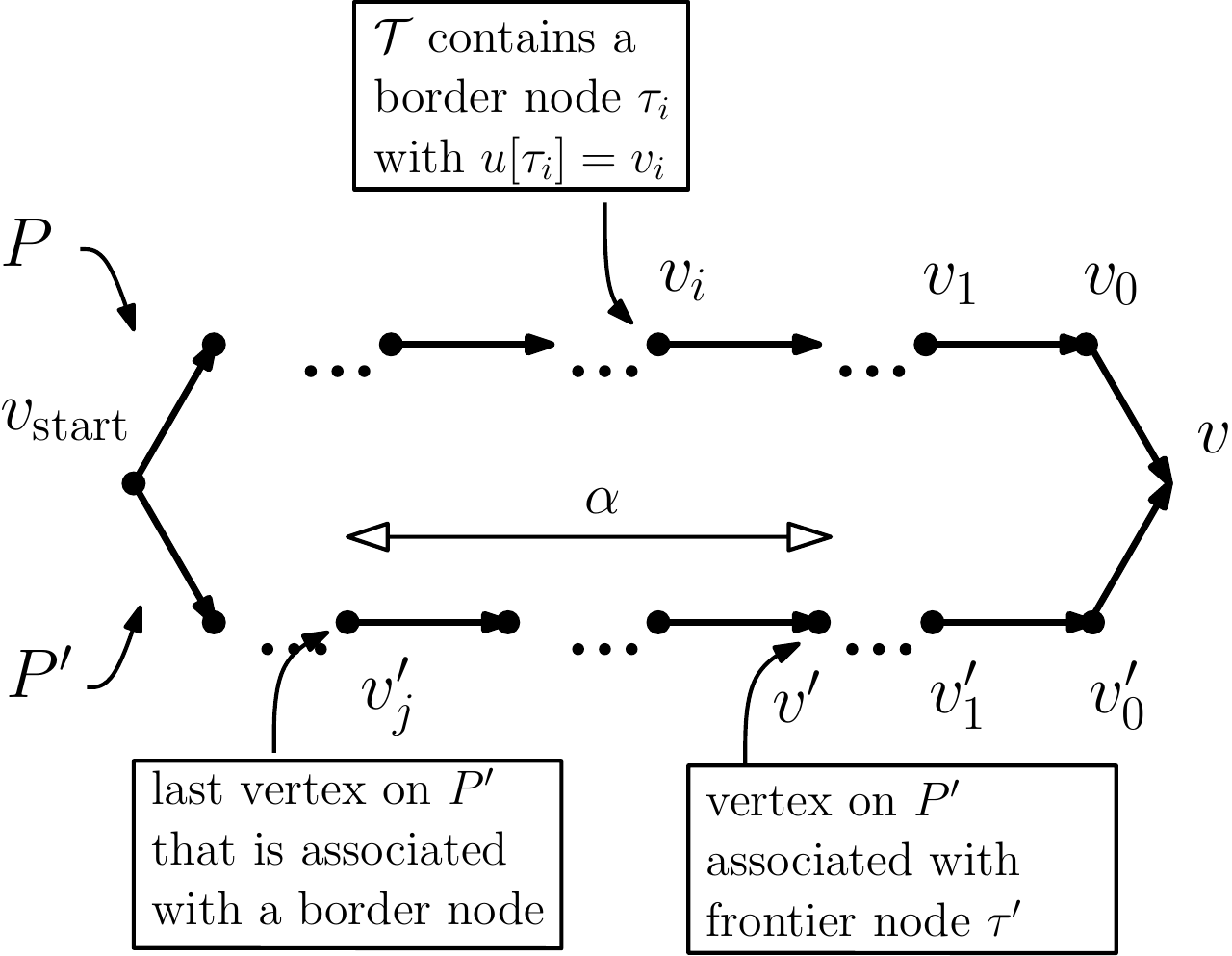}
		\caption{Case C1.}\label{fig:C1}		
	\end{subfigure}
	\quad
	\quad
	\begin{subfigure}[t]{0.35\textwidth}
		\centering
		\includegraphics[height = 4cm]{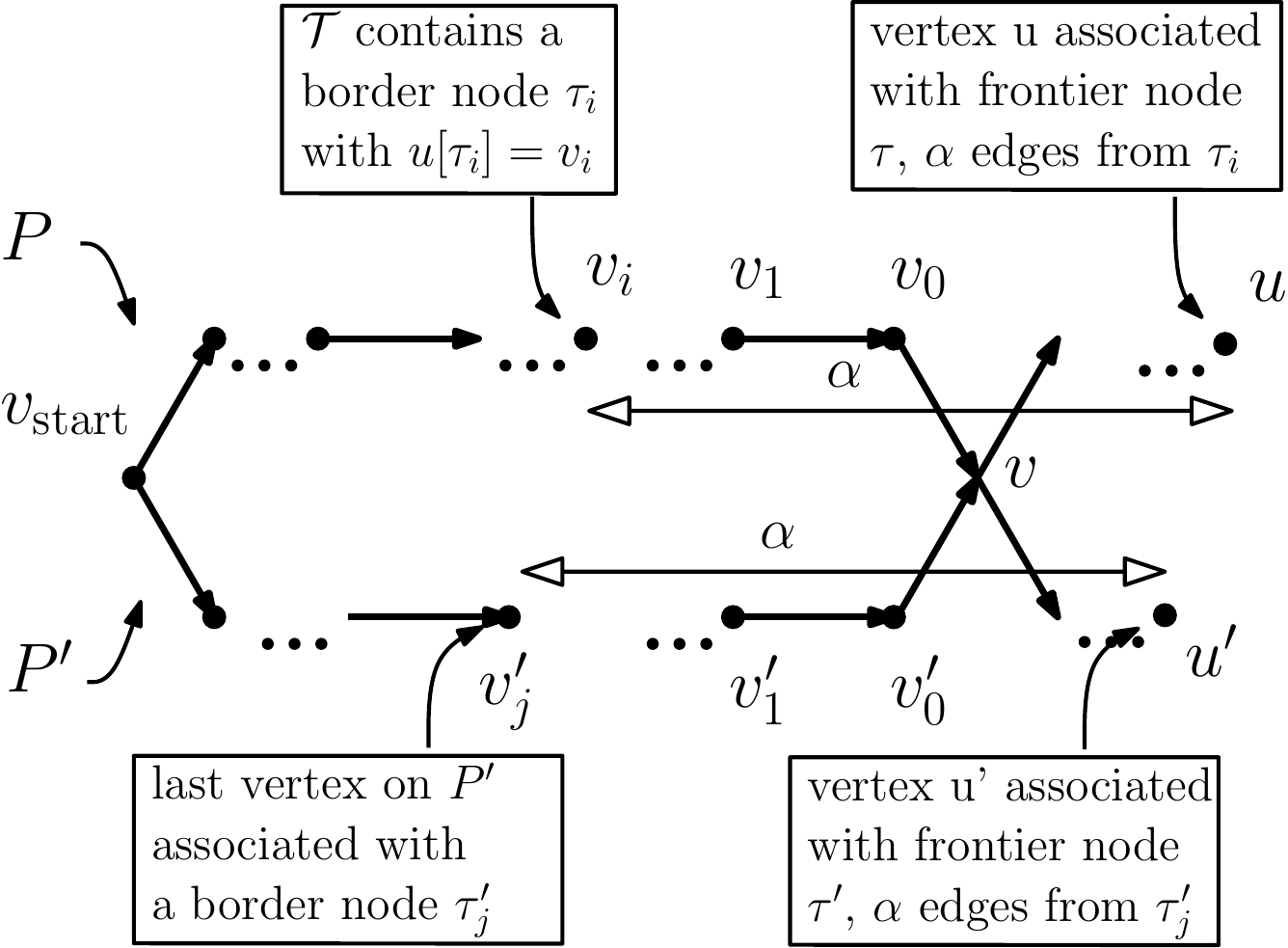}
		\caption{Case C2.}\label{fig:C2}
	\end{subfigure}
	\caption{Different cases considered in the proof of Lemma~\ref{lem:correctness}.}\label{fig:1}
\vspace{-3.5mm}
\end{figure*}

\lemmaTwo*
\begin{proof}
Notice that  \ab with $\alpha = \infty$ has an infinite lookahead. 
Thus, at each iteration, it will take the shortest path connecting $\vs$ to $\vg$ according to the lazy estimate $\hat{w}$ excluding edges that were already found to be in collision. It will then test edges on this path until one is found to be in collision or until an entire path was found to be collision-free.
Moreover, any edge $e$ tested by \ab lies on a path $P$ with $\hat{w}(P) \leq w(P^*)$.

Assume that there exists some shorttest-path problem and some algorithm \texttt{ALG} where \ab tests more edges than \texttt{ALG}. 
Let $e=(u,v)$ be an edge tested by \ab and not by \texttt{ALG}.
Since \ab tested $e$, there exists a collision-free path connecting $\vs$ to~$u$.
Furthermore, $e$ lies on a path $P$ with $\hat{w}(P) \leq w(P^*)$.

If $\hat{w}(P) = w(P^*)$ then $P = P^*$ and \texttt{ALG} must validate all edges of $P$ including $e$.
If $\hat{w}(P) < w(P^*)$ then there exists an edge on $P$ after $u$ that is in collision and \texttt{ALG} must indeed test the first edge on $P$ that is in collision. This implies that \texttt{ALG} must test  $e$.
\end{proof}

\lemmaHeuristic*
\begin{proof}
Assume that $E_1 \setminus E_2 \neq \emptyset$ and let $(v_0,v_1) \in E_1 \setminus E_2 $ be the edge such that $v_0$'s cost-to-come is minimal.
Let $u$ be the parent of $v_0$ on the shortest path from $\vs$ to $v_0$ and note that $(u, v_0) \in E_1 \cap E_2$.
Furthermore, let $\calT_i$ denote the search tree of 
\ab with heuristic $h_i$, $i \in \{1,2\}$.
Finally, recall that $w^*$ denotes the (true) weight of the shortest path from~$\vs$ to vertex $\vg$
and that
$\forall v \in \calV$ with $v\neq \vg$
we have that
$h^*_{\calG}(v) \geq h_1(v) > h_2(v)$
.

The edge $(v_0,v_1)$ was evaluated by 
\ab with heuristic $h_1$.
Thus, a frontier node~$\tau_{v_{\alpha}}^1 \in \calT_1$ associated with some vertex~$v_{\alpha}$ was at the head of~$\Qfront$ (with~$v_{\alpha}$ exactly $\alpha$ edges from~$v_0$). 
Since node~$\tau_{v_{\alpha}}^1$ was popped from~$\Qfront$, it's key is minimal.
Specifically, 
$$
c[\tau_{v_{\alpha}}^1] + \ell[\tau_{v_{\alpha}}^1] + h_1[\tau_{v_{\alpha}}^1] \leq w^*.$$

The edge $(u, v_0)$ was evaluated by 
\ab with heuristic~$h_2$,
thus a border node~$\tau_{v_0}^2 \in \calT_2$ associated with the vertex $v_0$ exists.
This, in turn, implies that there is a frontier node associated with the vertex $v_{\alpha}$.
Let $\tau_{v_{\alpha}}^2 \in \calT_2$ be this node which was created immediately after edge $(u, v_0)$ was evaluated. 
Note that
$c[\tau_{v_{\alpha}}^1] = c[\tau_{v_{\alpha}}^2]$,
$\ell[\tau_{v_{\alpha}}^1] = \ell[\tau_{v_{\alpha}}^2]$
and that
$h_1[\tau_{v_{\alpha}}^1] = h_1[\tau_{v_{\alpha}}^2]$
.
Since $h_1$ strictly dominates $h_2$, we have that 
$$
c[\tau_{v_{\alpha}}^2] + \ell[\tau_{v_{\alpha}}^2] + h_2[\tau_{v_{\alpha}}^2]
<
c[\tau_{v_{\alpha}}^2] + \ell[\tau_{v_{\alpha}}^2] + h_1[\tau_{v_{\alpha}}^2].
$$

From the above, it follows that $c[\tau_{v_{\alpha}}^2] + \ell[\tau_{v_{\alpha}}^2] + h_2[\tau_{v_{\alpha}}^2] < w^*$. Therefore, $\tau_{v_{\alpha}}^2$ will be popped from $\Qfront$ by \ab with heuristic $h_2$ before any node associated with $\vg$ which, in turn, implies that the edge $(v_0, v_1)$ will be evaluated by \ab with heuristic $h_2$. 
\end{proof}

Note that the proof of Lemma~\ref{lem:heuristic} assumes that $h_1$ strictly dominates~$h_2$. It will \emph{not} hold if $h_1$ weakly dominates~$h_2$ (namely, if $\forall v, h_1(v) \geq h_2(v)$.
Interestingly, there may be cases were both \astar and \textsf{IDA*} will \emph{not} expand fewer nodes if they use $h_2$ than if they use $h_1$~\cite{H10}.

\lemmaThree*
\begin{proof}
Assume that $E_1 \setminus E_2 \neq \emptyset$ and let $(v_0,v_1) \in E_1 \setminus E_2 $ be the edge such that $v_0$'s cost-to-come is minimal. 
Let $u$ be the parent of $v_0$ on the shortest path from $\vs$ to $v_0$ and note that $(u, v_0) \in E_1 \cap E_2$.
Let $\calT_i$ denote the search tree of 
\ab with laziness $\alpha_i$.
Finally, recall that $w^*(v)$ denotes the (true) weight of the shortest path from~$\vs$ to vertex $v$.

The edge $(v_0,v_1)$ was evaluated by 
\ab with laziness $\alpha_1$,
thus a frontier node~$\tau_{v_{\alpha_1}}^1 \in \calT_1$ associated with some vertex~$v_{\alpha_1}$ was at the head of $\Qfront$. 
Note that~$v_{\alpha_1}$ is exactly~${\alpha_1}$ edges from $v_0$ and set $v_1, \ldots, v_{\alpha_1-1}$ to be the intermediate vertices along this path.
Since node~$\tau_{v_{\alpha_1}}^1$ was popped from $\Qfront$, it's key is minimal.
Specifically, 
$$
c[\tau_{v_{\alpha_1}}^1] + \ell[\tau_{v_{\alpha_1}}^1] = w^*(u) + \sum_{i=1}^{\alpha_1}\hat{w}(v_{i-1}, v_i) \leq w^*(\vg).$$

The edge $(u, v_0)$ was evaluated by 
\ab with laziness $\alpha_2$,
thus a border node~$\tau_{v_0}^2 \in \calT_2$ associated with the vertex $v_0$ exists.
This, in turn, implies that there is a node associated with every vertex that is at most $\alpha_2$ edges from $v_0$, including with the vertex $v_{\alpha_2}$. 
Let $\tau_{v_{\alpha_2}}^2 \in \calT_2$ be this node which was created immediately after edge $(u, v_0)$ was evaluated.
We have that 
$$
c[\tau_{v_{\alpha_2}}^2] + \ell[\tau_{v_{\alpha_2}}^2] 
\leq
w^*(u) + \sum_{i=1}^{\alpha_2}\hat{w}(v_{i-1}, v_i).
$$
Since $\alpha_2 < \alpha_1$, we have that~$c[\tau_{v_{\alpha_2}}^2] + \ell[\tau_{v_{\alpha_2}}^2] < w^*(\vg)$ which implies that $\tau_{v_{\alpha_2}}^2$ will be popped from $\Qfront$ before any node associated with $\vg$. This implies edge $(v_0, v_1)$ will be evaluated by \ab with laziness $\alpha_2$.
See Fig.~\ref{fig:larger_lookahead_no_greediness}.
\end{proof}

\begin{figure}[tb]
  \centering
  	\includegraphics[width=0.45\textwidth]{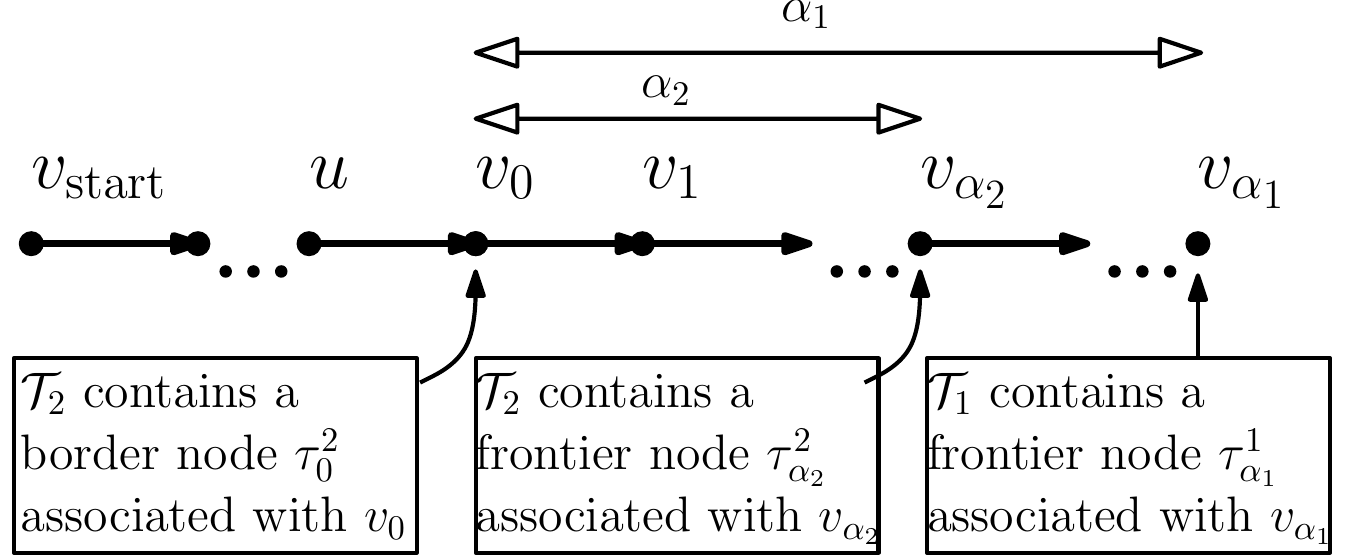}
  \caption{
  	Construction used in Lemma~\ref{lem:larger_lookahead_no_greediness}.
  	}
   	\label{fig:larger_lookahead_no_greediness}
\vspace{-3.5mm}
\end{figure}
%

%
%

\lemmaSeven*
\begin{proof}
Each vertex $v$ has exactly one node $\tau \in \calT$ associated with it.
It appears in a constant number of queues.
In addition, the algorithm needs to store the graph $G$.
Thus,
the total space complexity of our algorithm is bounded by $O(n + m)$.
\end{proof}

\lemmaEight*
\begin{proof}
Let Anc$(v)$ be the set of all vertices that are $\alpha$ edges from $v$ and lie on a path between $\vs$ and $v$ and note that $|\text{Anc}(v)|  = O(d^\alpha)$. 
Furthermore, the number of edges connecting vertices in Anc$(v)$ to $v$ is bounded by $O(d^\alpha)$.

We wish to bound the number of times $\tau_v$ associated with $v$ will be updated through the algorithm's execution.
We charge each update to the event that the algorithm evaluates one of the edges connecting vertices in Anc$(v)$ to $v$.

Each such update involves updating queues of nodes. The total number of nodes is bounded by $n$ and the cost of updating the queue is logarithmic in its size.
Finally, note that each edge is evaluated at most once. 

Thus, the algorithm's running time  can be bounded by
$$ 
\underbrace{O(n)}_{
\#\text{ of nodes}} 
\cdot 
\underbrace{O(d^\alpha)}_{
\#\text{ of node updates}} 
\cdot 
\underbrace{O(\log n)}_{
\text{cost of node update}}
+
\underbrace{O(m)}_{
\#\text{ of edges}},
$$
which concludes the proof. 
\end{proof}

\begin{figure*}	
	\centering
	\begin{subfigure}[t]{0.5\columnwidth}
		\centering
		\includegraphics[width=\columnwidth]{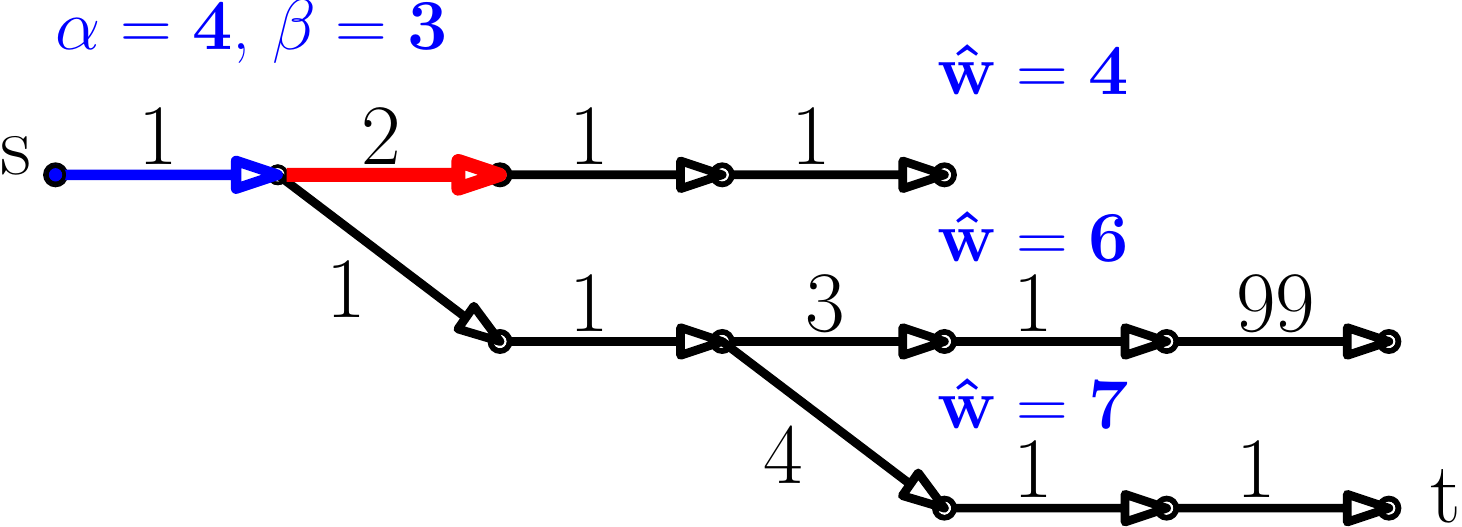}
		\caption{}\label{fig:1a}		
	\end{subfigure}
	\begin{subfigure}[t]{0.5\columnwidth}
		\centering
		\includegraphics[width=\columnwidth]{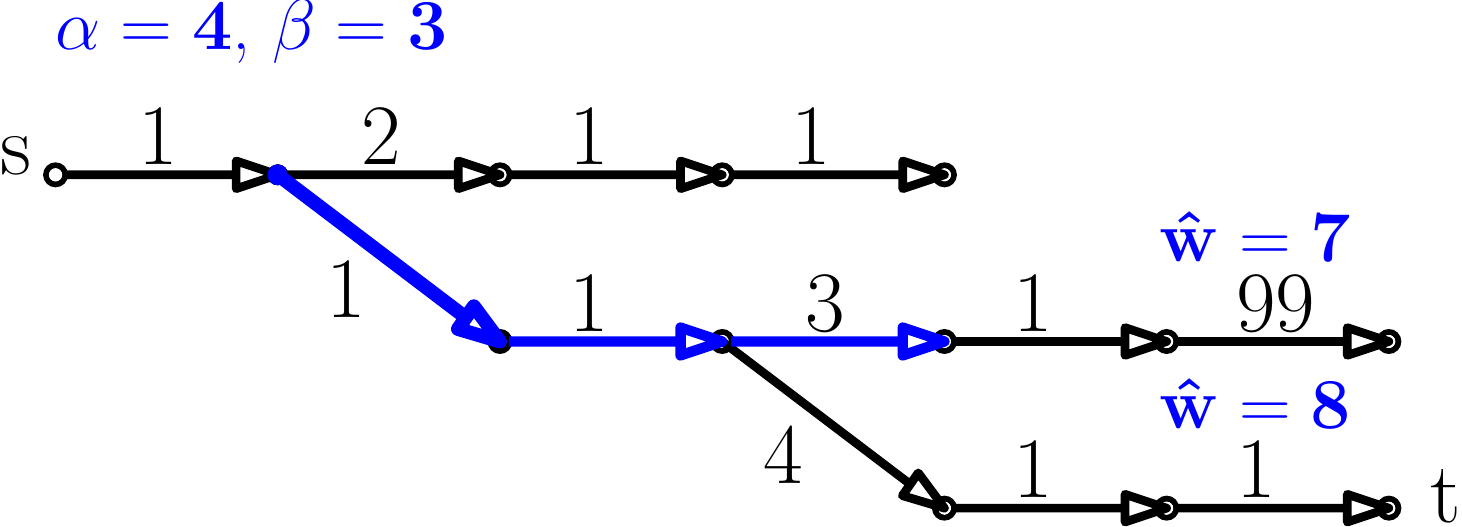}
		\caption{}\label{fig:1b}
	\end{subfigure}
	\begin{subfigure}[t]{0.5\columnwidth}
		\centering
		\includegraphics[width=\columnwidth]{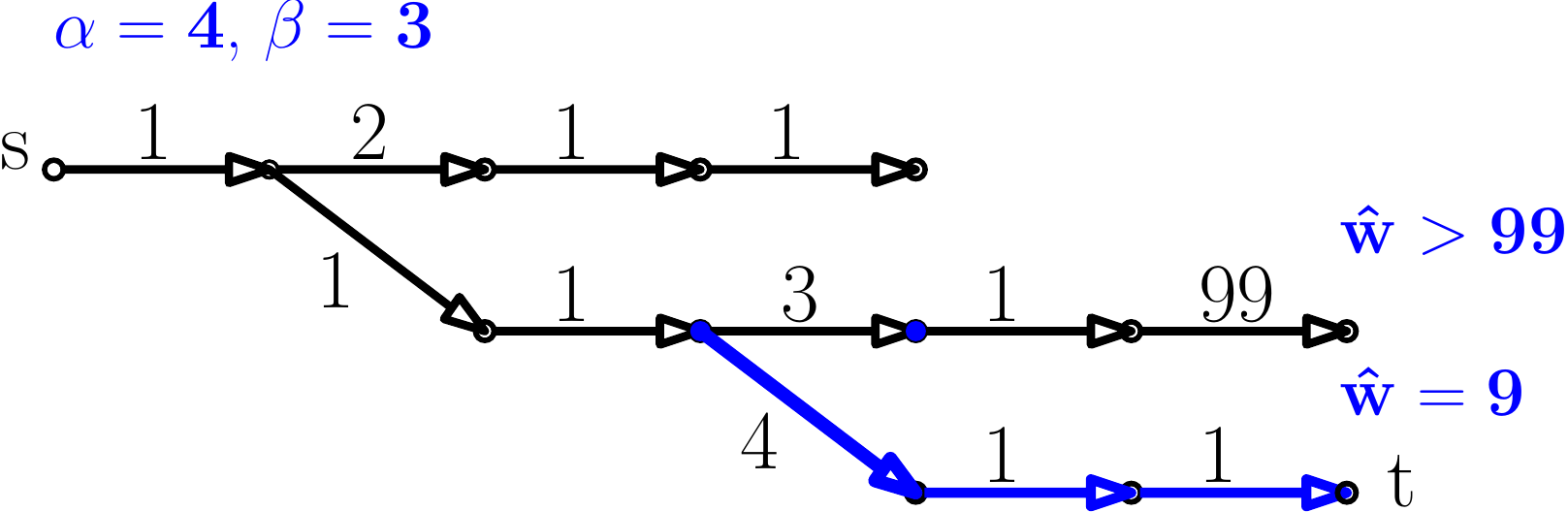}
		\caption{}\label{fig:1c}
	\end{subfigure}
	\begin{subfigure}[t]{0.5\columnwidth}
		\centering
		\includegraphics[width=\columnwidth]{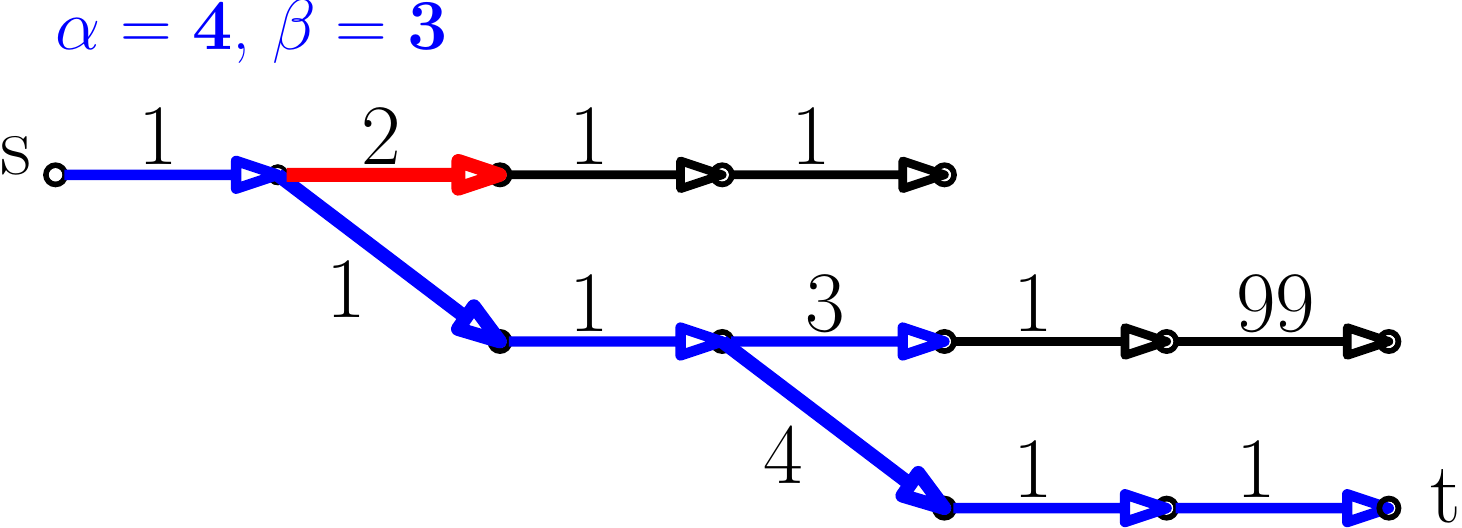}
		\caption{}\label{fig:1d}
	\end{subfigure}
	\begin{subfigure}[t]{0.5\columnwidth}
		\centering
		\includegraphics[width=\columnwidth]{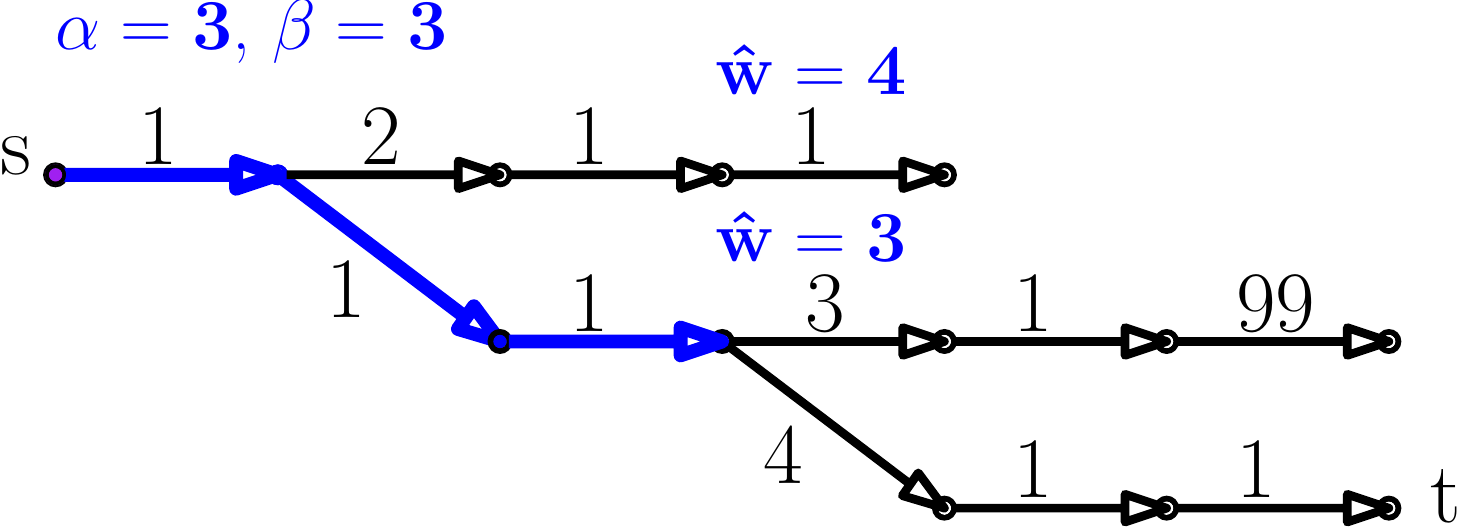}
		\caption{}\label{fig:2a}		
	\end{subfigure}
	\begin{subfigure}[t]{0.5\columnwidth}
		\centering
		\includegraphics[width=\columnwidth]{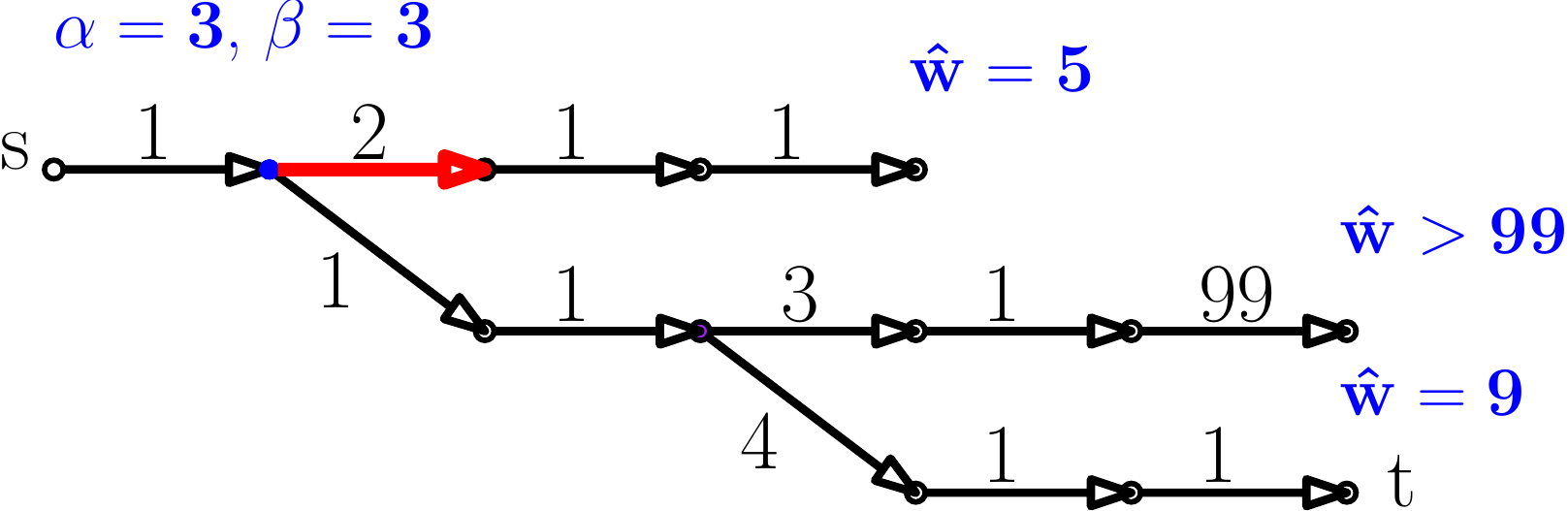}
		\caption{}\label{fig:2b}
	\end{subfigure}
	\begin{subfigure}[t]{0.5\columnwidth}
		\centering
		\includegraphics[width=\columnwidth]{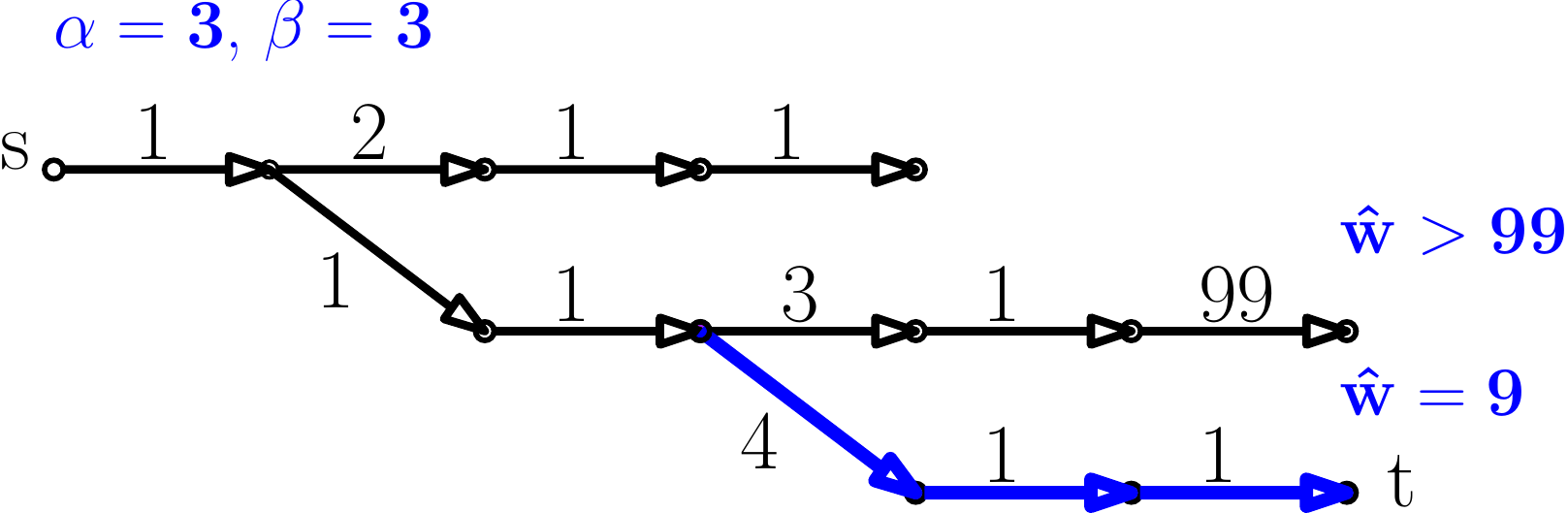}
		\caption{}\label{fig:2c}
	\end{subfigure}
	\begin{subfigure}[t]{0.5\columnwidth}
		\centering
		\includegraphics[width=\columnwidth]{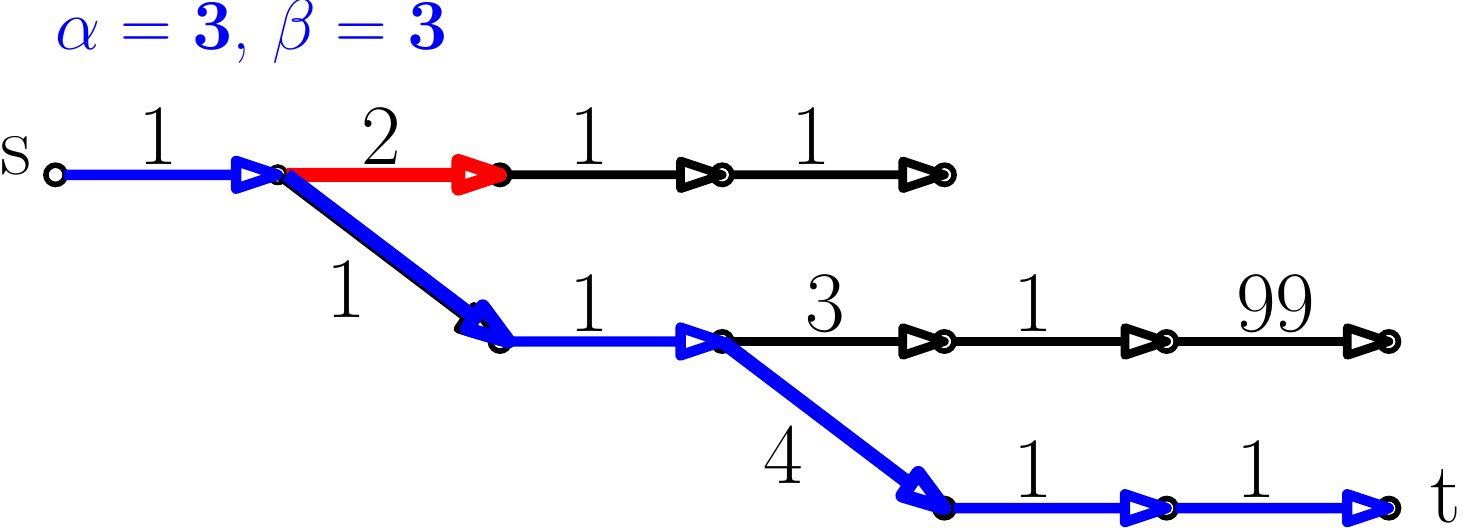}
		\caption{}\label{fig:2d}
	\end{subfigure}
	\caption{
	Example where given a larger lookahead results in \emph{more} edge evaluations.
	Top (\subref{fig:1a}-\subref{fig:1d}) 
	and
	bottom (\subref{fig:2a}-\subref{fig:2d}) 
	depict the flow of \ab for lazy lookahead value of $\alpha = 4$ and $\alpha = 3$, respectfully and a greediness value of 	$\beta = 3$.
	Each figure depicts one iteration, blue edges are found to be collision free while red edges are found to be in collision.
	Value of frontier nodes are shown at each iteration and all the edges evaluated by each algorithm are shown in (\subref{fig:1d}) and (\subref{fig:2d}).
	}\label{fig:counter_example}
\end{figure*}

\section{Is greediness beneficial?---Theorems}
\label{app:proofs2}
In this section we provide accompanying theorems and proofs to the discussion regarding greediness presented in Section~\ref{sec:greediness}.
We start by noting that Lemma~\ref{lem:correctness} and~\ref{lem:infinte_lookahead} can be easily extended to the case where greediness is employed.
We next continue by fixing one parameter 
(either lazy lookahead~$\alpha$ or greediness~$\beta$) and show under what conditions of the other parameter ($\beta$ or $\alpha$, respectively) $\ab$ performs less edge evaluations.
Recall (Lemma~\ref{lem:larger_lookahead_no_greediness}) that where the algorithm has no greediness, namely, $\beta = 1$.  then when considering edge evaluations, the larger the lookahead, the better the algorithm.
This can be seen as warm-up for Lemma~\ref{lem:larger_lookahead_with_greediness} which gives a general relationship between different lookaheads $\alpha$ for a fixed greediness $\beta$.
We then move on to fix $\alpha$ and see how varying~$\beta$ affects the algorithm.
In Lemma~\ref{lem:smaller_greediness} we show that for a fixed lookahead, no greediness ($\beta = 1$) is always better (in terms of edge evaluations) when compared to any other value of~$\beta$.
Finally, in Lemma~\ref{lem:smaller_greediness_generic} we show the somewhat counter-intuitive result that for larger greediness values ($\beta > 1$) and fixed lookahead, there is always an example where the greater the greediness, the better. 

We start by noting that if~$\beta > 1$ it may be the case that the larger the lookahead, the better (when considering edge evaluations).
In a nutshell, the greediness~$\beta$ may drive the algorithm to evaluate edges along paths that, at first glance, seem promising but as the algorithm evaluates edges, it becomes evident that other paths are more promising.
See Fig.~\ref{fig:counter_example} for an example.
In the next Lemma, we show under what conditions (for~$\beta > 1$) this natural behaviour does indeed hold.

\begin{restatable}{lem}{lemmaFour}
\label{lem:larger_lookahead_with_greediness}
For every graph $G$ and every  $\alpha_1 > \alpha_2 \geq \beta$, we have that $E_1 \subseteq E_2$ if $\alpha_1 \geq \alpha_2 + \beta - 1$.
Here, $E_i(\beta)$ denotes the set of edges evaluated by \ab with laziness $i$ and greediness~$\beta$.
\end{restatable}
\begin{proof}
Assume that $E_1 \setminus E_2 \neq \emptyset$ and let $(v_0,v) \in E_1 \setminus E_2 $ be an edge such that $v_0$'s cost-to-come is minimal. 
Note that this implies that both algorithms compute the shortest path to $v_0$. 
Furthermore, let $\calT_i$ denote the search tree 
of \ab with laziness $i$ and greediness~$\beta$.

Consider the iteration before 
\ab with laziness $\alpha_1$ and greediness~$\beta$
evaluates $(v_0,v)$
and let 
$(v_{\beta - 1},v_{\beta - 2},\ldots,v_1,v_0,v)$ be the sequence of vertices along the $\beta$ edges lying on the shortest path from $v_{\rm{start}}$ to~$v$. Since the edge $(v_0,v)$ is evaluated,
there exists a border node $\tau_i \in \calT_1$ associated with $v_i$ where $0 \leq i \leq \beta-1$.
Furthermore, the lazy path from~$\tau_i$ that passes through $v_i$ was considered, hence there is a node~$\tau \in \calT_1$ which is $\alpha_1$ edges from $\tau_i$ whose key is minimal. Namely $\bar{w}(P[\tau])< w^*$ with $w^* = w^*(\vg)$ the minimal cost to reach $\vg$. Note that this path $P[\tau]$ contains the edge $(v_0,v)$.

Clearly, $\calT_2$ contains a border node associated with $v_0$.
\ab with laziness $\alpha_2$ and greediness~$\beta$
does not expand any path from $\tau_{0}$ that contains the edge $(v_0,v)$, thus all paths $\alpha_2$ edges away from~$\tau_{v_0}$ passing through $v$ have lazy cost larger than $w^*$.
However, the node $\tau$ 
(which caused \ab with laziness $\alpha_1$ and greediness~$\beta$ to evaluate $(v_0,v)$)
is $\alpha_1-i \geq \alpha_1 - (\beta - 1) \geq \alpha_2$ edges from $v_0$.
We know that $\bar{w}(P[\tau])< w^*$  thus $(v_0,v)$ should have been evaluated by 
\ab with laziness $\alpha_2$ and greediness~$\beta$
which gives us a contradiction.
For a visualization, see Fig.~\ref{fig:larger_lookahead_with_greediness}.
\end{proof}

%
\begin{figure*}
  	\begin{subfigure}[h]{0.5\textwidth}
  	\includegraphics[height=4cm]{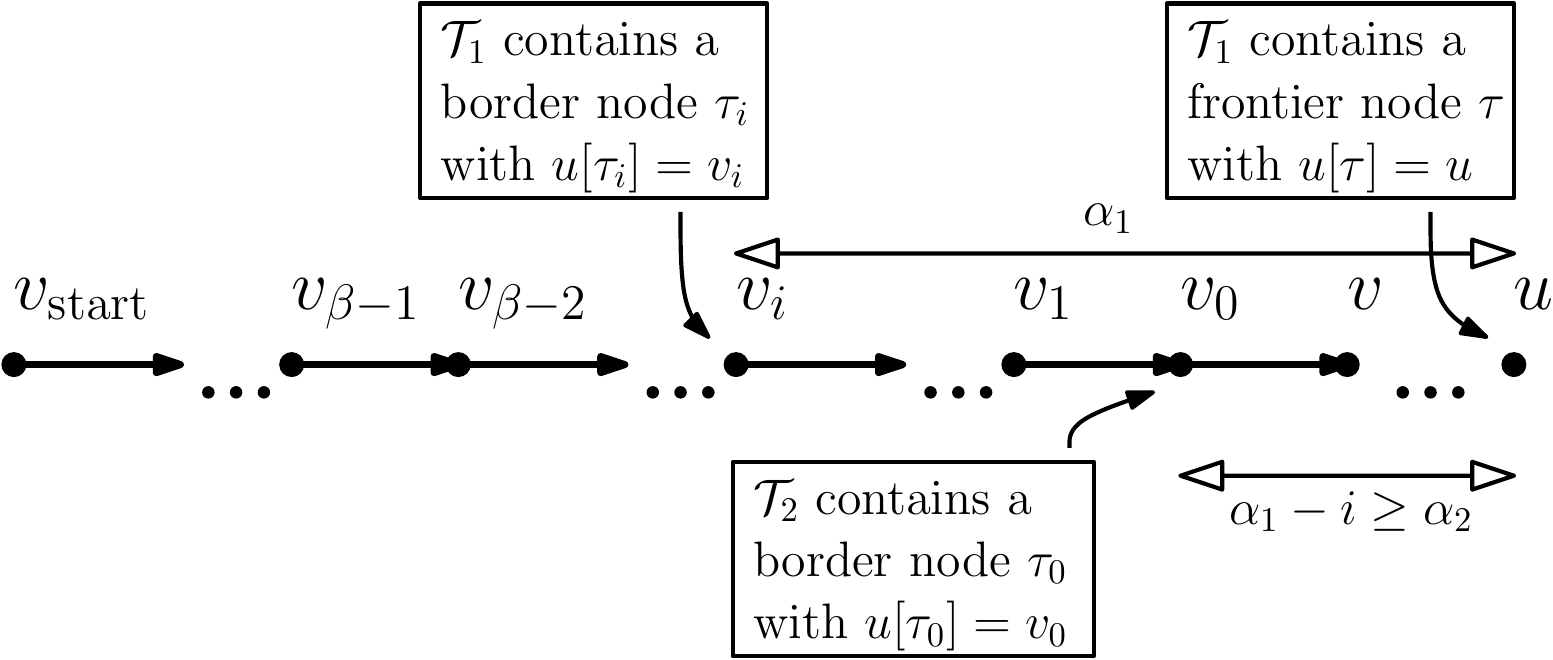}
  \caption{
  	}
   	\label{fig:larger_lookahead_with_greediness}
\end{subfigure} \hfill
\begin{subfigure}[h]{0.5\textwidth}
  \centering
  	\includegraphics[height=4cm]{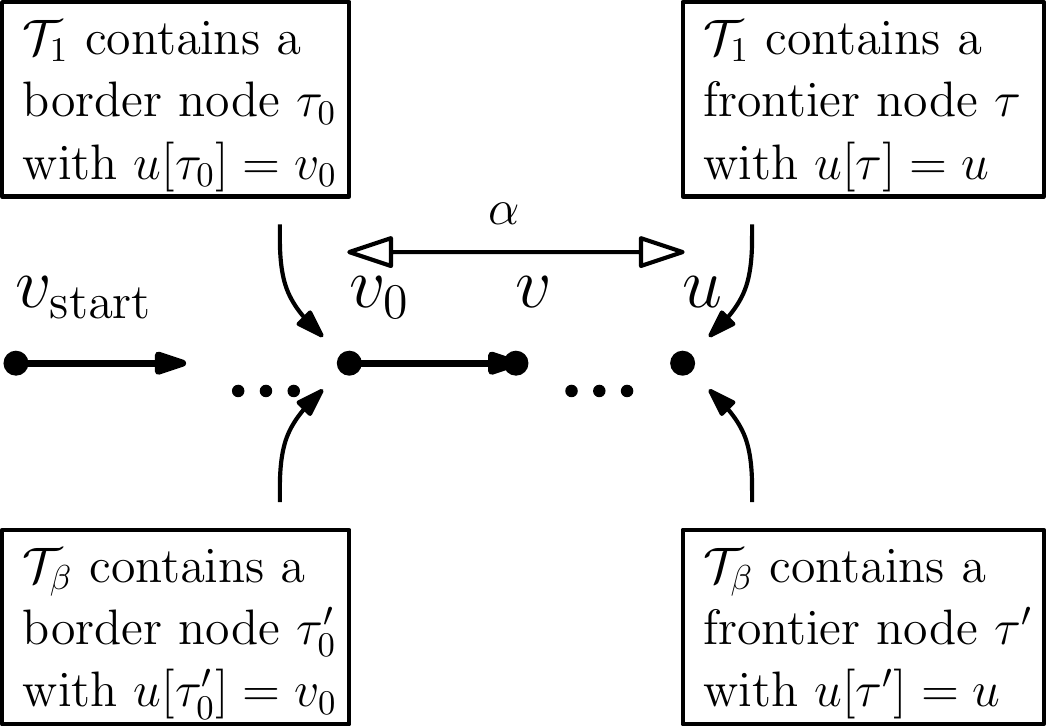}
  \caption{
  	}
   	\label{fig:smaller_greediness}
\end{subfigure}
\caption{
Constructions used in Lemma~\ref{lem:larger_lookahead_with_greediness} (Fig. \subref{fig:larger_lookahead_with_greediness}) and 
Lemma~\ref{lem:smaller_greediness} (Fig. \subref{fig:smaller_greediness})
}
\end{figure*}

We now move to the case where the lookahead $\alpha$ is fixed and we compare the edge evaluation of \ab for different values of $\beta$.
We start with the simple case where one algorithm has a greediness value of $\beta  =1$.

\begin{restatable}{lem}{lemmaFive}
\label{lem:smaller_greediness}
For every graph $\calG$ and every  $\alpha \geq \beta > 1$, we have that $E_1 \subseteq E_\beta$.
Here $E_x$ denotes the edges set of \ab with laziness $\alpha$ and greediness $x$.
\end{restatable}
\begin{proof}
Assume that $E_1 \setminus E_\beta \neq \emptyset$ and let $(v_0,v) \in E_1 \setminus E_\beta $ be an edge such that $v_0$'s cost-to-come is minimal. 
Note that this implies both algorithms compute the shortest path to $v_0$. 
Furthermore, let $\calT_x$ denote the search tree 
of \ab with greediness $x$.

Consider the iteration before \ab with greediness $1$ (no greediness) evaluates $(v_0,v)$. Since the edge $(v_0,v)$ is evaluated, the node $\tau_0\in \calT_1$ associated with $v_0$ was the border node and there exists a node $\tau\in \calT_1$ which is $\alpha$ edges from $\tau_0$ whose key is minimal. Namely, $\bar{w}(P[\tau])< w^*$ with $w^* = w^*(\vg)$ the minimal cost to reach~$\vg$. Note that the path $P[\tau]$ contains the edge $(v_0,v)$. 

Now, consider the search tree $\calT_\beta$ of \ab with greediness~$\beta$. Clearly, $\calT_\beta$ contains a border node $\tau_{0}'$ with $u[\tau_{0}'] = v_0$. 
There exists a node $\tau'$ with $u[\tau'] = u[\tau]$.
Namely, the node $\tau'$ which is exactly $\alpha$ edges away from $\tau_0'$ has $\bar{w}(P[\tau'])< w^*$ where $P[\tau']$ contains the edge $(v_0,v)$. Hence the node $\tau'$ would be popped from $\Qfront$ before any node associated with $\vg$ implying the edge $(v_0,v)$ would be evaluated giving us a contradiction.
For a visualization, see Fig.~\ref{fig:smaller_greediness}.
\end{proof}


We continue to examine the general case where 
the lookahead $\alpha$ is fixed and we compare the edge evaluation of \ab for different values of $\beta$ for $\beta >1$.
Intuitively, we would expect some result stating that the smaller the greediness the better (in terms of maximal number of edge evaluations).
Indeed, in Lemma~\ref{lem:smaller_greediness} we showed that this is the case when $\beta  =1$.
However, the following Lemma states that for general values of $\beta$, there exists cases where an algorithm with large greedines may outperform an algorithm with smaller greedines.

\begin{restatable}{lem}{lemmaSix}
\label{lem:smaller_greediness_generic}
For every lookahead $\alpha < \infty$
and
every greediness $\alpha \geq \beta_2 > \beta_1 > 1$,
there exists a graph $\calG$ where 
$E_{\beta_1} \setminus E_{\beta_2} \neq \emptyset$.
Here, $E_\beta$ denotes the set of edges evaluated by \ab with laziness $\alpha$ and greediness $\beta$.
\end{restatable}
\begin{proof}
We construct the graph $\calG$ explicitly. (Fig.~\ref{fig:lemma6a} and~\ref{fig:lemma6b}.

We consider two following two cases
(i)~$\beta_2 \text{mod} \beta_1 \neq 0$ 
and (ii)~$\beta_2 \text{mod} \beta_1 = 0$.
For each case we provide a different graph $\calG$ and show that $E_{\beta_1} \setminus E_{\beta_2} \neq \emptyset$.
See Fig.~\ref{fig:lemma6a} and~\ref{fig:lemma6b} for depictions of each case described.

For case (i) where~$\beta_2 \text{mod} \beta_1 \neq 0$, we have a path of length $\beta_2$ followed by two paths of length $\alpha$.
For \ab with greediness $\beta_1$, (Fig.~\ref{fig:3a}-\ref{fig:3d}), the algorithm starts by evaluating edges along the path of length $\beta_2$ (Fig.~\ref{fig:3a}).
Since $\beta_2 \text{mod} \beta_1 \neq 0$, at some point it will evaluate the first edge along the upper path
(which is in collision)
(Fig.~\ref{fig:3b}).
This path is longer than the lower one, but to see this, the algorithm requires a lookahead of $\alpha$ edges from the end of the first path.
The algorithm continues to evaluate edges along the lower path until the target is reached 
(Fig.~\ref{fig:3c}).

For \ab with greediness $\beta_2$  (Fig.~\ref{fig:4a}-\ref{fig:4c}), the algorithm starts by evaluating all edges along the path of length $\beta_2$ (Fig.~\ref{fig:3a}).
Since it can see all edges  along the upper path
(which is collision free) it continues to evaluate the lower path until the target is reached 
(Fig.~\ref{fig:4b}).
The final edges evaluated by each algorithm are depicted in Fig.~\ref{fig:3d} and~\ref{fig:4c}.

For case (ii) where~$\beta_2 \text{mod} \beta_1 = 0$, we have a path of length~$\beta_2$ which after one edge has a shorter path of $\alpha$ edges.
The rest of the construction is similar to case (i).
Essentially, 
\ab with greediness $\beta_1$
and
\ab with greediness $\beta_2$
behave similarly to case (i) except that both algorithms will evaluate the first edge along the path of length $\beta_2$ followed by the first (in-collision edge) of the shorter path of $\alpha$ edges.
After this first iteration for both algorithms (Fig.~\ref{fig:5a} and~\ref{fig:6a})
the behaviour reduces to that of case (i).
\end{proof}

\begin{figure*}	
	\centering
	\begin{subfigure}[t]{0.5\columnwidth}
		\centering
		\includegraphics[width=\columnwidth]{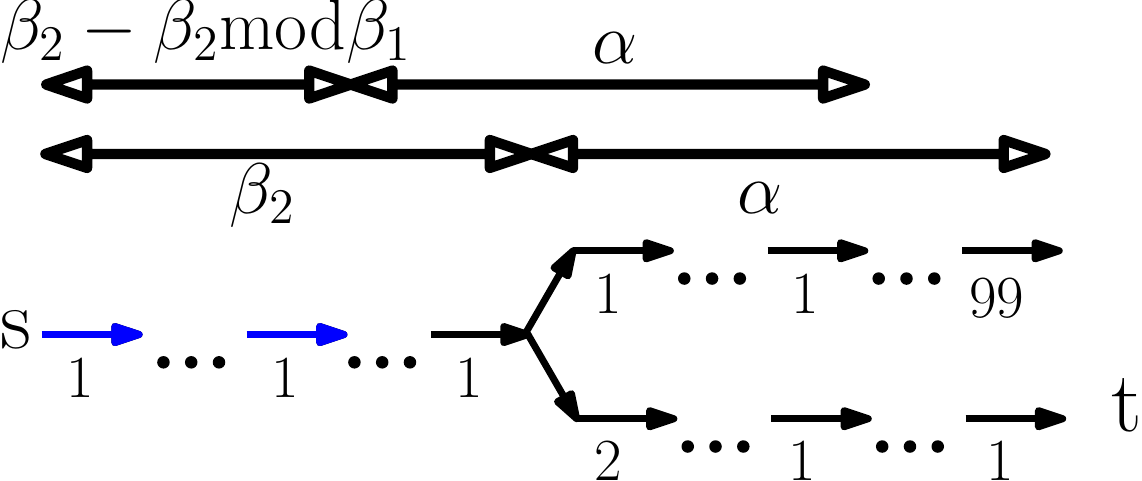}
		\caption{}\label{fig:3a}		
	\end{subfigure}
	\begin{subfigure}[t]{0.5\columnwidth}
		\centering
		\includegraphics[width=\columnwidth]{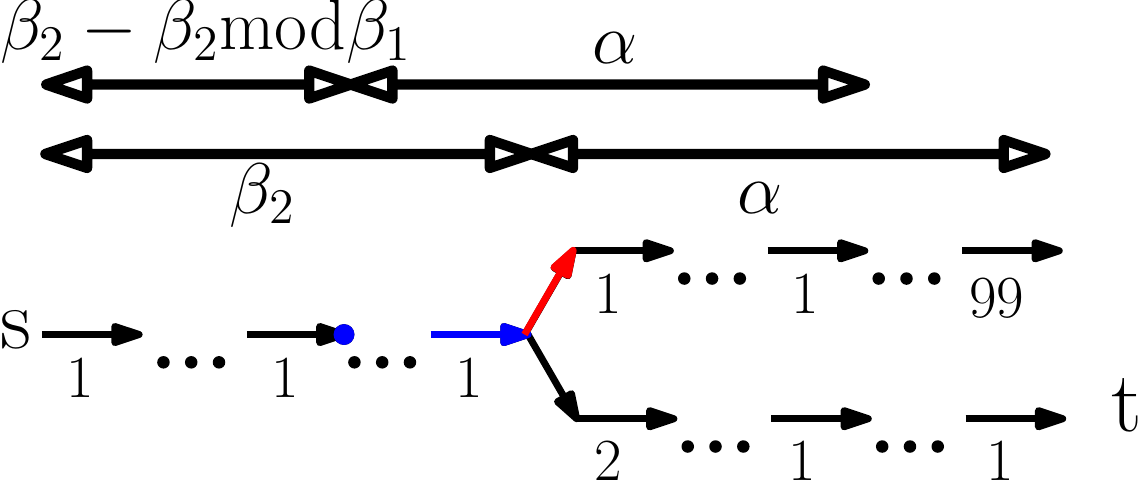}
		\caption{}\label{fig:3b}
	\end{subfigure}
	\begin{subfigure}[t]{0.5\columnwidth}
		\centering
		\includegraphics[width=\columnwidth]{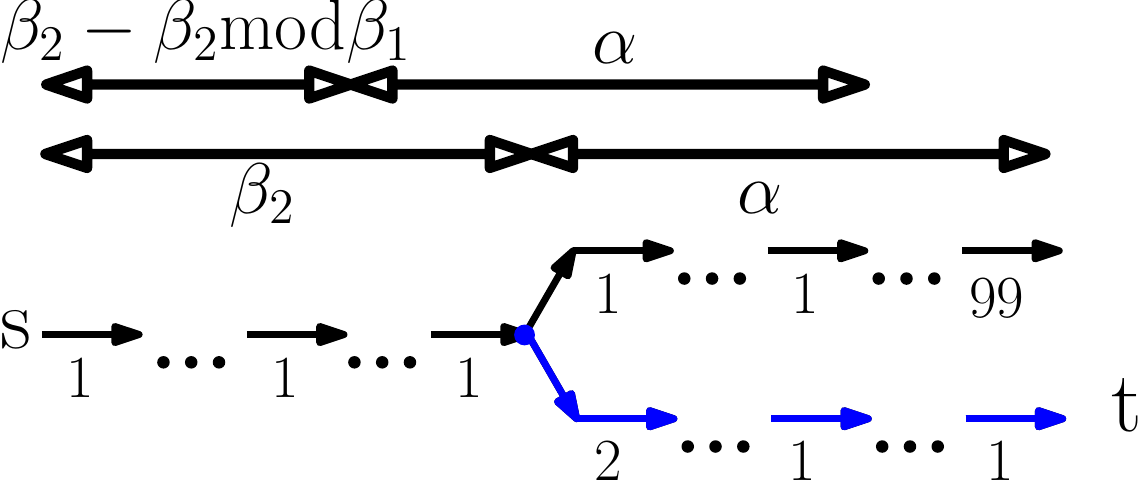}
		\caption{}\label{fig:3c}
	\end{subfigure}
	\begin{subfigure}[t]{0.5\columnwidth}
		\centering
		\includegraphics[width=\columnwidth]{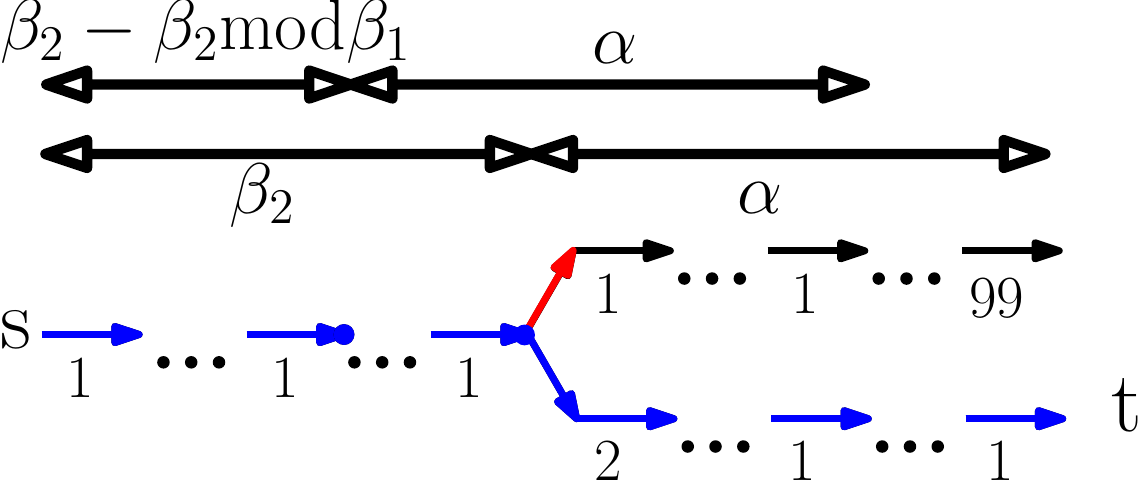}
		\caption{}\label{fig:3d}
	\end{subfigure}
	\begin{subfigure}[t]{0.5\columnwidth}
		\centering
		\includegraphics[width=\columnwidth]{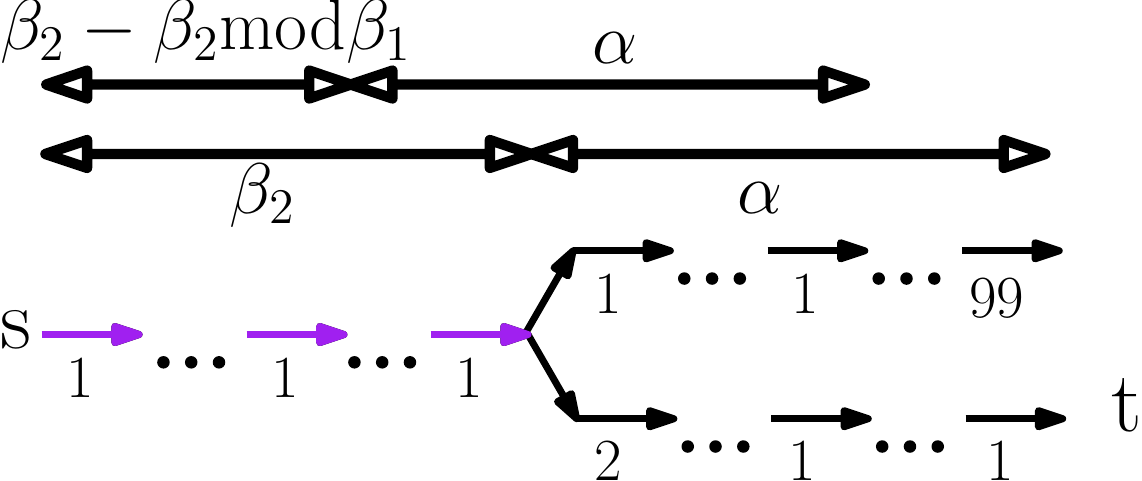}
		\caption{}\label{fig:4a}		
	\end{subfigure}
	\begin{subfigure}[t]{0.5\columnwidth}
		\centering
		\includegraphics[width=\columnwidth]{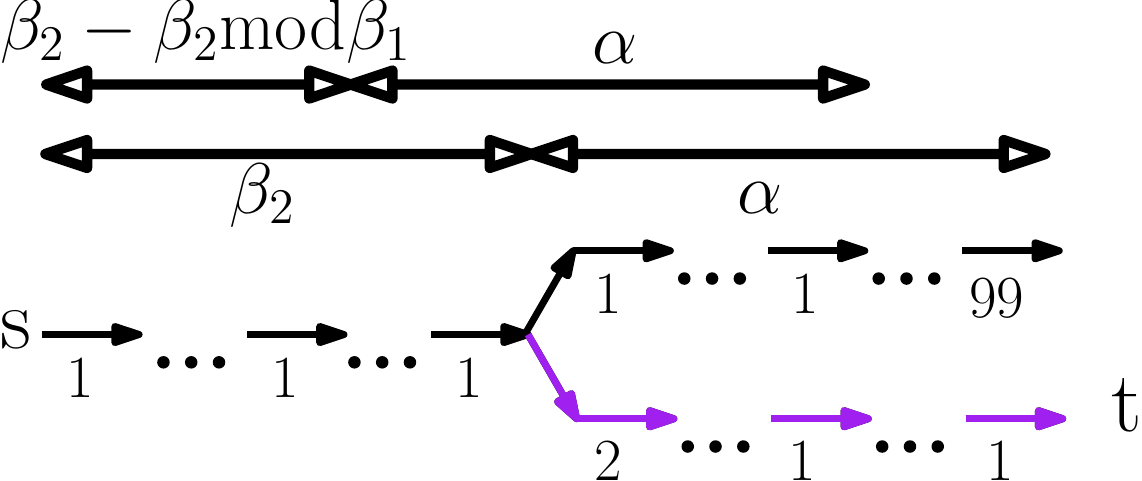}
		\caption{}\label{fig:4b}
	\end{subfigure}
	\begin{subfigure}[t]{0.5\columnwidth}
		\centering
		\includegraphics[width=\columnwidth]{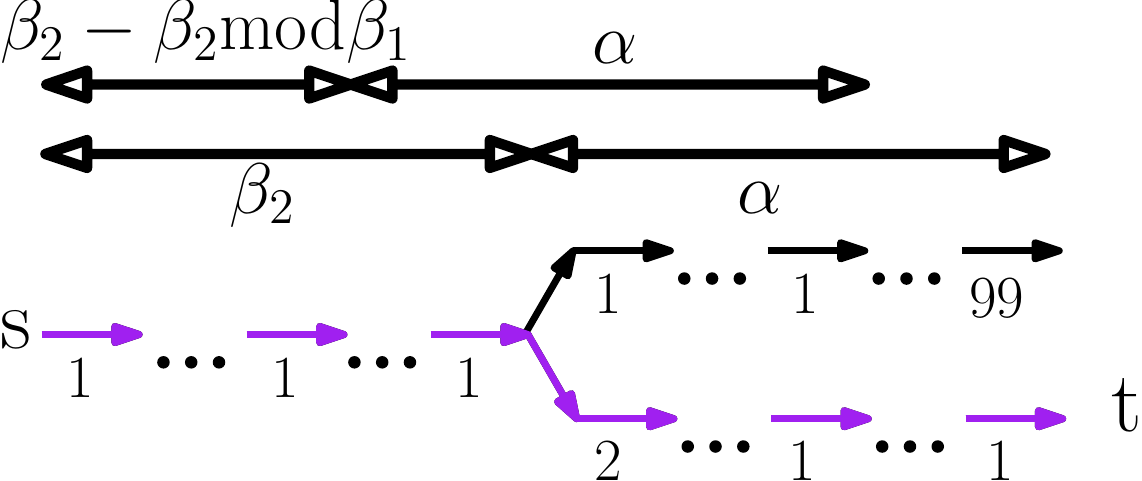}
		\caption{}\label{fig:4c}
	\end{subfigure}
	\vspace{-1.5mm}
	\caption{Construction used in Lemma~\ref{lem:smaller_greediness_generic}, when $\beta_2 \text{mod} \beta_1 \neq 0$.
	The flow of \ab with greediness $\beta_1$ is depicted in (\subref{fig:3a}-\subref{fig:3d}).
	Edges found to be collision free and in collision are depicted in blue and red, respectively.
	The flow of \ab with greediness~$\beta_2$ is depicted in (\subref{fig:4a}-\subref{fig:4c}).
	Edges found to be collision free and in collision are depicted in purple and red, respectively.
	All the edges evaluated by each algorithm are shown in~(\subref{fig:3d}) and~(\subref{fig:4c}).
	}\label{fig:lemma6a}
\end{figure*}

\begin{figure*}[tbh]
	\centering
	\begin{subfigure}[t]{0.5\columnwidth}
		\centering
		\includegraphics[width=\columnwidth]{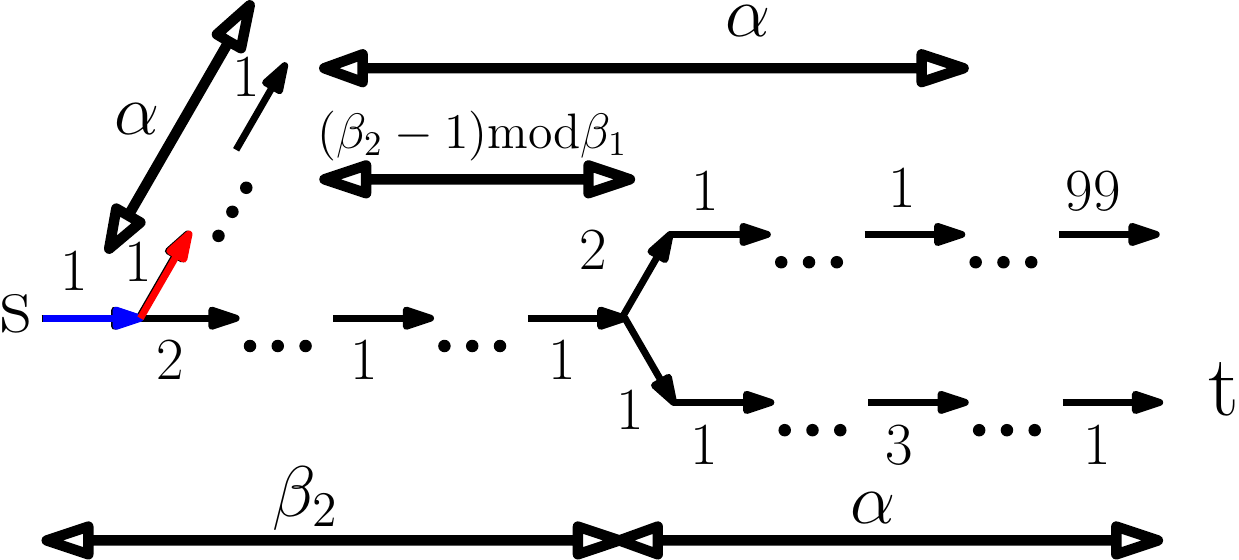}
		\caption{}\label{fig:5a}		
	\end{subfigure}
	\begin{subfigure}[t]{0.5\columnwidth}
		\centering
		\includegraphics[width=\columnwidth]{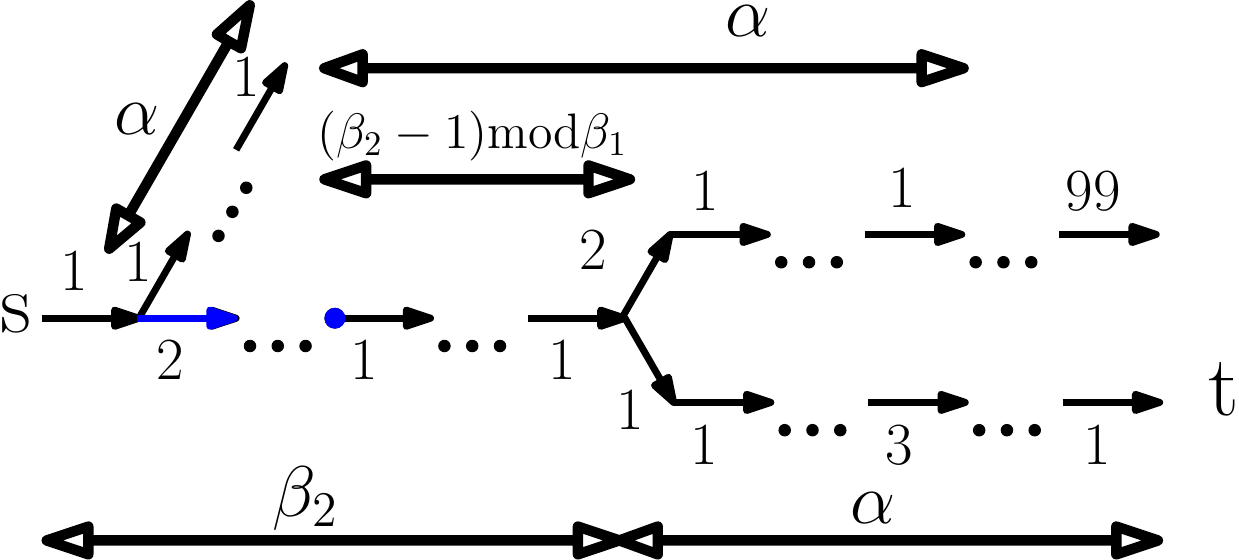}
		\caption{}\label{fig:5b}
	\end{subfigure}
	\begin{subfigure}[t]{0.5\columnwidth}
		\centering
		\includegraphics[width=\columnwidth]{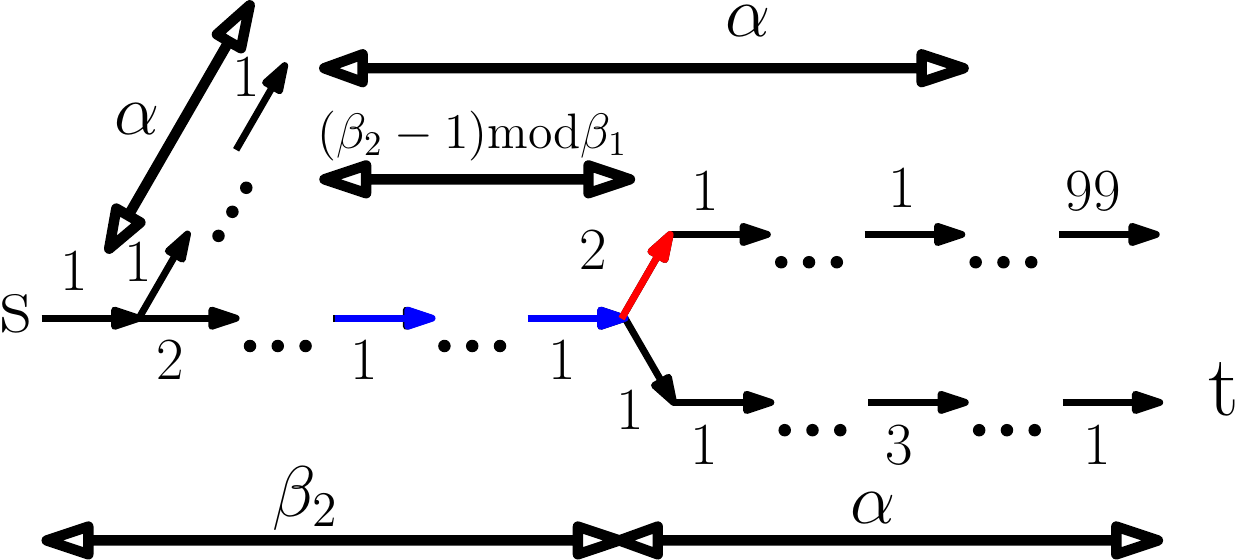}
		\caption{}\label{fig:5c}
	\end{subfigure}
	\begin{subfigure}[t]{0.5\columnwidth}
		\centering
		\includegraphics[width=\columnwidth]{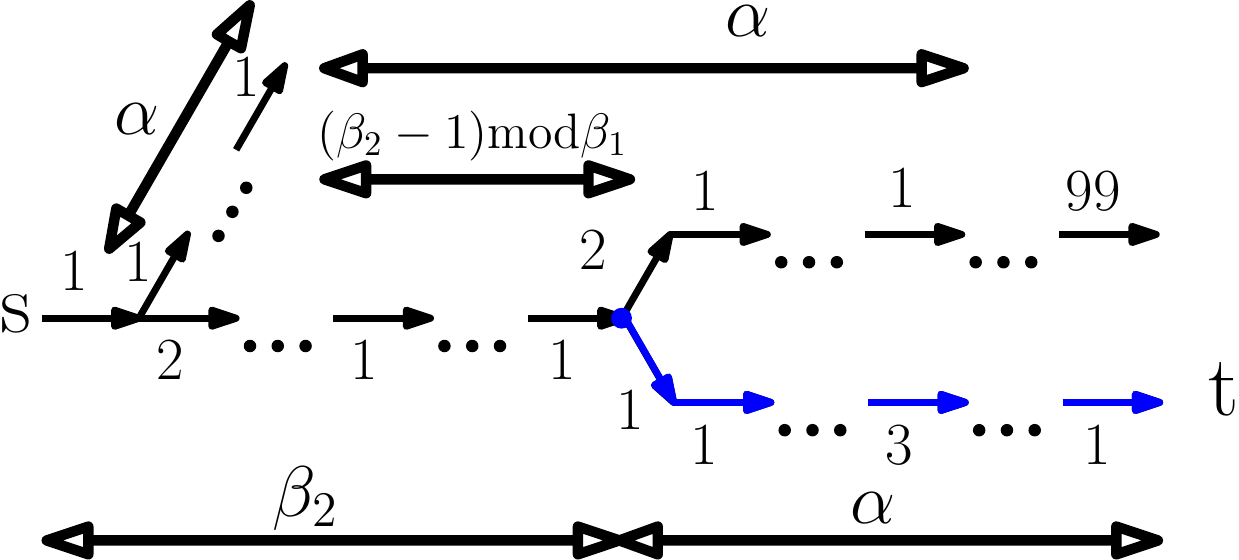}
		\caption{}\label{fig:5d}
	\end{subfigure}
	\begin{subfigure}[t]{0.5\columnwidth}
		\centering
		\includegraphics[width=\columnwidth]{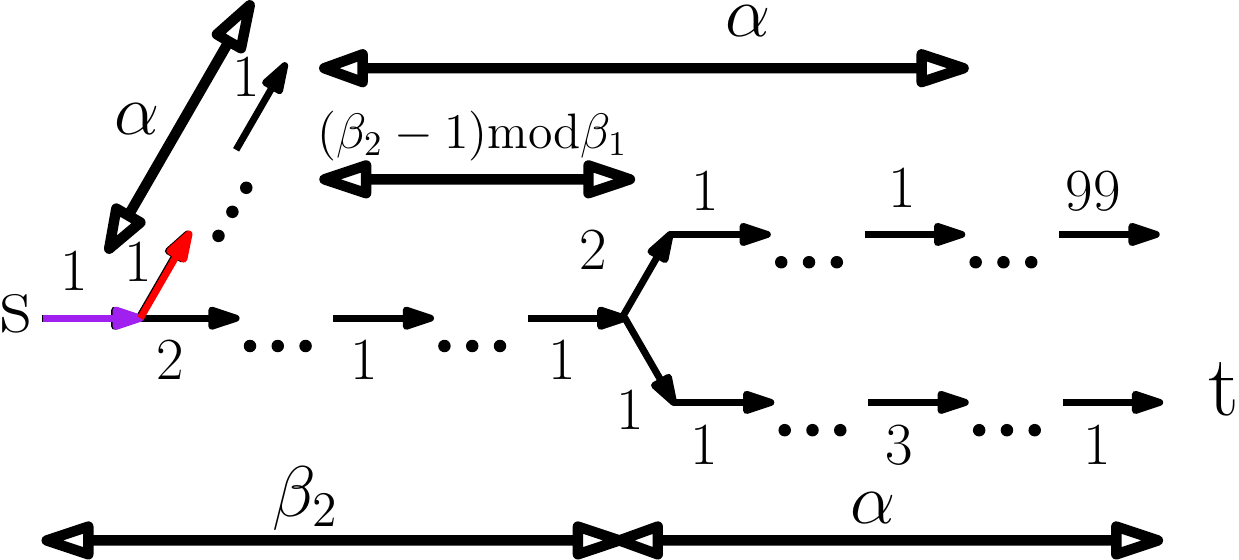}
		\caption{}\label{fig:6a}		
	\end{subfigure}
	\begin{subfigure}[t]{0.5\columnwidth}
		\centering
		\includegraphics[width=\columnwidth]{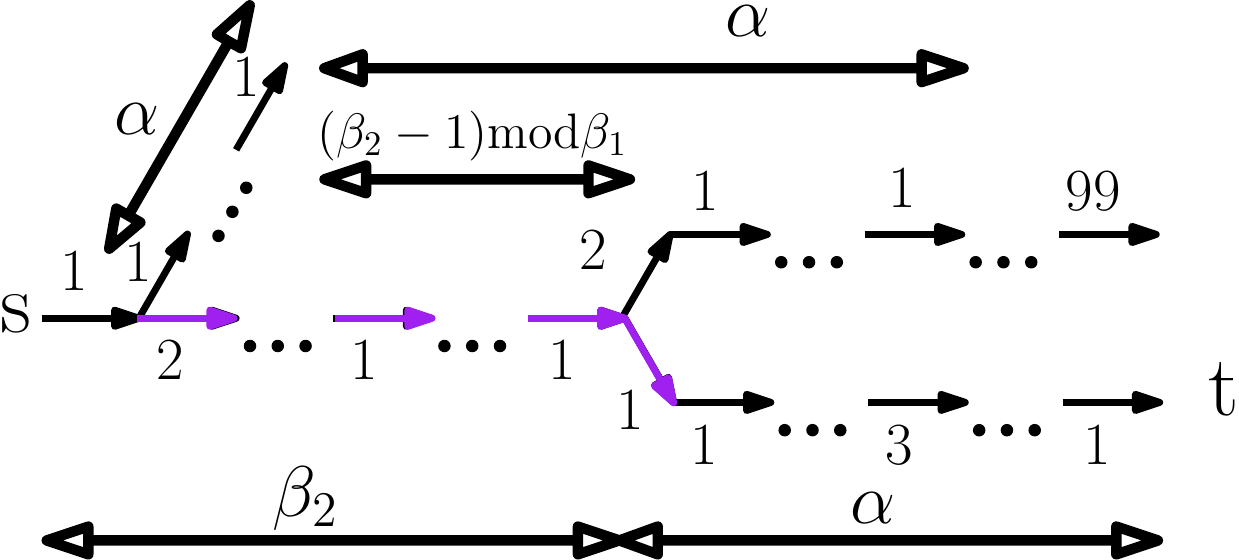}
		\caption{}\label{fig:6b}
	\end{subfigure}
	\begin{subfigure}[t]{0.5\columnwidth}
		\centering
		\includegraphics[width=\columnwidth]{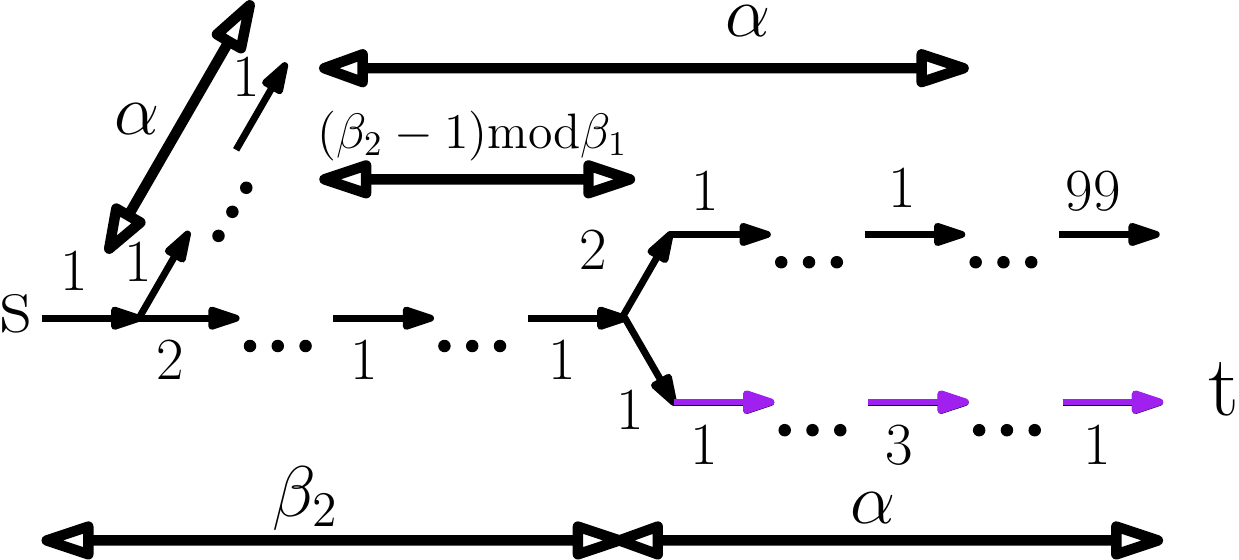}
		\caption{}\label{fig:6c}
	\end{subfigure}
	\caption{Construction used in Lemma~\ref{lem:smaller_greediness_generic}, when $\beta_2 \text{mod} \beta_1 = 0$.
	The flow of \ab with greediness $\beta_1$ is depicted in (\subref{fig:5a}-\subref{fig:5d}).
	Edges found to be collision free and in collision are depicted in blue and red, respectively.
	The flow of \ab with greediness~$\beta_2$ is depicted in (\subref{fig:6a}-\subref{fig:6c}).
	Edges found to be collision free and in collision are depicted in purple and red, respectively.
	All the edges evaluated by each algorithm are shown in~(\subref{fig:5d}) and~(\subref{fig:6c}).	
	}\label{fig:lemma6b}
	\vspace{-4mm}
\end{figure*}

\end{appendices}
}

\bibliographystyle{aaai}
\bibliography{bibliography}

\end{document}